%% file: main.tex
\documentclass[conference]{IEEEtran}
\usepackage{times}

\usepackage[numbers]{natbib}
\usepackage{multicol}
\usepackage[pagebackref=true,breaklinks=true,colorlinks,bookmarks=false]{hyperref}
\usepackage{amsthm}
\usepackage{adjustbox}

\input{preamble_packages.tex}

\input{preamble_symbols.tex}

\input{shortcuts.tex}

\begin{document}

\title{\huge \nameshort: Fast Quantification of Pose Uncertainty Sets}

\author{Author Names Omitted for Anonymous Review. Paper-ID [192]}



%
\author{\authorblockN{Yihuai Gao\authorrefmark{1}\authorrefmark{2},
Yukai Tang\authorrefmark{1}\authorrefmark{3},
Han Qi\authorrefmark{4} and 
Heng Yang\authorrefmark{4}}
\authorblockA{\authorrefmark{2}Stanford University, \texttt{yihuai@stanford.edu}}
\authorblockA{\authorrefmark{3}Princeton University, \texttt{yt3846@princeton.edu}}
\authorblockA{\authorrefmark{4}Harvard University, \texttt{\{hqi,hankyang\}@g.harvard.edu}
}
}
\newcommand\blfootnote[1]{%
  \begingroup
  \renewcommand\thefootnote{}\footnote{#1}%
  \addtocounter{footnote}{-1}%
  \endgroup
}

\maketitle
\blfootnote{\authorrefmark{1} equal contribution.}

\input{sections/abstract.tex}

\IEEEpeerreviewmaketitle

\input{sections/introduction.tex}
\input{sections/formulation-megb.tex}

\input{sections/dynamical-system.tex}

\input{sections/algorithm.tex}

\input{sections/experiments.tex}

\input{sections/conclusion.tex}

\section*{Acknowledgments}
We would like to thank William Zhang for drawing the 3D visualizations in Fig.~\ref{fig:purse}.

\clearpage
\onecolumn
\appendix

\renewcommand{\theequation}{A\arabic{equation}}
\renewcommand{\theproposition}{A\arabic{proposition}}
\renewcommand{\thetheorem}{A\arabic{theorem}}
\renewcommand{\thefigure}{A\arabic{figure}}
\renewcommand{\thetable}{A\arabic{table}}
\renewcommand{\thealgocf}{A\arabic{algocf}}

\input{sections/app-non-gaussian-3d3d.tex}

\input{sections/app-uncertainty-calibration-3d3d.tex}
\input{sections/app-uncertainty-calibration-foundation.tex}
\input{sections/app-geodesic-gradient-descent.tex}
\input{sections/app-algorithm.tex}

\input{sections/app-intuition-of-alternate.tex}

\bibliographystyle{plainnat}
\bibliography{references}

\end{document}

%% file: preamble_packages.tex
\usepackage{comment}
\usepackage{siunitx}
\usepackage{relsize}
\usepackage{ifthen}
\usepackage[colorinlistoftodos]{todonotes}

\usepackage[vlined,ruled,linesnumbered]{algorithm2e}
\usepackage{graphics} 
\usepackage{rotating}
\usepackage{color}
\usepackage{enumerate}
\usepackage[T1]{fontenc}
\usepackage{psfrag}
\usepackage{epsfig} 
\usepackage{booktabs}
\usepackage{graphicx,url}
\usepackage{multirow}
\usepackage{array}
\usepackage{latexsym}
\usepackage{amsfonts}
\usepackage{amsmath}
\usepackage{amssymb}
\usepackage{mathtools}
\usepackage{xstring}
\usepackage[noend]{algorithmic}
\usepackage{multirow}
\usepackage{xcolor}
\usepackage{prettyref}
\usepackage{flexisym}
\usepackage{bigdelim}
\usepackage{breqn} 
\usepackage{listings}
\let\labelindent\relax
\usepackage{enumitem}
\usepackage{xspace}
\usepackage{bm}
\graphicspath{{./figures/}}
\usepackage{tikz}
\usetikzlibrary{matrix,calc}

\usepackage{mdwlist}

\let\itemize\undefined
\makecompactlist{itemize}{stditemize}


%% file: preamble_symbols.tex
\newrefformat{prob}{Problem\,\ref{#1}}
\newrefformat{def}{Definition\,\ref{#1}}
\newrefformat{sec}{Section\,\ref{#1}}
\newrefformat{sub}{Section\,\ref{#1}}
\newrefformat{prop}{Proposition\,\ref{#1}}
\newrefformat{app}{Appendix\,\ref{#1}}
\newrefformat{alg}{Algorithm\,\ref{#1}}
\newrefformat{cor}{Corollary\,\ref{#1}}
\newrefformat{thm}{Theorem\,\ref{#1}}
\newrefformat{lem}{Lemma\,\ref{#1}}
\newrefformat{fig}{Fig.\,\ref{#1}}
\newrefformat{tab}{Table\,\ref{#1}}

\newtheorem{theorem}{Theorem}

\newtheorem{lemma}[theorem]{Lemma}

\newtheorem{definition}[theorem]{Definition}
\newtheorem{proposition}[theorem]{Proposition}

\newtheorem{example}[theorem]{Example}

\newcommand{\bdmath}{\begin{dmath}}
\newcommand{\edmath}{\end{dmath}}
\newcommand{\beq}{\begin{equation}}
\newcommand{\eeq}{\end{equation}}
\newcommand{\bdm}{\begin{displaymath}}
\newcommand{\edm}{\end{displaymath}}
\newcommand{\bea}{\begin{eqnarray}}
\newcommand{\eea}{\end{eqnarray}}
\newcommand{\beal}{\beq \begin{array}{ll}}
\newcommand{\eeal}{\end{array} \eeq}
\newcommand{\beas}{\begin{eqnarray*}}
\newcommand{\eeas}{\end{eqnarray*}}
\newcommand{\ba}{\begin{array}}
\newcommand{\ea}{\end{array}}
\newcommand{\bit}{\begin{itemize}}
\newcommand{\eit}{\end{itemize}}
\newcommand{\ben}{\begin{enumerate}}
\newcommand{\een}{\end{enumerate}}

\newcommand{\calD}{{\cal D}}

\newcommand{\calI}{{\cal I}}

\newcommand{\calS}{{\cal S}}

\newcommand{\calY}{{\cal Y}}

\newcommand{\eg}{\emph{e.g.,}\xspace}
\newcommand{\ie}{\emph{i.e.,}\xspace}
\newcommand{\resp}{\emph{resp.,}\xspace}

\newcommand{\hide}[1]{}
\newcommand{\wrt}{w.r.t.\xspace}

\newcommand{\hiddenText}{{\color{gray} hidden text.}}
\newcommand{\hideWithText}[1]{\hiddenText}

\newcommand{\subject}{\text{ subject to }}

\newcommand{\norm}[1]{\left\| #1 \right\|}

\newcommand{\tran}{^{\mathsf{T}}}

\newcommand{\trace}[1]{\mathrm{tr}\left(#1\right)}
\newcommand{\conv}[1]{\mathrm{conv}\left(#1\right)}

\newcommand{\eye}{{\mathbf I}}

\newcommand{\Real}[1]{ { {\mathbb R}^{#1} } }

\newcommand{\SEthree}{\ensuremath{\mathrm{SE}(3)}\xspace}

\newcommand{\SOthree}{\ensuremath{\mathrm{SO}(3)}\xspace}

\newcommand{\scenario}[1]{{\smaller \sf#1}\xspace}

\newcommand{\blue}[1]{{\color{blue}#1}}

\newcommand{\linkToPdf}[1]{\href{#1}{\blue{(pdf)}}}
\newcommand{\linkToPpt}[1]{\href{#1}{\blue{(ppt)}}}
\newcommand{\linkToCode}[1]{\href{#1}{\blue{(code)}}}
\newcommand{\linkToWeb}[1]{\href{#1}{\blue{(web)}}}
\newcommand{\linkToVideo}[1]{\href{#1}{\blue{(video)}}}
\newcommand{\linkToMedia}[1]{\href{#1}{\blue{(media)}}}
\newcommand{\award}[1]{\xspace} 



%% file: shortcuts.tex
\newcommand{\R}{\mathbb{R}}

\renewcommand{\norm}[1]{\left\lVert #1 \right\rVert}

\newcommand{\vectorize}[1]{\mathrm{vec}\parentheses{#1}}

\newcommand{\cbrace}[1]{\left\{#1\right\}}
\newcommand{\sym}[1]{\mathbb{S}^{#1}}

\newcommand{\bmat}{\left[ \begin{array}}
\newcommand{\emat}{\end{array}\right]}

\newcommand{\pd}[1]{\sym{#1}_{++}}
\newcommand{\parentheses}[1]{\left(#1\right)}
\newcommand{\half}{\frac{1}{2}}

\newcommand{\abs}[1]{\left|#1\right|}

\newcommand{\rotsub}{_{{R}}}
\newcommand{\transub}{_{{t}}}

\newcommand{\halfpi}{\frac{\pi}{2}}
\newcommand{\Exp}[1]{\mathrm{Exp}\parentheses{#1}}

\renewcommand{\cos}[1]{\mathrm{cos}\parentheses{#1}}
\renewcommand{\sin}[1]{\mathrm{sin}\parentheses{#1}}
\renewcommand{\cot}[1]{\mathrm{cot}\parentheses{#1}}
\newcommand{\acos}[1]{\mathrm{arccos}\parentheses{#1}}

\newcommand{\purse}{\scenario{PURSE}}
\newcommand{\pursetwo}{\scenario{PURSE2D3D}}
\newcommand{\pursethree}{\scenario{PURSE3D3D}}
\newcommand{\purseregression}{\scenario{PURSEreg}}

\newcommand{\threedmatch}{\scenario{3DMatch}}
\newcommand{\lmo}{\scenario{LM-O}}
\newcommand{\lm}{\scenario{LM}}

\newcommand{\fpfh}{\scenario{FPFH}}
\newcommand{\fcgf}{\scenario{FCGF}}

\newcommand{\ransag}{\scenario{RANSAG}}

\renewcommand{\eye}{\mathrm{I}}
\newcommand{\megb}{\scenario{MEGB}}
\newcommand{\rotdist}{\mathrm{dist}_{\SOthree}}
\newcommand{\trandist}{\mathrm{dist}_{\Real{3}}}
\newcommand{\ballrot}{B_{\SOthree}}
\newcommand{\balltran}{B_{\Real{3}}}
\newcommand{\hatC}{\hat{C}}
\newcommand{\hatc}{\hat{c}}
\newcommand{\hatD}{\hat{D}}
\newcommand{\hatd}{\hat{d}}
\newcommand{\hatS}{\hat{S}}
\newcommand{\nameshort}{\scenario{CLOSURE}}
\newcommand{\supp}{Supplementary Material}
\newcommand{\dgr}{\scenario{DGR}}
\newcommand{\nameacc}{\scenario{CLOSURE++}}
\newcommand{\grcc}{\scenario{GRCC}}
\newcommand{\ransagfmincon}{\scenario{RANSAG+fmincon}}
\newcommand{\calnum}{\scenario{Calibration}}

%% file: sections/abstract.tex

\begin{abstract}
We investigate uncertainty quantification of 6D pose estimation from {learned noisy} measurements {(\eg keypoints and pose hypotheses)}. Assuming \emph{unknown-but-bounded} measurement noises, a \emph{pose uncertainty set} {(\purse)} is a subset of $\SEthree$ that contains all possible 6D poses compatible with the measurements. Despite being simple to formulate and its ability to embed uncertainty, the \purse is difficult to manipulate and interpret due to the many abstract nonconvex polynomial constraints defining it. An appealing simplification of \purse --motivated by the bounded state estimation error assumption in robust control-- is to find its \emph{minimum enclosing geodesic ball} (\megb), \ie a point pose estimation with minimum worst-case error bound.
We contribute (i) a {geometric interpretation of the nonconvex \purse}, and (ii) a fast algorithm to \emph{inner} approximate the \megb. Particularly, we show the \purse corresponds to the feasible set of a constrained dynamical system {or the intersection of multiple geodesic balls}, and this perspective allows us to design an algorithm to densely sample the \emph{boundary} of the \purse through strategic random walks that are efficiently parallelizable on a GPU. We then use the miniball algorithm by Gärtner (1999) to compute the \megb of \purse samples, leading to an \emph{inner} approximation of the true \megb. Our algorithm is named \nameshort (\emph{enClosing baLl frOm purSe boUndaRy samplEs}) and it enables computing a \emph{certificate of approximation tightness} by calculating the \emph{relative ratio} between the size of the inner approximation and the size of the \emph{outer} approximation \grcc from Tang, Lasserre, and Yang (2023). Running on a {single RTX 3090 GPU}, {\nameshort achieves the relative ratio of $92.8\%$ on the \lmo object pose estimation dataset, $91.4\%$ on the \threedmatch point cloud registration dataset and $96.6\%$ on the \lm object pose estimation dataset with an average runtime below $0.3$ seconds}. Obtaining comparable worst-case error bound but $398\times$, $833\times$ {and $23.6\times$} faster than the outer approximation \grcc, \nameshort enables uncertainty quantification of 6D pose estimation to be implemented in real-time robot perception applications.

\end{abstract}

%% file: sections/introduction.tex

\section{Introduction}
\label{sec:intro}

6D pose estimation (\ie a 3D rotation and a 3D translation) from images and point clouds is a longstanding problem in robotics and vision and finds extensive applications in localization and mapping~\cite{song21ral-right}, robotic manipulation~\cite{deng20icra-self}, virtual and augmented reality~\cite{wen24cvpr-foundationpose}, and autonomous driving~\cite{shi23tro-optimal}.

{We focus on two popular paradigms for pose estimation in this paper and aim to endow them with rigorous \emph{uncertainty quantification}.}
{The first paradigm} is to start by detecting salient \emph{keypoints} in the sensor data --often done using deep neural networks~\cite{choy2020deep,yang21cvpr-self,peng19cvpr-pvnet,pavlakos17icra-6D}-- and then leverage the \emph{maximum likelihood estimation} (MLE) framework to estimate the optimal pose that best fits the keypoint measurements. {The second paradigm --initiated by PoseNet~\cite{kendall15iccv-posenet} and PoseCNN~\cite{xiang2017rss-posecnn} but recently became state of the art via FoundationPose~\cite{wen24cvpr-foundationpose}-- circumvents the need to detect keypoints and directly \emph{regresses} pose hypotheses (potentially followed by another MLE step).} Formally, let $x = (R,t) \in \SOthree \times \Real{3}:=\SEthree$ be the unknown pose to be estimated, {both paradigms generate $N$ noisy measurements $y_i \in \calY,i=1,\dots,N$ that satisfy}
\bea\label{eq:generative-model}
g(x,y_i) = \epsilon_i, \quad i=1,\dots,N,
\eea
where $g: \SEthree \times \calY \rightarrow \Real{m}$ is a known residual function that measures the discrepancy between the measurement $y_i$ and pose $x$. {In the keypoint-based paradigm, $y_i$ are keypoints; in the direct regression paradigm, $y_i$ are regressed pose hypotheses.} When $y_i$ is a perfect (\ie noise-free) measurement, $g$ evaluates to zero; when $y_i$ is noisy, $\epsilon_i \in \Real{m}$ describes the measurement noise in neural network predictions. We give {three} instantiations of~\eqref{eq:generative-model} that will be the focus of this paper. 
\begin{example}[{Keypoint-based} Object Pose Estimation~\cite{yang23cvpr-object,kneip14eccv-upnp}]\label{ex:2D3D}
Let $y_i = (z_i,Z_i) \in \Real{2}\times \Real{3}$ be a pair of matched 2D image keypoint and 3D object keypoint. The function
\bea\label{eq:g-2D3D}
g(x,y_i) = z_i - \Pi(R Z_i + t)
\eea
describes the \emph{reprojection error} of the 3D keypoint $Z_i$, where $\Pi(\cdot)$ is the camera projection function.\footnote{$\Pi: \Real{3} \rightarrow \Real{2}$, $\Pi(v) = [v_1/v_3,v_2/v_3]\tran$.}
\end{example}
\begin{example}[{Keypoint-based} Point Cloud Registration~\cite{yang20tro-teaser,choy2020deep}]\label{ex:3D3D}
Let $y_i = (a_i, b_i) \in \Real{3} \times \Real{3}$ be a pair of matched 3D keypoints in the source and target point clouds, respectively. The function
\bea\label{eq:g-3D3D}
g(x,y_i) = b_i - (R a_i + t)
\eea
describes the \emph{Euclidean error} between the keypoints.
\end{example}

{
    \begin{example}[Direct Pose Regression~\cite{kendall15iccv-posenet,xiang2017rss-posecnn,wen24cvpr-foundationpose}]\label{ex:poseregression}
        Let $y_i = (R_i, t_i) \in \SEthree$ be a pose hypothesis, the function
        \bea\label{eq:g-poseregression}
        g(x,y_i) = \begin{bmatrix}
            \vectorize{R} - \vectorize{R_i} \\
            t - t_i
        \end{bmatrix}
        \eea
        describes the \emph{relative pose} between $(R,t)$ and $(R_i,t_i)$,  where $\vectorize{\cdot}$ vectorizes a matrix as a vector.
        \end{example}
}
 
Given the {noisy measurements} $\{y_i\}_{i=1}^N$, one then formulates
\bea \label{eq:MLE}
\min_{x \in \SEthree} \sum_{i=1}^N \rho(g(x,z_i)),
\eea 
whose solution provides an estimate of the unknown pose $x$. For example, when the noise $\epsilon_i$ is assumed to follow a standard Gaussian distribution, then choosing $\rho(g(x,z_i)) = \Vert g(x,z_i) \Vert^2$ makes the solution of \eqref{eq:MLE} to be the maximum likelihood estimator. In practice, there are often measurements that do not follow \eqref{eq:generative-model} and lead to very large $\epsilon_i$, commonly known as an \emph{outlier}. To regain robustness to outliers, $\rho$ in \eqref{eq:MLE} is modified to be a \emph{robust loss}~\cite{huber04book-robust,black96ijcv-unification} and the optimization \eqref{eq:MLE} is referred to as \emph{M-estimation} (\ie MLE-like estimation).

{\bf Uncertainty Quantification}. Despite the many algorithmic advances in solving \eqref{eq:MLE}~\cite{bustos17pami-guaranteed,antonante21tro-outlier,chin22book-maximum}, \eg with certifiable global optimality guarantees~\cite{yang22pami-certifiably}, we argue there are two fundamental issues with the M-estimation framework. First, the starting assumption that $\epsilon_i$ follows a Gaussian-like distribution (up to the removal of outliers) is not justified. For example, \cite{tang23arxiv-uncertainty} shows noises generated by neural network keypoint detections in Example~\ref{ex:2D3D} fail almost all statistical multivariate normality tests~\cite{korkmaz2014mvn}, regardless of whether outliers are removed. We reinforce this empirical observation in \supp~and provide similar results showing real noises in keypoint matches of Example~\ref{ex:3D3D} also deviate far from a Gaussian distribution. Second, the solution of \eqref{eq:MLE} provides a \emph{single point estimate} that usually comes with no \emph{uncertainty quantification}, which is crucial when the pose estimations need to be used for downstream decision-making~\cite{dean20l4dc-robust,chou22wafr-safe}. Though it is possible to use the inverse of Fisher Information at the optimal solution to approximate the uncertainty,\footnote{{In linear least squares, this coincides with the covariance of the posterior distribution; in nonlinear least squares, this is called the Cramer-Rao lower bound, see \cite[Section B.6]{szeliski22book-computer} for a detailed explanation.}} such approximation (i) again builds on the assumption that $\epsilon_i$ is Gaussian, and (ii) is known to underestimate the true uncertainty~\cite{szeliski22book-computer}. Recent works~\cite{rosen19ijrr-sesync,yang20tro-teaser,carlone23ftr-estimation} provided uncertainty estimation for a few robot perception problems. However, the uncertainty either depends on uncheckable assumptions and cannot be computed~\cite{rosen19ijrr-sesync,yang20tro-teaser}, or build on machinery that only applies to estimators based on expensive semidefinite relaxations~\cite{yang22pami-certifiably}. 

{\bf Set-Membership Estimation with Noise Calibration}. An alternative (albeit less popular) estimation framework that this paper advocates for the purpose of practical and rigorous uncertainty quantification is that of \emph{set-membership estimation} (SME), widely known in control theory for system identification~\cite{milanese91automatica-sme,li23arxiv-learning,wang18automatica-set}. Instead of placing distributional assumptions on the measurement noise $\epsilon_i$ in~\eqref{eq:generative-model}, SME assumes the noise is \emph{unknown but bounded}, \ie
\bea\label{eq:unknown-but-bounded} 
\sqrt{\epsilon_i\tran \Lambda_i \epsilon_i}=:\Vert \epsilon_i \Vert_{\Lambda_i} \leq \beta_i, \ \ i=1,\dots,N
\eea
for positive definite matrices $\Lambda_i \in \pd{m}$ and positive noise bounds $\beta_i > 0$. With this assumption, and invoking~\eqref{eq:generative-model}, a \emph{Pose UnceRtainty SEt} (\purse) can be formulated as~\cite{yang23cvpr-object}
\begin{equation}\label{eq:purse}
    \hspace{-3mm} S = \cbrace{x \in \SEthree \mid \Vert g(x,y_i) \Vert_{\Lambda_i} \leq \beta_i,i=1,\dots,N }, \tag{\purse}
\end{equation}
\ie the set of all possible poses that are compatible with the measurements $y_i$ and the assumption~\eqref{eq:unknown-but-bounded}. The careful reader may ask ``\emph{why should we trust the unknown-but-bounded noise assumption~\eqref{eq:unknown-but-bounded} in SME more than the Gaussian-like noise assumption in MLE}?'' The answer is that recent statistical tools for distribution-free uncertainty calibration~\cite{roth22book-uncertain}, \eg \emph{conformal prediction}~\cite{angelopoulos21arxiv-gentle}, enables rigorous estimation of $\Lambda_i$ and $\beta_i$ from a calibration dataset. For instance, \cite{yang23cvpr-object} demonstrated how to calibrate noises generated by a pretrained neural network for Example~\ref{ex:2D3D}, and in \supp~we show how to calibrate neural network detection noises for Example~\ref{ex:3D3D} {and Example~\ref{ex:poseregression}} using conformal prediction. In fact within robotics, the SME framework has recently also gained popularity in simultaneous localization and mapping (SLAM) for guaranteed error analysis~\cite{ehambram22icra-interval,mustafa18ral-guaranteed}. However, the SLAM literature used \emph{interval analysis} to bound the estimation error of 2D poses, while in this paper we focus on uncertainty quantification of 3D rotations and translations.

{\bf Computational Challenges}. Although SME provides a natural description of uncertainty and its unknown-but-bounded noise assumption can be made practical by modern uncertainty calibration tools, it does bring computational challenges because the {\purse} is an abstract subset of $\SEthree$ defined by many nonconvex constraints (more precisely, polynomial inequalities~\cite{yang23cvpr-object}). This makes \purse difficult to be manipulated, \eg draw samples, visualize and interpret, extract a point estimate, compute volume, to name a few.  Therefore, it is desirable to simplify the {\purse}. Motivated by the common practice in robust perception-based control~\cite{dean20l4dc-robust,khalil96book-robust,cosner21iros-measurement} that assumes the state estimation has bounded error, it is appropriate to simplify \purse as a point estimate together with a \emph{worst-case error bound}, \ie an \emph{enclosing geodesic ball}\footnote{Since $\SEthree$ is a Riemannian manifold, we use the term ``geodesic''.} centered at the point estimate.\footnote{Given an enclosing geodesic ball, one can draw samples from \purse and estimate its volume by performing rejection sampling.} Since it is desired to find an enclosing geodesic ball with \emph{minimum conservatism}, we wish to compute the \emph{minimum enclosing geodesic ball} (\megb), \ie the smallest geodesic ball that encloses the original \purse (to be formulated precisely in Section~\ref{sec:megb}, {see Fig.~\ref{fig:inner-outer} for a graphical illustration}). Towards this goal, \cite{yang23cvpr-object} proposed an algorithm to compute the worst-case error bound for any point estimate, and~\cite{tang23arxiv-uncertainty} showed how to compute a hierarchy of enclosing geodesic balls (with decreasing sizes) that asymptotically converge to the \megb. There are, however, two shortcomings in~\cite{yang23cvpr-object,tang23arxiv-uncertainty}. First, both works provide \emph{outer} approximations of the \megb and it is unclear how conservative they are (\eg how much larger are these approximations compared to the true \megb). Second, they rely on semidefinite relaxations for the outer approximation, which are too expensive to be practical for real-time robotics applications (see Section~\ref{sec:experiments}). 

\input{sections/fig-inner-outer.tex}

{\bf Contributions}. We contribute an algorithm that quantifies the uncertainty of \purse in real time for Examples~\ref{ex:2D3D}-\ref{ex:poseregression}. The key perspective is that, \emph{algebraically}, {\purse} is defined by unstructured nonconvex constraints, but \emph{geometrically}, {\purse} corresponds to precisely the feasible sets of certain constrained dynamical systems {in Examples \ref{ex:2D3D}-\ref{ex:3D3D}, and the intersection of many geodesic balls in Example \ref{ex:poseregression}}. This observation leads to a natural algorithm: we first generate random samples in~{\purse} as initial states, and then perturb the states properly until it reaches the \emph{boundary} of the feasible set. Effectively, this algorithm aims to densely sample the {boundary} of \purse by \emph{strategic random walks}, which lends itself to fast parallization. Using the samples on the \purse boundary, we run the fast \emph{miniball} algorithm by Gärtner~\cite{gartner1999fast} to return a geodesic ball enclosing the samples. We name our algorithm \emph{enClosing baLl frOm purSe boUndaRy samplEs} (\nameshort). This geodesic ball in turn becomes an \emph{inner} approximation of the \megb. By comparing the \emph{relative ratio} between the size of the inner approximation computed by \nameshort and the size of the outer approximation computed by~\cite{tang23arxiv-uncertainty}, we obtain a \emph{certificate of tightness} for the approximation, {as shown in Fig.~\ref{fig:inner-outer}.} We test \nameshort on the \lmo dataset~\cite{brachmann14eccv-lmo} for Example~\ref{ex:2D3D}, the \threedmatch dataset~\cite{zeng17cvpr-3dmatch} for Example~\ref{ex:3D3D}, {and the \lm dataset~\cite{hinterstoisser2013model} for Example~\ref{ex:poseregression}}, and demonstrate that (i) with an average runtime of $0.1879$ second, \nameshort computes inner approximations with average relative ratio $92.8\%$ for Example~\ref{ex:2D3D}, (ii) with an average runtime of $0.1774$ second, \nameshort computes inner approximations with average relative ratio $91.4\%$ for Example~\ref{ex:3D3D} {and (iii) with an average runtime of $0.2768$ second, \nameshort computes inner approximations with average relative ratio $96.6\%$ for Example~\ref{ex:poseregression}}; all using a single NVIDIA RTX 3090 GPU. This means \nameshort is $398\times$, $833\times$, {and $23.6\times$} faster than the outer approximation algorithm in~\cite{tang23arxiv-uncertainty}, respectively, making it feasible in real-time perception systems. 

{\bf Paper Organization}. We formalize the notion of \megb and discuss high-level outer and inner approximation strategies in Section~\ref{sec:megb}. We then focus on the inner approximation: in Section~\ref{sec:dynamical-system} we show {the geometric structure of \purse} and in Section~\ref{sec:algorithm} we detail the \nameshort algorithm. We provide experimental results in Section~\ref{sec:experiments} and conclude in Section~\ref{sec:conclusion}.



%% file: sections/fig-inner-outer.tex

\begin{figure}[t]
	\begin{center}
    \includegraphics[width=0.9\columnwidth]{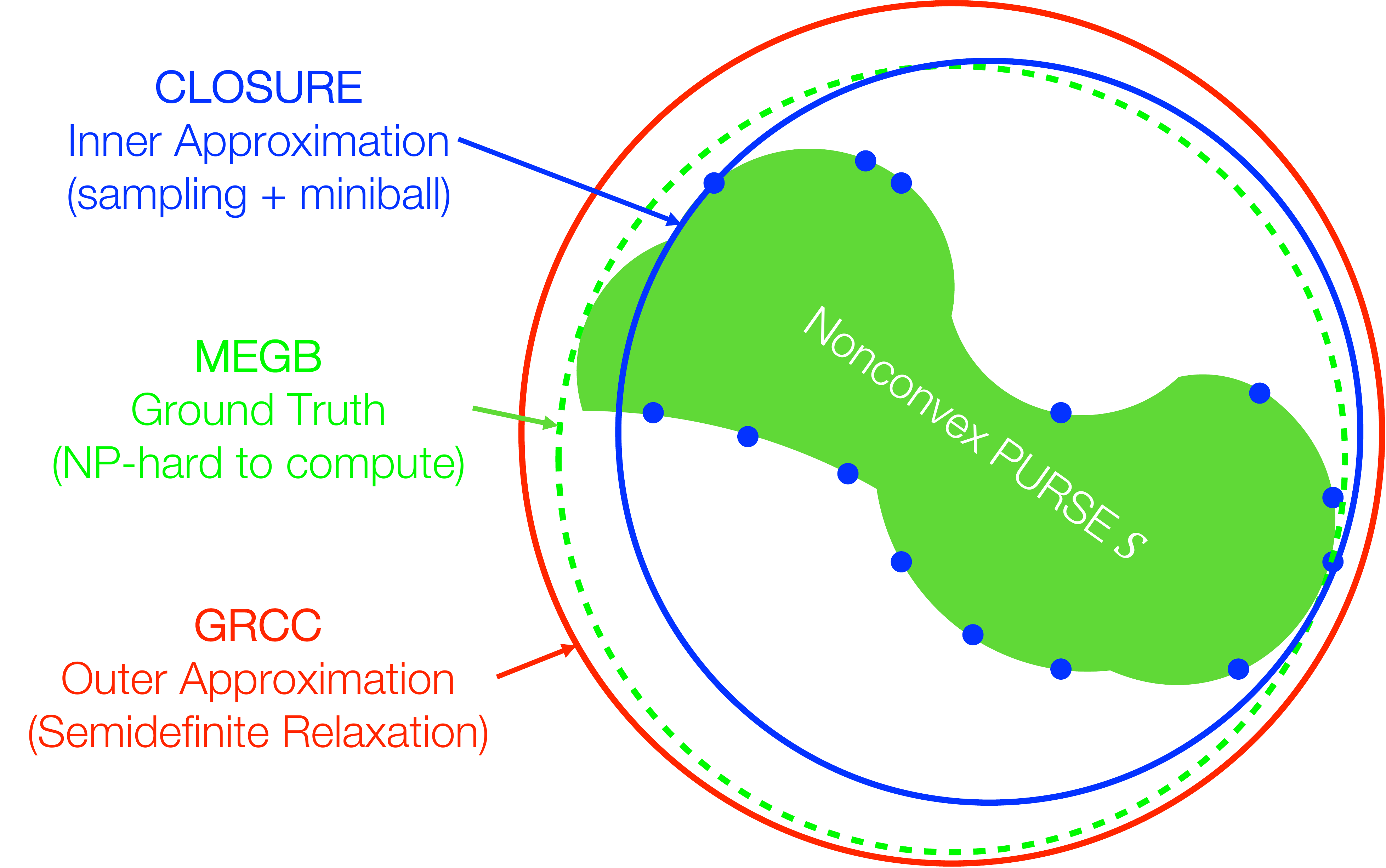}\\
    \end{center}
    \vspace{-5mm}
	\caption{{Illustration of the relationship between the outer approximation (provided by the \grcc algorithm from~\cite{tang23arxiv-uncertainty}), inner approximation (provided by the proposed \nameshort algorithm) and the ground truth \megb. Blue dots are the boundary pose samples in \purse, which are used to compute the inner approximation through the miniball algorithm~\cite{gartner1999fast}.}
	\label{fig:inner-outer}}
	\vspace{-6mm}
\end{figure}

%% file: sections/formulation-megb.tex

\section{Minimum Enclosing Geodesic Ball}
\label{sec:megb}

\input{sections/fig-megb-toy.tex}

Given the uncertainty set $S$ defined in {\purse}, denote $S\rotsub \subseteq \SOthree$, $S\transub \subseteq \Real{3}$ as the projection of $S$ onto $\SOthree$ and $\Real{3}$, respectively. Consider the optimization problems 
\begin{equation}\label{eq:megb-rot}
    \min_{C \in \SOthree} \max_{R \in S\rotsub} \ \rotdist (C,R)
\end{equation}
and 
\begin{equation}\label{eq:megb-tran}
    \min_{c \in \Real{3}} \max_{t \in S\transub} \ \trandist (c,t)
\end{equation}
where $\rotdist(C,R) := \arccos((\trace{C\tran R} - 1)/2)$ and $\trandist(c,t) := \Vert c - t \Vert$ denote the geodesic distance metrics in $\SOthree$ and $\Real{3}$, respectively. Let $C^\star$ (resp. $D^\star$) and $c^\star$ (resp. $d^\star$) be the optimizer (resp. optimum) of~\eqref{eq:megb-rot} and~\eqref{eq:megb-tran}, then the minimum enclosing geodesic ball (\megb) of $S\rotsub$ is 
\bea\label{eq:megb-rot-ball}
\ballrot(C^\star, D^\star) = \cbrace{R \in \SOthree\! \mid\! \rotdist(R,C^\star)\! \leq\! D^\star},\!\!
\eea
and the \megb of $S\transub$ is
\bea\label{eq:megb-tran-ball}
\balltran(c^\star, d^\star) = \cbrace{t \in \Real{3} \mid \trandist(t,c^\star) \leq d^\star}.
\eea
By the minimax nature of problems~\eqref{eq:megb-rot}-\eqref{eq:megb-tran}, we have
$$
S\rotsub \subseteq \ballrot(C^\star, D^\star), \quad S\transub \subseteq \balltran(c^\star,d^\star),
$$
\ie the geodesic balls enclose the original uncertainty sets. See {Fig.~\ref{fig:inner-outer} for an illustration of the \megb for a nonconvex \purse.} 
The center $(C^\star,c^\star) \in \SEthree$ is the best point estimate achieving \emph{minimum} worst-case error bounds of $(D^\star,d^\star)$. We choose to seperate the \megb for $\SOthree$ and $\Real{3}$ because their distance metrics have different units, \ie degrees and meters. A classical result in Riemannian geometry~\cite{afsari11ams-riemannian} states the \megb exists and is unique under mild conditions. However, solving the minimax optimizations~\eqref{eq:megb-rot}\eqref{eq:megb-tran} is known to be intractable~\cite{arnaudon13cg-approximating}. We review several algorithmic choices for solving~\eqref{eq:megb-rot}\eqref{eq:megb-tran}.

{\bf Geodesic Gradient Descent}. A general recipe leveraging the \emph{geodesic convexity} of $f(C):=\max_{R \in S\rotsub} \rotdist(C,R)$ in $C$ can be constructed as follows: start with an initial rotation $C^{(0)}$; then at each iteration (a) find the rotation $R$ in $S\rotsub$ that attains the maximum distance to the current iterate $C^{(k)}$, and (b) move along the geodesic from $C^{(k)}$ to $R$ by a portion of $\gamma \in (0,1)$, which is equivalent to performing {(sub)gradient} descent on $f(C)$. Due to geodesic convexity of $f(C)$, this algorithm is guaranteed to converge with well-known complexity analysis~\cite{zhang16colt-first}. However, computing the maximum distance from $S\rotsub$ to $C^{(k)}$ in step (a) boils down to solving a nonconvex optimization and is in general not implementable. In \supp~we attempt this algorithm with step (a) implemented by the semidefinite relaxation proposed in~\cite{yang23cvpr-object} and show that it does converge but is impractical due to the excessive runtime (near one hour).

Since exactly solving~\eqref{eq:megb-rot}\eqref{eq:megb-tran} using gradient descent is impractical, we turn to approximation algorithms.

{\bf Outer Approximation}. \cite{tang23arxiv-uncertainty} proposed an algorithm that produces a hierarchy of enclosing geodesic balls that asymptotically converge to the \megb from above. This algorithm, however, has two drawbacks. First, the convergence is not detectable and therefore it is unclear how conservative is the computed enclosing geodesic ball compared to the true \megb. Second, this algorithm is also built upon semidefinite relaxations and it is too slow to be practical (see Section \ref{sec:experiments}). 

We focus on inner approximations in this paper.

{\bf Inner Approximation}. If in addition to the outer approximation provided by~\cite{tang23arxiv-uncertainty}, one can produce an \emph{inner} approximation of the \megb, then by comparing the relative ratio between the sizes of the inner and outer approximations, a certficate of approximation tightness can be obtained. The next result provides a straightforward way to compute inner approximations.

\begin{proposition}[Inner Approximation of \megb] \label{prop:inner}
    Let $\hatS\rotsub \subseteq S\rotsub$ and $\hatS\transub \subseteq S\transub$ be nonempty subsets. Consider the optimizations 
\bea\label{eq:inner-megb-rot}
\min_{C \in \SOthree} \max_{R \in \hatS\rotsub} \ \rotdist(C,R)
\eea
and
\bea\label{eq:inner-megb-tran}
\min_{t \in \Real{3}} \max_{t \in \hatS\transub} \ \trandist(c, t),
\eea 
and denote their optimizers (resp. optima) to be $\hatC$ (resp. $\hatD$) and $\hatc$ (resp. $\hatd$). Then $\ballrot(\hatC,\hatD)$ is no greater than the \megb~\eqref{eq:megb-rot-ball} and $\balltran(\hatc,\hatd)$ is no greater than the \megb~\eqref{eq:megb-tran-ball}, \ie 
\bea
\hatD \leq D^\star, \quad \hatd \leq d^\star.
\eea  
\end{proposition}
\begin{proof}
    In~\eqref{eq:megb-rot} and~\eqref{eq:inner-megb-rot}, for any $C \in \SOthree$ denote
    \bea 
    f(C):=& \max_{R \in S\rotsub} \rotdist(C,R), \nonumber \\
    \hat{f}(C):=& \max_{R \in \hatS\rotsub} \rotdist(C,R). \nonumber
    \eea
    Since $\hatS\rotsub \subseteq S\rotsub$, we have
    $$
    f(C) \geq \hat{f}(C), \quad \forall C \in \SOthree.
    $$
    Therefore,
    $$
     \min_{C\in \SOthree} f(C)= D^\star \geq \hatD = \min_{C \in \SOthree} \hat{f}(C).
    $$
    The proof for $\hatc \leq c^\star$ is similar.
\end{proof}

Proposition~\ref{prop:inner} states that, given any nonempty subsets of the original \purse $S\rotsub$ and $S\transub$, solving the minimax problems~\eqref{eq:inner-megb-rot}-\eqref{eq:inner-megb-tran} leads to valid inner approximations. We remark that here the word ``inner'' is abused in the sense that we can only guarantee the \megb of $\hatS\rotsub$ (resp. $\hatS\transub$) is \emph{smaller} than, but \emph{not enclosed by}, the \megb of the original \purse $S\rotsub$ (resp. $S\transub$). Fig.~\ref{fig:megb-toy} gives such an counterexample. Nevertheless, we chose to use ``inner'' approximation to parallel the ``outer'' approximation in previous works. To make Proposition~\ref{prop:inner} useful, we need to strategically choose the subsets $\hatS\rotsub$ and $\hatS\transub$ such that (i) the minimax problems are easier to solve, and (ii) the inner approximations are as tight (large) as possible. Requirement (i) is easy to satisfy if we choose the subsets as \emph{discrete samples}, \ie we approximate the original \purse as a point cloud of poses. In this case, the minimax problems can be solved exactly and efficiently~\cite{gartner1999fast}, which we will further explain in Section~\ref{sec:algorithm}. To satisfy requirement (ii), we leverage the intuition that the \megb must touch the \purse at a finite number of \emph{contact points}~\cite{xie16thesis-inner,lasserre15mp-generalization,henk12dm-lowner} that lie on the boundary of the \purse $\partial S\rotsub$ and $\partial S\transub$. Therefore, if we can densely sample $\partial S\rotsub$ and $\partial S\transub$, then there is high probability that the contact points will be included and the inner approximation will be close to the true \megb. For example, the \megb of the rectangle in Fig.~\ref{fig:megb-toy} is exactly the same as the \megb of the discrete point cloud containing four corners of the rectangle. Therefore, finding good samples on the boundary of the \purse is crucial to compute tight inner approximations of the \megb.

{We summarize the outer approximation and the inner approximation in Fig.~\ref{fig:inner-outer} for better understanding.}

However, the {\purse} is defined by abstract constraints and there exists no general algorithms that can sample its boundary, \ie points at which some of the inequalities become equalities. In the next section, we will show that, despite being algebraically unstructured, the \purse has simple geometrical structures that can be exploited to sample its boundary.

%% file: sections/fig-megb-toy.tex

\begin{figure}[t]
	\begin{center}
    \includegraphics[width=0.55\columnwidth]{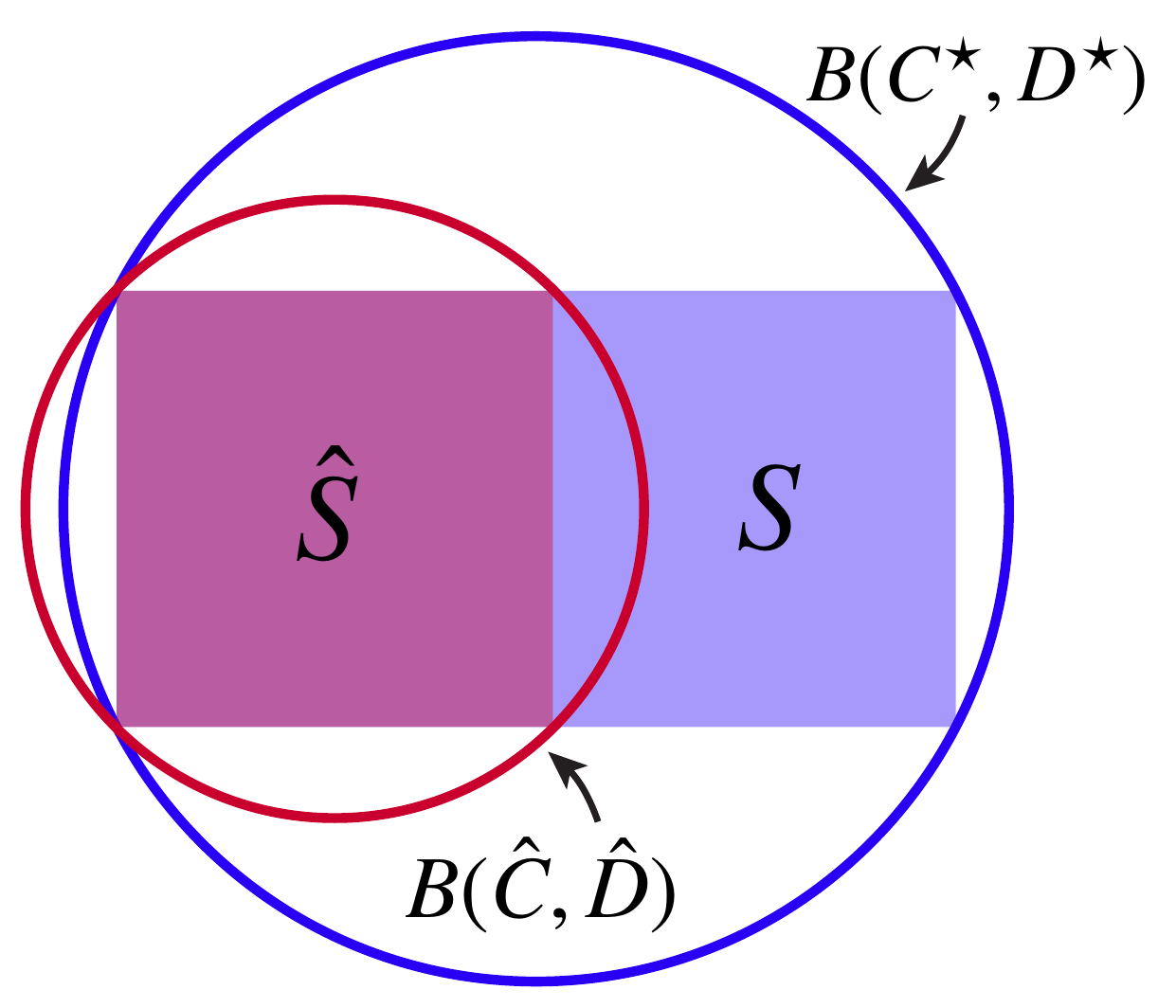}\\
    \end{center}
    \vspace{-5mm}
	\caption{\megb $B(C^\star,D^\star)$ of a simple rectangle $S$ and \megb $B(\hatC,\hatD)$ of the subset $\hatS \subset S$. Note that $\hatD < D^\star$ but $B(\hatC,\hatD) \not\subset B(C^\star,D^\star)$.
	\label{fig:megb-toy}}
	\vspace{-6mm}
\end{figure}

%% file: sections/dynamical-system.tex

\section{{Geometric Structures}}
\label{sec:dynamical-system}

\input{sections/fig-purse.tex}

Plugging the $g$ functions in~\eqref{eq:g-2D3D}, \eqref{eq:g-3D3D}, and \eqref{eq:g-poseregression} to the definition of {\purse}, we explicitly write down the \purse for Example~\ref{ex:2D3D}
\begin{equation}\label{eq:PURSE2D3D}
    \hspace{-4mm} S = \cbrace{(R,t) \mid \Vert z_i - \Pi(R Z_i + t) \Vert_{\Lambda_i} \leq \beta_i,\forall i}, \tag{\pursetwo}
\end{equation}
for Example~\ref{ex:3D3D}
\begin{equation}\label{eq:PURSE3D3D}
    \hspace{-6mm} S = \cbrace{(R,t) \mid \Vert b_i - Ra_i - t \Vert_{\Lambda_i} \leq \beta_i,\forall i}. \tag{\pursethree}
\end{equation}
{and for Example~\ref{ex:poseregression}
\begin{equation}\label{eq:PURSEregression}
    \hspace{-6mm} S = \cbrace{(R,t) \mid \Vert \vectorize{R} - \vectorize{R_i} \Vert_{\Lambda_i^R} \leq \beta_{i}^R, \Vert t - t_i \Vert_{\Lambda_i^t} \leq \beta_{i}^t,\forall i}. \tag{\purseregression}
\end{equation}
Note that in \eqref{eq:PURSEregression} we separate the constraints in $R$ and $t$ since $\beta_i^R$ and $\beta_i^t$ have different units.}

Although the constraints are complicated in $(R,t)$, they actually define quite simple geometric sets. 

{\bf \pursetwo}. Each constraint in~\eqref{eq:PURSE2D3D} asks $\Pi(RZ_i + t)$ --the reprojection of $Z_i$-- to lie inside an ellipse on the image plane. Fig.~\ref{fig:purse}(a) depicts the geometric constraints: the line from the camera center to each 3D keypoint, \ie the bearing vector, needs to pass through the ellipse of bounded measurement noise. This is precisely a constrained mechanical system where there is (i) a prismatic joint between each 3D keypoint and each bearing vector, and (ii) a spherical joint between each bearing vector and the camera center. The $N$ 3D keypoints form a rigid body (as shown by the connected links in Fig.~\ref{fig:purse}(a)) that can move in 3D space subject to the bearing vectors not passing the 2D ellipses.

{\bf \pursethree}. Each constraint in~\eqref{eq:PURSE3D3D} enforces $R a_i + b_i$ --the rigid transformation of $a_i$-- to lie inside an ellipsoid centered at $b_i$. Fig.~\ref{fig:purse}(b) depicts the geometric constraints: each of the ellipsoid forms a shell fixed in 3D space, and the rigid body formed by 3D points $\{a_i\}_{i=1}^N$ (again showned by the connected links) can freely move in 3D space subject to not escaping the ellipsoidal shells. 

{
{\bf \purseregression}. Each constraint in~\eqref{eq:PURSEregression} enforces $R$ (\resp $t$) to lie inside a ball centered at $R_i$ (\resp $t_i$) with radius $\beta_i^R$ (\resp $\beta_i^t$). Therefore, \purseregression is simply the intersection of $N$ geodesic balls centered at $R_i$ and $t_i,i=1,\dots,N$.
}

This geometric perspective {inspires a natural algorithm to sample the boundary of \purse, \ie we can start with random samples inside the \purse and then add random perturbations to ``walk'' the samples on the $\SEthree$ manifold until they ``hit'' the boundary of the \purse. This is the \nameshort algorithm to be engineered in the next section.} 


%% file: sections/fig-purse.tex

\begin{figure}[t]
	\begin{center}
		\begin{tabular}{c}
			\begin{minipage}{0.8\columnwidth}
				\centering
				\includegraphics[width=\textwidth]{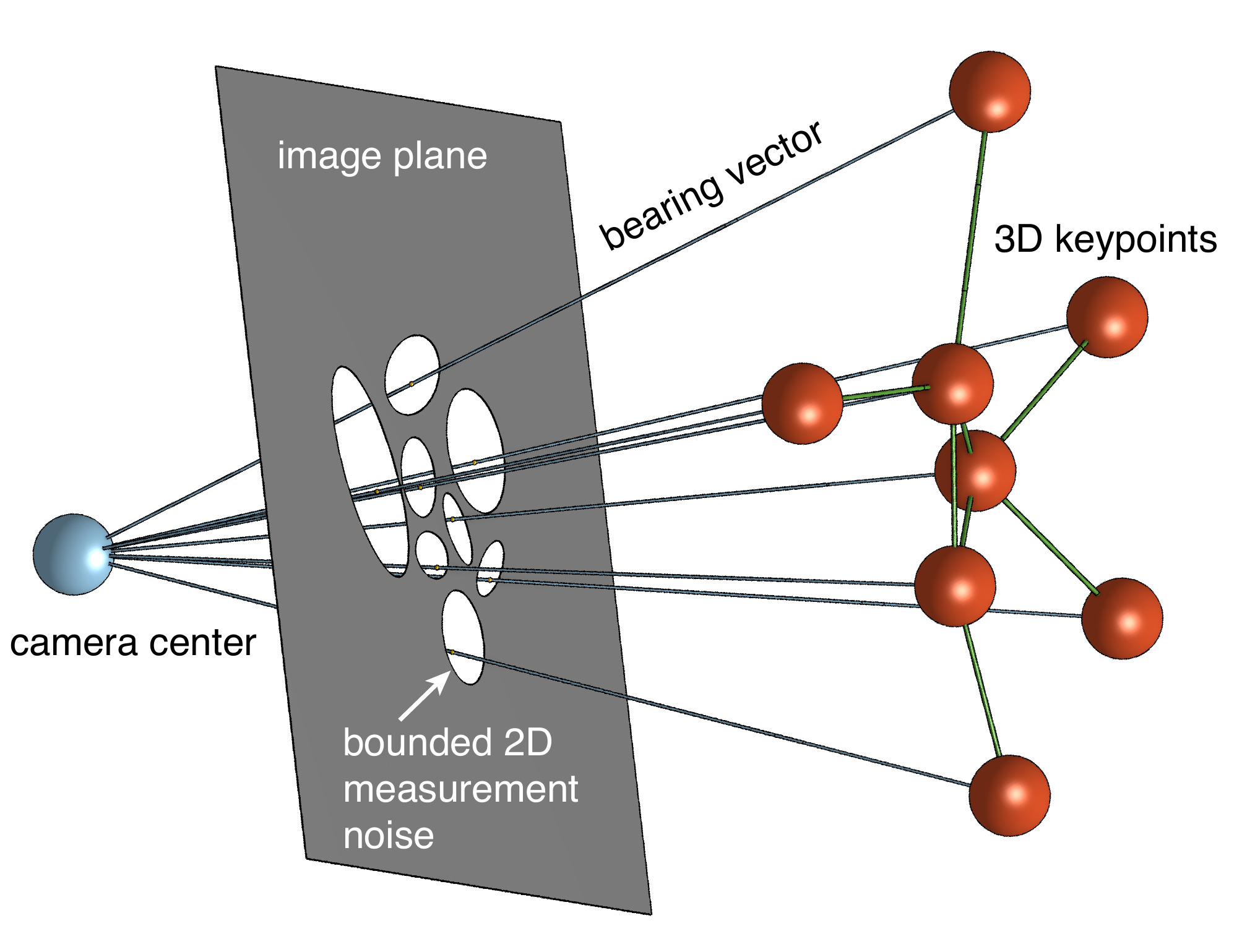}\\
                {\smaller (a) Constrained dynamical system for~\eqref{eq:PURSE2D3D}.}
			\end{minipage}
			\\
			\begin{minipage}{0.8\columnwidth}
				\centering
				\includegraphics[width=\textwidth]{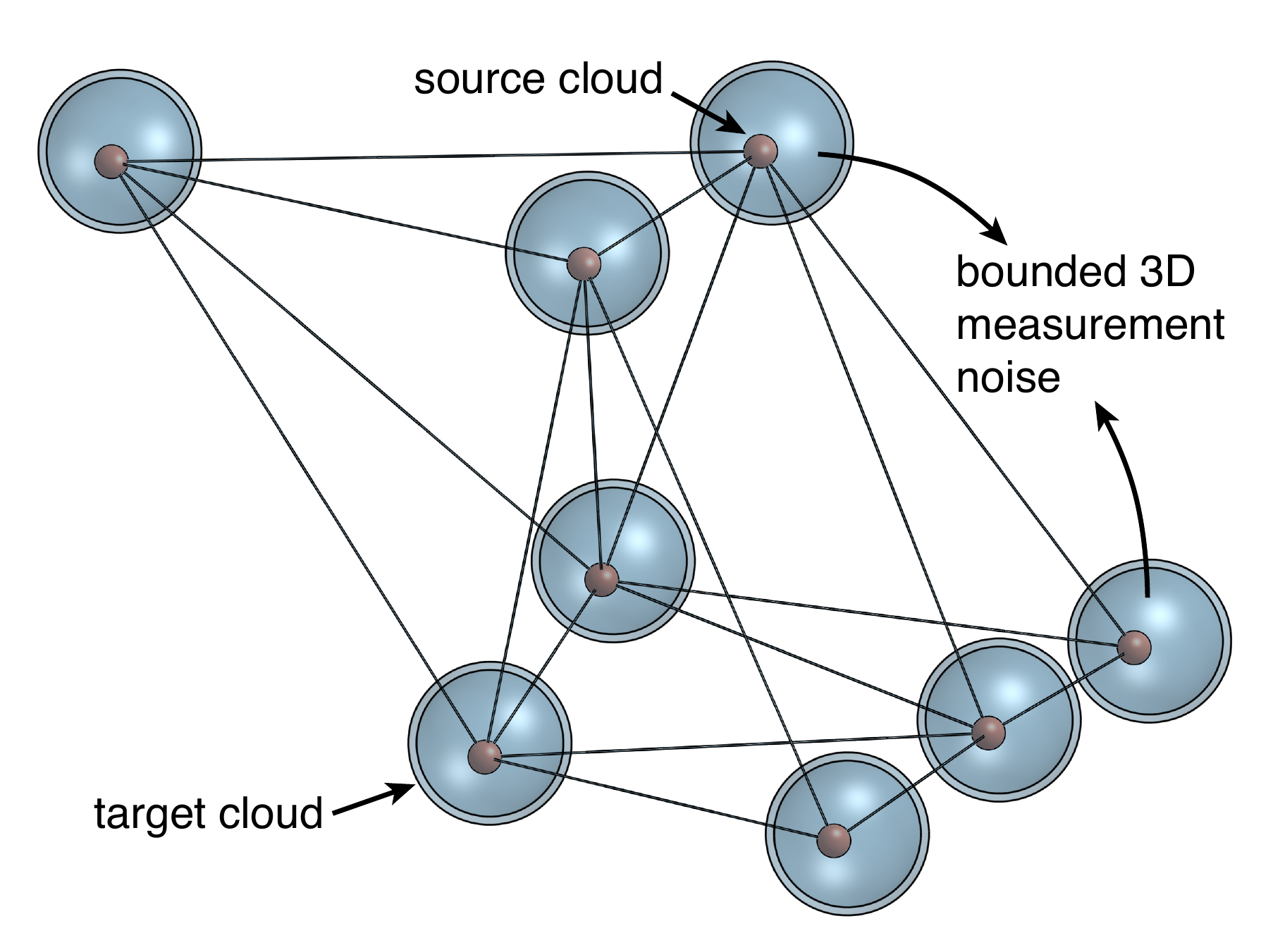}\\
                {\smaller (b) Constrained dynamical system for~\eqref{eq:PURSE3D3D}.}
			\end{minipage}
		\end{tabular}
	\end{center}
	\vspace{-2mm}
	\caption{Constrained dynamical systems whose feasible sets correspond to (a) \eqref{eq:PURSE2D3D} for Example~\ref{ex:2D3D} and (b) \eqref{eq:PURSE3D3D} for Example~\ref{ex:3D3D}.
	\label{fig:purse}}
\end{figure}

%% file: sections/algorithm.tex

\section{The \nameshort Algorithm}
\label{sec:algorithm}

The \nameshort algorithm contains three steps, which is overviewed in Algorithm~\ref{alg:overview} and Fig.~\ref{fig:overview}.

\input{sections/fig-overview.tex}
\input{sections/alg-overview.tex}
\input{sections/alg-boundary-sampler.tex}
\input{sections/alg-parallel.tex}

{\bf Step I: Initialize Starting Poses}.
Given a {\purse} $S$, we first sample a set of initial poses $S_0$ from $S$. {For Example~\ref{ex:2D3D} and Example~\ref{ex:3D3D}, we use} an algorithm proposed in~\cite{yang23cvpr-object} called \emph{random sample averaging} (\ransag). The basic idea of \ransag is to leverage the well-established \emph{minimal solvers}~\cite{kukelova08eccv-automatic} to quickly generate candidate poses and check if they belong to \purse. {For Example~\ref{ex:poseregression}, we design a convex-combinational sampler.}
\begin{itemize}
    \item Example~\ref{ex:2D3D}: In each \ransag trial, we randomly select 3 constraints (\ie 3 ellipses in Fig.~\ref{fig:purse}(a)) and find a 2D point in each constraint. We then apply the perspective-3-point (P3P) minimal solver~\cite{gao03pami-p3p} to obtain a candidate pose $s$. We add $s$ into the set $S_0$ if $s$ satisfies \eqref{eq:PURSE2D3D}.
    \item Example~\ref{ex:3D3D}: In each \ransag trial, we randomly select 3 pairs of corresponding points, and apply Arun's Method~\cite{arun1987least} to compute a candidate pose $s$. We put $s$ into $S_0$ if $s$ satisfies~\eqref{eq:PURSE3D3D}.
    {\item Example~\ref{ex:poseregression}: In each trial, we randomly generate a convex combination of $10$ pose hypotheses to compute a candidate pose $s$. We put $s$ into $S_0$ if $s$ satisfies~\eqref{eq:PURSEregression}.}
\end{itemize}

{\bf Step II: Strategic Random Walk}.
In this step, we sample poses that are close to the boundary of the \purse, $\partial S$, starting from the initial samples $S_0 \subset S$ by Step I. The basic idea is to add random perturbations to every pose in $S_0$ until they hit the boundary of the \purse and violate the defining constraints. We engineer several techniques to make this idea more efficient. 

\begin{itemize}
    \item {\bf Walk away from the center}. We first find the average rotation $\bar{R}$ and average translation $\bar{t}$ of $S_0$ (line~\ref{line:Ravg}-\ref{line:tavg}). For each pose $(R_0, t_0)$ in $S_0$, we initialize $N_W$ walks with randomized angular or center velocities while pointing outwards $\bar{R}$ and $\bar{t}$, which explicitly encourages the samples to move away from the center and explore the boundary. To be exact, when sampling the rotation boundary $\partial S_R$, suppose $(u, \theta)$ is the axis-angle representation of the relative rotation between $R_0$ and $\bar{R}$, we initialize the angular velocity $\omega$ as $\omega = \omega_0(u+\tilde{u})$, where $\omega_0$ is angular velocity magnitude and $\tilde{u}\in \R^3$ is a random unit vector. Similarly, when sampling the translation boundary $\partial S_t$, we initialize the center velocity $v$ as $v = v_0\left((t_0-\bar{t})/\Vert t_0-\bar{t} \Vert +\tilde{v}\right)$, where $\tilde{v}\in \R^3$ is a random unit vector. However, in Example~\ref{ex:2D3D}, $S_t$ in the camera projection direction appears significantly longer than the other two directions as shown in Fig. \ref{fig:sample_comp}(a). To balance the geometrical singularity, we apply additional normalization to the translation velocity according to the PCA analysis of $S_0$.
    \item {\bf $\bm{\partial S \neq \partial S\rotsub \times \partial S\transub}$}. A vanilla walking strategy is to perform $N_I$ iterations of rigid body movement with a variety of perturbations and step sizes (reflected in the time step lengths) so that the evolved poses can get close enough to $\partial S$. However, recall that our starting goal is to sample $\partial S\rotsub$ and $\partial S\transub$, not $\partial S$. One may think that $\partial S = \partial S_R \times \partial S_t$, but this is in general \emph{not true}. In fact, $\partial S_R =\partial(S|_\SOthree) \subset(\partial S)|_\SOthree$ and $\partial S_t =\partial(S|_{\R^3})\subset(\partial S)|_{\R^3}$. In words, the boundary of $S\rotsub$ (resp. $S\transub$) is a subset of the boundary of $S$ projected onto $\SOthree$ (resp. $\Real{3}$). This means even if $(R, t)\in \partial S$, it is not guaranteed that $R\in \partial S_R$ and $t\in \partial S_t$. Please refer to \supp~for an intuitive example. With this observation in mind, we sample $\partial S_R$ and $\partial S_t$ separately. The $\partial S_R$ sampler is shown in Algorithm \ref{alg:sampleR}, and the $\partial S_t$ sampler is similar and presented in \supp~for brevity.
    
    When sampling $\partial S_R$ (resp. $\partial S_t$), we fix the angular velocity $\omega$ (resp. center velocity $v$) and add random perturbations to center translation (resp. rotation). In each iteration step, we first find top $N_p^*$ perturbations out of $N_p$ that \emph{drags the pose away from $\partial S$}, \ie maximizing the distance to $\partial S$. The distance between a pose $(R, t)\in \SEthree$ and $\partial S$ is defined heuristically as 
    \begin {equation}\label{eq:partial-s-dist}
        \text{dist}(R, t, \partial S) = \min_i \left( \beta_i - \Vert g_i((R, t), y_i) \Vert_{\Lambda_i}\right),
    \end{equation}
    where $g_i$ is the estimation error defined in \eqref{eq:g-2D3D}-\eqref{eq:g-poseregression}.
    Then we apply different step scales $\gamma^{m-1}, m=1, \cdots, N_T$ to the pose after each perturbation. Among all the $N_p^* N_T$ poses, we find the one that is still in $S$ and has the maximum rotation (resp. translation) movement to update the optimal pose $(R^*, t^*)$. Finally, after $N_I$ iterations, we add the optimal pose to $\hat{\partial S_R}$. In this way, we can sample $|S_0|N_W$ poses that are close enough to $\partial S_R$ (resp. $\partial S_t$) from different approaching directions. Algorithm~\ref{alg:sampleR} summarizes the strategic random walk that samples $\partial S\rotsub$.
    \item {\bf Parallelization}. To improve the sampling speed, we implement Algorithm~\ref{alg:sampleR} with parallel programming using CuPy,\footnote{\url{https://cupy.dev/}, a python interface for NVIDIA CUDA library.} presented in Algorithm~\ref{alg:sampleR_parallel}. We notice that $(R^*, t^*)$ is the only variable that is updated consecutively in each iteration, while all the other loops are independent and can be executed in parallel. Specifically, there are $|S_0|N_W$ walks in total with different initial velocities. In each iteration, the $N_P$ perturbations can be applied to the optimal poses in parallel, from which we can find $N_P^*$ optimal perturbations that maximize the distance to $\partial S$ according to~\eqref{eq:partial-s-dist}. Next, we apply all $N_T$ time step scales in a roll and broadcast to $N_P^*$ optimal perturbations for object movement simulation. At the end of each iteration, we update the optimal poses $(R^*, t^*)$ with the farthest movements. With in-depth analysis of the sampling algorithm, the steps with heaviest computation are~\ref{line:dist_gpu} and~\ref{line:in_purse_gpu}. In these steps, $|S_0|N_W N_P$ and $|S_0|N_W N_P^*N_T$ poses\footnote{In a typical case, $|S_0|\approx 100, N_W=2, N_P=150, N_P^*=10, N_T=15$, thus $|S_0|N_WN_P=|S_0|N_WN_P^*N_T\approx 3\times 10^4$.} are checked, while checking each pose requires to calculate \eqref{eq:PURSE2D3D}-\eqref{eq:PURSEregression}. Executing these steps in parallel significantly reduces the time cost compared with CPU based implementation.
\end{itemize}

{\bf Step III: Miniball}.
In this step, after acquiring a set of rotations $\hat{\partial S_R}$ and translations $\hat{\partial S_t}$ close to $\partial S_R$ and $\partial S_t$ respectively, we calculate the \megb of $\hat{\partial S_R}$ and $\hat{\partial S_t}$. The \emph{miniball} algorithm\footnote{\url{https://people.inf.ethz.ch/gaertner/subdir/software/miniball.html}} introduced in~\cite{gartner1999fast} provides a fast implementation that can exactly compute the \megb of a point cloud in Euclidean space. This means we can directly apply it to solve~\eqref{eq:inner-megb-tran} and compute $\balltran(\hatc,\hatd)$ because $\hat{\partial S_t} \subset \Real{3}$ lives in an Euclidean space. However, we cannot directly apply miniball to $\hat{\partial S_R}$ because it lives in $\SOthree$. Fortunately, \cite[Proposition A12, A13]{tang23arxiv-uncertainty} points out, under mild conditions, computing the \megb of a set on $\SOthree$ is equivalent to computing the \megb of the corresponding set on the space of unit quaternions, which in turn is equivalent to simply treating the unit quaternions as points in $\Real{4}$. Therefore, we can apply the same algorithm after embedding the rotations in $\Real{4}$ and solve~\eqref{eq:inner-megb-rot} to compute $\ballrot(\hatC,\hatD)$. Thanks to the efficiency of the miniball algorithm, the runtime of this step is negligible compared to \ransag and the strategic random walk.

In summary, the \nameshort algorithm takes the \purse as input, and outputs $\ballrot(\hatC,\hatD)$ and $\balltran(\hatc,\hatd)$, serving as inner approximations of the true \megb. By taking advantage of efficient GPU parallel computing, \nameshort can be executed around 0.2 seconds with satisfying performance and is suitable for real-time applications.

%% file: sections/fig-overview.tex

\begin{figure}[t]
	\begin{center}
    \begin{tabular}{cc}
        \hspace{-10mm}	
        \begin{minipage}{0.2\textwidth}
            \centering
            \includegraphics[width=\textwidth]{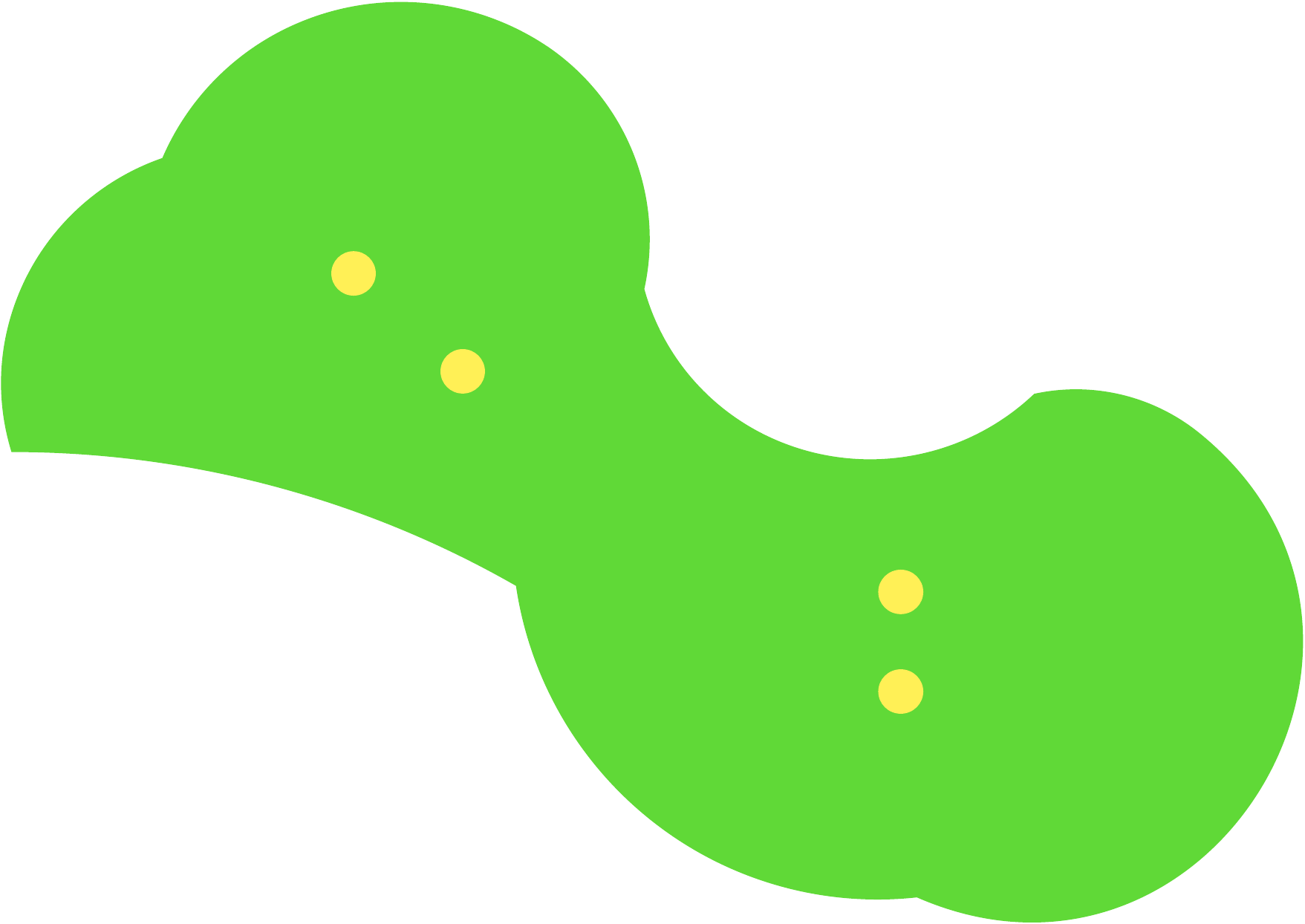}
            {{\smaller (a) Initialize poses $S_0$ through \ransag.}}
        \end{minipage} 
        \begin{minipage}{0.2\textwidth}
            \centering
            \includegraphics[width=\textwidth]{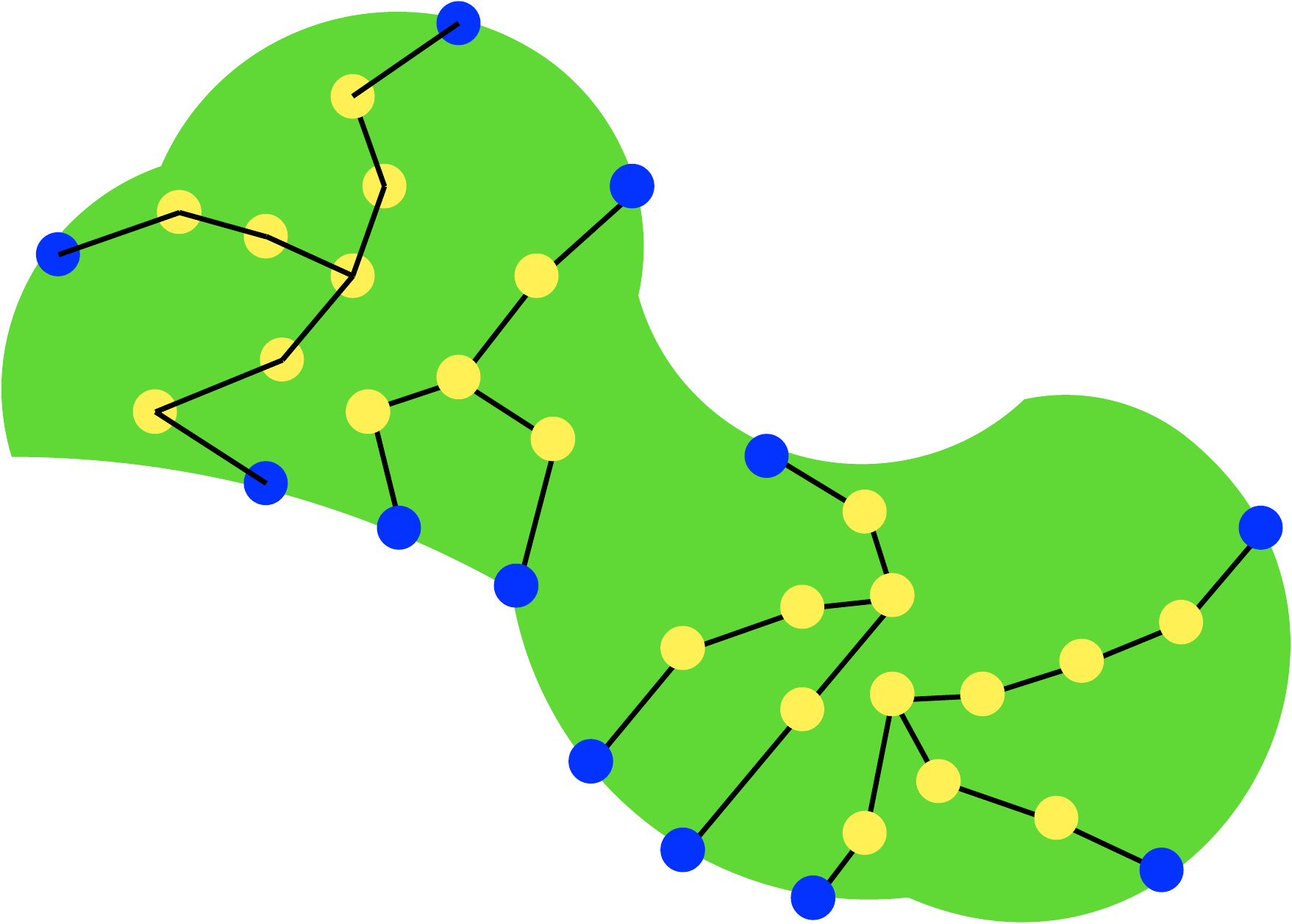}
            {{\smaller (b) Sample boundaries $\hat{\partial S_R}$ and $\hat{\partial S_t}$ for \purse $S$.}}
        \end{minipage} 
    \end{tabular}
    
    \begin{minipage}{0.2\textwidth}
        \vspace{1mm}
        \centering
        \includegraphics[width=\textwidth]{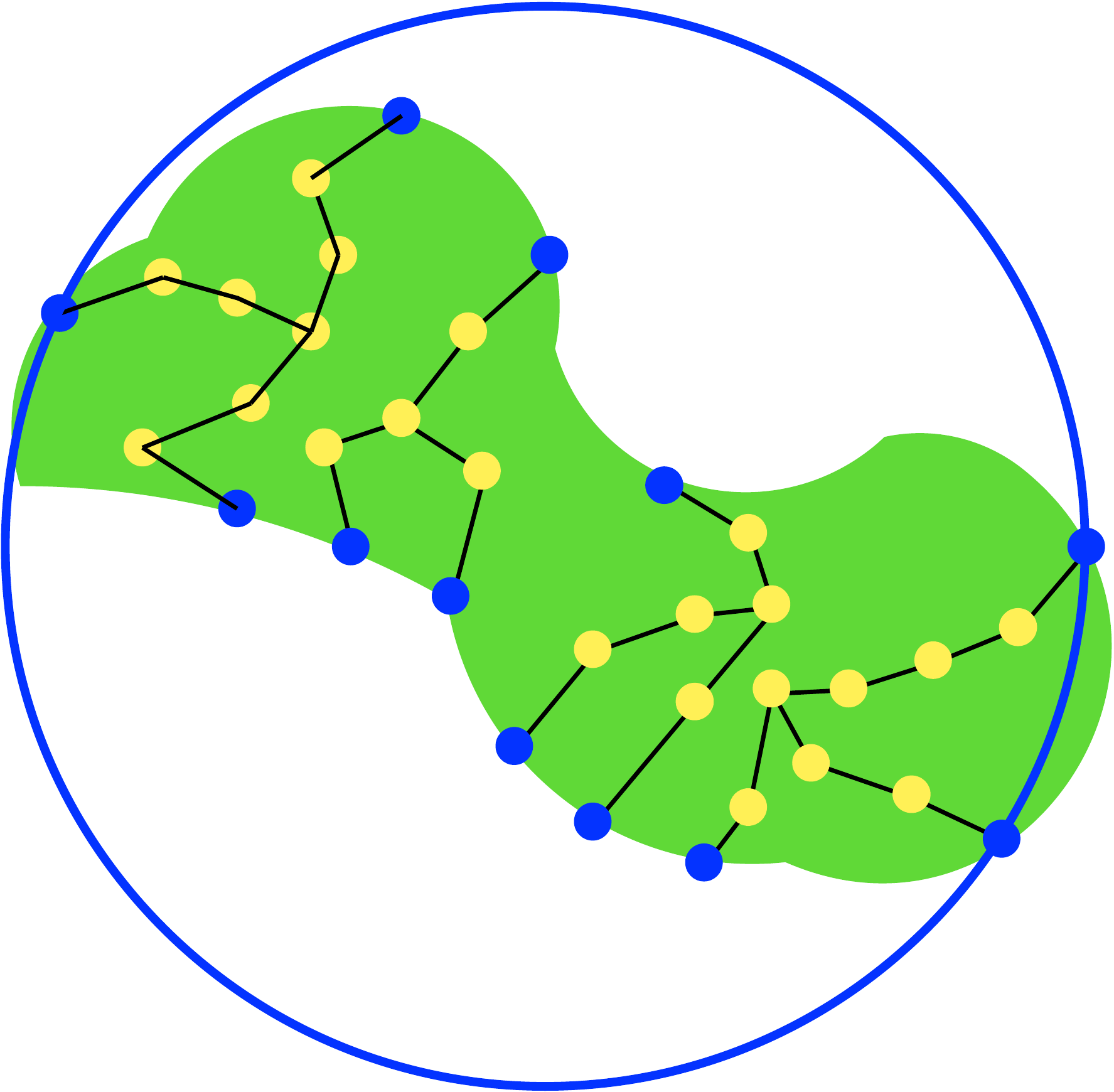}
        {{\smaller (c) Apply miniball algorithm to calculate inner approximation.}}
    \end{minipage}\\

    \end{center}
	\caption{{Overview of the \nameshort algorithm in 2D.}
	\label{fig:overview}}
\end{figure}

%% file: sections/alg-overview.tex
\setlength{\textfloatsep}{0pt}%
\begin{algorithm}[t]
\SetAlgoLined
    {\bf Input: } a pose uncertainty set {\purse} $S$; {init sample trial number $N_{\text{sample}}$}; simulation parameters $\mathcal{P}$; \\
    {\bf Output: } miniball center pose $s^*\in \SEthree$ ; radius of rotation miniball $\hatD$; radius of translation miniball $\hatd$\\

    $S_0 \gets$ init_sample($S$, $N_{\text{sample}}$);\\
    $\hat{\partial S_R} \gets$ sample_rotation_boundary($S_0$, $S$, $\mathcal{P}$);\\
    $\hat{\partial S_t} \gets$ sample_translation_boundary($S_0$, $S$, $\mathcal{P}$);\\
    $R^*, \hatD \gets$ minimum_enclosing_geodesic_ball($\hat{\partial S_R}$);\\
    $t^*, \hatd \gets$ minimum_enclosing_ball($\hat{\partial S_t}$);\\
    {\bf return:} $s^* \gets (R^*,t^*)$, $\hatD$, $\hatd$\\

    \caption{\nameshort Overview \label{alg:overview}}
\end{algorithm}

%% file: sections/alg-boundary-sampler.tex
\setlength{\textfloatsep}{0pt}%
\begin{algorithm}[t]
\SetAlgoLined
    {\bf Input: } initial poses $S_0\subset \SEthree$; \purse $S$; base angular velocity magnitude $\omega_0$; time step decay factor $\gamma$; translation perturbation scale $t_p$; random walk trial number $N_W$; iteration number $N_I$; perturbation number $N_P$; optimal perturbation number $N_P^*$; time step scaling number $N_T$;\\
    {\bf Output: } sampled boundary poses $\hat{\partial S_R} \subset \SOthree$ as an inner approximation of $\partial S_R$;\\

    $\bar{R} \gets \text{proj}_{\SOthree}( \sum_{(R_j,*) \in S_0} R_j )$; \label{line:Ravg}\\
    $\bar{t} \gets \frac{1}{\vert S_0 \vert } \sum_{(*,t_j) \in S_0} t_j $; \label{line:tavg} \\
    $\hat{\partial S_R} \gets \emptyset$; \\
    \For{ $(R_0, t_0) \in S_0$}{
        \For { $n \gets 1$ to $N_W$} {
            $\omega \gets$ init_angular_velocity$(R_0, \bar{R}, \omega_0)$; \\
            $R^*\gets R_0, t^* \gets t_0$; \\
            \% Iterate $N_I$ times so that the evolved pose gets close to $\partial S_R$\\
            \For { $i \gets 1$ to $N_I$} { \label{line:iterate}
                \% Randomize $N_P$ translation perturbations
                \For {$j \gets 1$ to $N_P$} {
                    $t_j \gets t^* + $ perturbation($t_p$); \\
                    $d_j \gets \text{dist}(R^*, t_j, \partial S)$; \\
                }
                \% Pick out $N_P^*$ perturbations that drag the pose away from $\partial S$\\
                $\{j_k\}_{k=1}^{N_P^*} \gets$ top_k_indices($\{d_j\}_{j=1}^{N_P}$, $N_P^*$); \\
                \For {$k \gets 1$ to $N_p^*$} {
                    \For {$ m \gets 1$ to $N_T$} {
                        $\Delta T \gets \gamma^{m-1}$ \\
                        $R_{km} \gets $ update_rotation($R^*$, $\omega$, $\Delta T$); \\
                        $I_{km} \gets $ in_purse($R_{km}, t_{j_k}$); \\
                    }
                }
                \% Find the optimal pose that is still in $S$ and has the maximum rotation movement\\
                $m_0 \gets \min\cbrace{m \mid \exists k \text{\ s.t.\ } I_{km} = 1}$; \\
                $k_0 \gets \cbrace{k \mid I_{km_0} = 1}$; \\
                $R^* \gets R_{k_0m_0}$; \\
                $t^* \gets t_{k_0m_0}$; \\
            }
            
            $\hat{\partial S_R} \gets \hat{\partial S_R} \cup R^*$; \\
        }
    }
    {\bf return:} $\hat{\partial S_R}$\\

    \caption{$\partial S_R$ sampler~\label{alg:sampleR}}
\end{algorithm}

%% file: sections/alg-parallel.tex

\begin{algorithm}[t]
    \SetAlgoLined
    {\bf Input: } initial poses $S_0\subset \SEthree$; \purse $S$; base angular velocity magnitude $\omega_0$; time step decay factor $\gamma$; translation perturbation scale $t_p$; random walk trial number $N_W$; iteration number $N_I$; perturbation number $N_P$; optimal perturbation number $N_P^*$; time step scaling number $N_T$;\\
    {\bf Output: } sampled boundary poses $\hat{\partial S_R} \subset \SOthree$ as an inner approximation of $\partial S_R$;\\
    $\omega \gets$ init_angular_velocity($S_0, \omega_0, N_W$); \\
    $v \gets$ init_center_velocity($S_0, v_0, N_W$); \\
    $(R^*, t^*) \gets$ repeat($S_0, N_W$); \\
    $\Delta T \gets (1, \beta, \beta^2, \cdots, \beta^{N_T-1})$; \\

    \For { $i\gets 1$ to $N_I$} {
        $R \gets$ repeat($R^*, N_P$); \\
        $t \gets$ repeat($t^*, N_P$) + perturbation($t_p, |S_0|N_WN_P$); \\
        $d \gets$ dist($R, t, \partial S$); \\ \label{line:dist_gpu}
        $j \gets$ top_k_indices($d, N_P^*$); \\
        $\tilde{t} \gets$ repeat($t_{[j]}, N_T$); \\
        $\tilde{R}\gets $ update_rotation(repeat($R^*, N_P^*$)$, \omega, \Delta T$); \\
        $I \gets$ in_purse($\tilde{R}, \tilde{t}$); \\ \label{line:in_purse_gpu}
        $R^*, t^* \gets$ find_farthest_rotation($\tilde{R}, \tilde{t}, I$); \\
    }
    {\bf return:} $\hat{\partial S_R}\gets R^*$\\
    \caption{$\partial S_R$ parallel sampler~\label{alg:sampleR_parallel}}
\end{algorithm}

%% file: sections/experiments.tex

\section{Experiments}
\label{sec:experiments}

We test \nameshort on {three} real datasets, the \lmo dataset~\cite{brachmann14eccv-lmo} for object pose estimation (Example~\ref{ex:2D3D}), the \threedmatch dataset~\cite{zeng17cvpr-3dmatch} for point cloud registration (Example~\ref{ex:3D3D}) {and the \lm dataset~\cite{hinterstoisser2013model} for pose regression (Example~\ref{ex:poseregression})}. Since \nameshort takes in the {\purse} description, we briefly describe how we process the datasets to obtain \purse descriptions.

\begin{itemize}
    \item \lmo dataset. The \lmo test dataset contains 1214 images each capturing 8 different objects on a table and the goal here to estimate the 6D pose of each object while quantifying uncertainty. We leverage the pretrained semantic keypoint detector~\cite{pavlakos17icra-6D} to detect 2D semantic keypoints $\{z_i\}_{i=1}^N$ of each object that are matched to the manually labeled 3D keypoints $\{Z_i\}_{i=1}^N$ ($N$ is around 10). To calibrate the uncertainty of the 2D keypoint detections, we follow the conformal prediction approach in~\cite{yang23cvpr-object}, which produces descriptions of pose uncertainty in the form of~\eqref{eq:PURSE2D3D}. 
    \item \threedmatch dataset. The \threedmatch test dataset includes 1623 pairs of point clouds and the goal here is to estimate the 6D rigid transformation between each pair of clouds while quantifying uncertainty. We leverage the pretrained \dgr network~\cite{choy2020deep} that detects salient keypoint matches $\{a_i,b_i \}_{i=1}^K$ between each pair of clouds. To calibrate the uncertainty of the keypoint matches, we design a similar conformal prediction procedure as~\cite{yang23cvpr-object} with $400$ pairs in the calibration dataset. The details of the conformal prediction design are presented in \supp. After conformal prediction, pose uncertainty is given in the form of~\eqref{eq:PURSE3D3D} with $N=50$.
    \item {\lm dataset. We use 15787 images from 13 objects in the \lm dataset. The goal here is to estimate the 6D pose of each object while quantifying the uncertainty. We leverage the pretrained network FoundationPose~\cite{wen24cvpr-foundationpose} to directly output pose hypotheses. To calibrate the uncertainty of the output poses, we design a conformal prediction procedure leveraging the top 10 pose hypotheses and their scores. After conformal prediction, the pose uncertainty set is given in the form of~\eqref{eq:PURSEregression}.}
\end{itemize}

\subsection{Effectiveness of the Boundary Sampler}

A key idea in \nameshort is that it tries to apply strategic random walks to sample the boundary of the \purse. Here we investigate how effective is this boundary sampling strategy on two examples. To do so, we compare \nameshort with pure \ransag, \ie sampling inside \purse without any motivation to sample the boundary.
We choose one example in the \lmo dataset and one example in \threedmatch dataset. We compare {\nameshort} with pure \ransag for 100000 trials, {where \nameshort runs approixmately $0.2$ second and \ransag runs 1 second.}
Fig.~\ref{fig:sample_comp} plots the sampling results of the two algorithms. We use the stereographic projection to visualize the rotation samples (\ie we convert rotations to unit quaternions and use stereographic projection to represent unit quaternions on a 3D sphere). {Note that although \nameshort generates fewer samples,} we can clearly see that the samples from {\nameshort} are more spread-out than those from \ransag only. For this reason, the geodesic balls enclosing the {\nameshort} samples are larger than the geodesic balls enclosing the \ransag samples. 

\input{sections/fig-sample_comp.tex}

\subsection{Certificate of Approximation Tightness}

We then perform a large-scale experiment of uncertainty quantification on all the test samples of \lmo, \threedmatch, and \lm (excluding the calibration samples), where we compare the performance of \nameshort with other baseline algorithms. Our algorithm \nameshort runs on a single CPU with a RTX 3090 GPU, while the other baselines run on a workstations with 128 x AMD Ryzen Threadripper PRO 5995WX 64-Cores CPUs. Specifically, \ransag and \ransagfmincon run with MATLAB Parallel Computing Toolbox.

\subsubsection{\lmo} 
We implement two versions of \nameshort
\begin{itemize}
    \item \nameshort: the default \nameshort algorithm with parameters {$N_{\text{sample}} = 1500$}, $\omega_0 = 1$, $v_0 = 2$, $\gamma = 0.5$, $R_p = 0.2$, $t_p = 0.1$, $N_W = 2$, $N_I = 5$, $N_p=150$, $N_T = 15$.
    \item \nameacc: the more accurate algorithm with parameters {$N_{\text{sample}} = 1500$}, $\omega_0 = 1$, $v_0 = 2$, $\gamma = 0.5$, $R_p = 0.2$, $t_p = 0.1$, $N_W = 10$, $N_I = 10$, $N_p=150$, $N_T = 15$.
\end{itemize}
We implement the following baselines
\begin{itemize}
    \item \ransag: pure \ransag sampling with the miniball algorithm. We choose the number of pure \ransag trials so that its runtime is roughly one second on the workstation with parallel workers.
    \item \ransagfmincon: to encourage the \ransag initial samples to walk to the boundary of \purse, we explicitly solve an optimization problem
    \bea \label{eq:fmincon-problem}
    \text{Problem}(i) = \max_{(R,t)\in S} \|z_i - \Pi(RZ_i + t)\|_{\Lambda_i}
    \eea  
    where $z_i$ is the detected 2D key point, $Z_i$ is the 3D key point, and maximizing the objective explicitly pushes the rigid body system in Fig.~\ref{fig:purse}(a) to hit the boundary of the $i$-th ellipse. The algorithm proceeds as in Algorithm~\ref{alg:fmincon-miniball} and we solve the optimization problems using Matlab function \texttt{fmincon} with parallel workers.
    \item \grcc: the algorithm from~\cite{tang23arxiv-uncertainty} that computes an outer approximation of the \megb of \purse with relaxation order $\kappa = 2$ for translation and $\kappa = 3$ for rotation. It provides an outer approximation of the \megb (\ie $\bar{D}$ and $\bar{d}$), which is used to calculate $\eta_R$ and $\eta_t$ for other methods.
\end{itemize}

We investigate the runtime of these five algorithms, and more importantly, study how tight the inner and outer approximations of the \megb are. Specifically, \grcc is the only algorithm that produces upper bounds on the size of the true \megb. Let $\bar{D}$ and $\bar{d}$ be the radii of the enclosing geodesic balls produced by \grcc on $\SOthree$ and $\Real{3}$, respectively. The other four algorithms all produce lower bounds on the size of the true \megb. Let $\hatD$ and $\hatd$ be the radii of the geodesic balls produced by these inner approximation algorithms. Define the \emph{relative ratios} 
\bea\label{eq:relative-ratio}
\eta\rotsub := \frac{\hatD}{\bar{D}}\leq 1, \quad \eta\transub := \frac{\hatd}{\bar{d}}\leq 1.
\eea
Clearly, $\eta\rotsub = 1$ (resp. $\eta\transub = 1$) certifies that the true \megb has been found, and the closer $\eta\rotsub, \eta\transub$ are w.r.t $1$ the tighter the inner approximations are. Therefore, $\eta\rotsub,\eta\transub$ produce certificates of approximation tightness.

\input{sections/alg-fmincon.tex}

{\bf Results}. Table~\ref{table:2D3D-performance} summarizes the average runtime and the relative rotation and translation ratios. We observe that (i) \nameshort runs below $0.2$ seconds on average and attains $\eta\rotsub > 0.92, \eta\transub > 0.97$, making it the method of choice for real-time uncertainty quantification; (ii) with more runtime budget, \nameacc runs below 1 second but boosts the relative ratios to almost $1$; (iii) \ransag, without the strategic random walks to sample the boundary, runs slower with worse relative ratios, again showing the value of the \purse boundary sampler. (iv) \ransagfmincon attains the best $\eta\rotsub$ but the worst $\eta\transub$, {The reason is two-fold. First, from the set shape perspective, using the keypoint-based algorithm leads to information loss along the $z$-direction (depth direction), which causes extremely large uncertainty of the set $S_t$ along the $z$-direction. As depicted in Fig.~\ref{fig:sample_comp}, $S_t$ shapes like an extremely thin ellipsoid, which makes sampling the translation points near its two ends very difficult. 
Second, from the algorithm perspective, the \ransagfmincon algorithm is designed to push one of the keypoints to the boundary using any possible means. (a) In this sense, we can only guarantee that the pose $(R,t)$ is on the boundary of the \purse set. However, $\partial S \neq \partial S_R \times \partial S_t$. And this intuition isn't encoded in the \ransagfmincon algorithm. (b) Because the \ransagfmincon algorithm only seeks one way to push the key point to the boundary. So in general it will find the simplest way to push the key point to the boundary. Thus, it's easier to do more rotation than to use the translation at two ends of $S_t$. This makes the algorithm perform the worst in translation compared to the other algorithms.}
(v) The \grcc algorithm, despite being the only algorithm that can produce outer approximations, is too slow for real-time applications. We believe the results in Table~\ref{table:2D3D-performance} show that \nameshort and \nameacc can be real-time alternatives of the \grcc because the amount of underestimation of the uncertainty is very minor.

\begin{table}[h]
    \centering
    \caption{Performance of 5 methods on \lmo (Example~\ref{ex:2D3D}) \label{table:2D3D-performance}}
    \vspace{-2mm}
    \begin{tabular}{|c|c|c|c|}
        \hline
                    & avg runtime (seconds) & avg $\eta\rotsub$  & avg $\eta\transub$ \\
        \hline
        \nameshort & $\bm{0.1879}$ & $0.9280$ & $0.9781$ \\
        \hline
        \nameacc & $0.7563$ & $0.9660$ & $\bm{0.9918}$ \\
        \hline
        \ransag & $1.8368$ & $0.8916$ & $0.8791$ \\
        \hline
        \ransagfmincon & $6.1095$ & $\bm{0.9726}$ & $0.8119$ \\
        \hline
        \grcc & $74.9724$ & $\backslash$ & $\backslash$ \\
        \hline
    \end{tabular}
\end{table}

\subsubsection{\threedmatch}

We implement two versions of \nameshort
\begin{itemize}
    \item \nameshort: {$N_{\text{sample}} = 1500$}, $\omega_0 = 0.5$, $v_0 = 0$, $\gamma = 0.5$, $R_p = 0$, $t_p = 0.1$, $N_W = 2$, $N_I = 5$, $N_p=150$, $N_T = 15$.
    \item \nameacc: {$N_{\text{sample}} = 1500$}, $\omega_0 = 0.5$, $v_0 = 0$, $\gamma = 0.5$, $R_p = 0$, $t_p = 0.1$, $N_W = 20$, $N_I = 10$, $N_p=150$, $N_T = 15$
\end{itemize}

We implement the same \ransag and \grcc baselines as in the \lmo case. {In the new implementation of the \ransagfmincon algorithm under \pursethree settings, we leverage the constraint pruning technique in \cite{tang23arxiv-uncertainty} to speed up the computation of the algorithm. 

Since the number of constraints in \pursethree is much larger than that in \pursetwo, we limit the $|\calI|$ to $50$ (Setting I), and $80$ (Setting II), which will lead to faster computation but worse ratio than using all of the \ransag results.} We then investigate the runtime of these algorithms as well as the certifcates of approximation tightness as defined in~\eqref{eq:relative-ratio}.

{\bf Results}. Table~\ref{table:3D3D-performance} summarizes the average runtime and the relative rotation and translation ratios. We observe that (i) \nameshort runs below $0.2$ seconds on average and attains $\eta\rotsub > 0.91, \eta\transub > 0.93$, which benefits real-time uncertainty quantification. (ii) \nameacc consumes more time but achieves the most accurate relative ratios. (iii) \ransag runs slow and has the worst relative ratios. (iv) \grcc is too slow for real-time applications. (v) {\ransagfmincon algorithm fails to achieve tight approximation ratios in both rotation and translation even with extremely high runtime. We suspect there are two reasons for this. (a) \scenario{fmincon} runs slow when the number of constraints is large, which makes the result worse when runtime is limited. (b) \ransagfmincon pushes one of the keypoints to $\partial S$, but it doesn't guarantee to touch $S_R$ and $S_t$. Under \pursethree settings, the number of keypoints is so large that it's more difficult to find the one that touches $\partial S_R$ and $\partial S_t$.}

\begin{table}[h]
    \centering
    \caption{Performance of {6} methods on \threedmatch (Example~\ref{ex:3D3D}) \label{table:3D3D-performance}}
    \vspace{-2mm}
    \begin{tabular}{|c|c|c|c|}
        \hline
                    & avg runtime (seconds) & avg $\eta\rotsub$ & avg $\eta\transub$ \\
        \hline
        \nameshort & $\bm {0.1774}$ & $0.9140$ & $0.9364$ \\
        \hline
        \nameacc & $1.3362$ & $\bm {0.9318}$ & $\bm {0.9563}$ \\
        \hline
        \ransag & $1.2820$ & $0.8480$ & $0.8689$ \\
        \hline
        {\ransagfmincon I} & {$32.0642$} & {$0.7254$} & {$0.7518$} \\
        \hline
        {\ransagfmincon II} & {$43.2750$} & {$0.8323$} & {$0.8600$} \\
        \hline
        \grcc & $147.8408$ & $\backslash$ & $\backslash$ \\
        \hline
    \end{tabular}
\end{table}

{
    \subsubsection{\lm}

    We implement \nameshort using the parameters listed as follows:
    \begin{itemize}
        \item \nameshort: {$N_{\text{sample}} = 200$}, $\omega_0 = 0.5$, $v_0 = 0$, $\gamma = 0.5$, $R_p = 0$, $t_p = 0.1$, $N_W = 20$, $N_I = 5$, $N_p=150$, $N_T = 10$.
    \end{itemize}

    We implement \grcc as in the \lmo and \threedmatch cases.  We then investigate the runtime of the \nameshort algorithm as well as the certifcates of approximation tightness as defined in~\eqref{eq:relative-ratio}. The random sampling baseline (like \ransag) is not included here because we only sample from the ``convex hull'' defined by pose hypotheses, so it is unlikely the random samples can possibly filled the whole space.

    {\bf Results}. Table~\ref{table:foundation-performance} summarizes the average runtime and the relative rotation and translation ratios. We observe that (i) \nameshort runs below $0.3$ seconds on average and attains $\eta_R > 0.96, \eta_t > 0.99$. (ii) \grcc is relatively faster than previous examples (Example~\ref{ex:2D3D} and Example~\ref{ex:3D3D}) because: (a) we separately compute the bound for rotation and translation, which reduces the size of the optimization problem in \grcc, (b) the constraints are simpler. However, it's still too slow for real-time applications.

    \begin{table}[h]
        \centering
        \caption{{Performance of 2 methods on \lm (Example~\ref{ex:poseregression})} \label{table:foundation-performance}}
        \vspace{-2mm}
        \begin{tabular}{|c|c|c|c|}
            \hline
                        & avg runtime (seconds) & avg $\eta\rotsub$  & avg $\eta\transub$ \\
            \hline
            \nameshort & $\bm{0.2768}$ & $0.9659$ & $0.9909$ \\
            \hline
            \grcc & $6.5405$ & $\backslash$ & $\backslash$ \\
            \hline
        \end{tabular}
    \end{table}

}

\subsection{Time Decomposition}

We implement two versions of \nameshort
\begin{itemize}
    \item \nameshort: {$N_{\text{sample}} = 1500$}, $\omega_0 = 0.5$, $v_0 = 0$, $\gamma = 0.5$, $R_p = 0$, $t_p = 0.1$, $N_W = 2$, $N_I = 5$, $N_p=150$, $N_T = 15$.
    \item \nameacc: {$N_{\text{sample}} = 1500$}, $\omega_0 = 0.5$, $v_0 = 0$, $\gamma = 0.5$, $R_p = 0$, $t_p = 0.1$, $N_W = 20$, $N_I = 10$, $N_p=150$, $N_T = 15$
\end{itemize}

We analyze the time decomposition of \nameshort. Our algorithm consists of three steps as described in Section~\ref{sec:algorithm}: (1) \ransag sampling, (2) strategic random walk, (3) miniball. The result of the time decomposition in Example~\ref{ex:2D3D} and Example~\ref{ex:3D3D} are shown in Table~\ref{table:time-decomposition}. We see that for \nameshort, the time of the \ransag sampling is comparable to that of strategic random walk, and the miniball time is negligible.

\begin{table}[h]
    \centering
    \caption{Time decomposition of \nameshort and \nameacc.
    \label{table:time-decomposition}
    }
    \vspace{-3mm}
    \begin{tabular}{|c|c|c|c|}
        \hline
        & \ransag & Random walk & Miniball  \\
        \hline
        Example~\ref{ex:2D3D} \nameshort & $0.0900$ & $0.0915$ & $0.0064$ \\
        \hline
        Example~\ref{ex:2D3D} \nameacc & $0.0869$ & $0.6235$ & $0.0459$ \\
        \hline
        Example~\ref{ex:3D3D} \nameshort & $0.0883$ & $0.0838$ & $0.0053$ \\
        \hline
        Example~\ref{ex:3D3D} \nameacc & $0.0874$ & $1.1997$ & $0.0491$ \\
        \hline

    \end{tabular}
\end{table}

\subsection{Ablation: Sensitivity to Parameters}

We study how the parameters of \nameshort impact its efficiency, \ie how the performance (in terms of relative ratios) and runtime changes with the input parameter set. We present the results for sampling the rotation boundary $\partial S_R$ of~\eqref{eq:PURSE3D3D} and refer the interested reader to \supp~for more results.

Specifically, though a large parameter set is available to tune, the parameters that have the largest impact on the efficiency of \nameshort are iteration step number $N_I$ and number of the parallel workers $N_W, N_T, N_P, N_P^*$. We initialize the parameter set with $N_I=5, N_W=20, N_T=15, N_P=150, N_P^*=10, \omega=0.5, \beta=0.5, t_p=0.1$. When setting new values to each parameter, the rest of the parameters are kept the same as the initial values. During the study, we fix the \ransag output of each experiment and compare the performance of different parameter values. The results are shown in Fig.~\ref{fig:param-sweep}, which demonstrate that \nameshort's performance is quite robust to parameter tuning. For all the parameter choices, \nameshort achieves relative ratio $\eta\rotsub > 90\%$ and its runstime is around 1 second.

\input{sections/fig-parameter-sweep.tex}

It is clear that both runtime and $\eta_R$ grows as the all numbers increase, especially when the numbers start from a small value. However, as $N_I$ and $N_W$ grow higher, $\eta_R$ saturates around $0.93$. This might be a result that the outer approximation $\bar{D}$ provided by \grcc is not tight enough. Therefore, even if the sampler fully explores $\partial S_R$, the relative ratio will not reach $1$. In the mean while, the sensitivity is not significant as the $\eta_R$ stays above $0.9$ for a wide range of parameter values.

\subsection{How Large is the Pose Uncertainty?}

The previous results focused on showing the relative ratios that are informative for the approximation performance of \nameshort. In Fig.~\ref{fig:cdf_plot}, we show the cumulative distribution of absolute sizes of the inner approximations of \megb, just so the reader is aware of how much uncertainty is induced from the calibrated noise bounds on the learned measurements generated by modern neural networks. {We can draw two observations from Fig.~\ref{fig:cdf_plot}. (i) Comparing the two keypoint-based methods, \ie \cite{pavlakos17icra-6D} for \lmo and \dgr~\cite{choy2020deep} for \threedmatch, \dgr has lower uncertainty than \cite{pavlakos17icra-6D}. There are two possible reasons for this: (a) 3D-3D keypoint matches better constrain the pose hypothesis space than 2D-3D keypoint matches; (b) \dgr was proposed three years after \cite{pavlakos17icra-6D} and it is plausible that \dgr is better trained. (ii) Comparing the direct pose regression paradigm with the keypoint-based paradigm, we see that direct pose regression, in particular the state-of-the-art FoundationPose model, has much smaller uncertainty. The rotation uncertainty in FoundationPose is around 4 degrees and the translation uncertainty in FoundationPose is around 8mm. These are almost one order of magnitude better than \dgr~\cite{choy2020deep} and \cite{pavlakos17icra-6D}. As far as we know, this is the first time such small calibrated pose uncertainty is reported in the literature.}   

\input{sections/fig-cdf.tex}

%% file: sections/fig-sample_comp.tex
\begin{figure}[h]
	\begin{center}
		\begin{tabular}{cc}
            \hspace{-10mm}	
			\begin{minipage}{0.2\textwidth}
				\centering
				\includegraphics[width=0.9\textwidth]{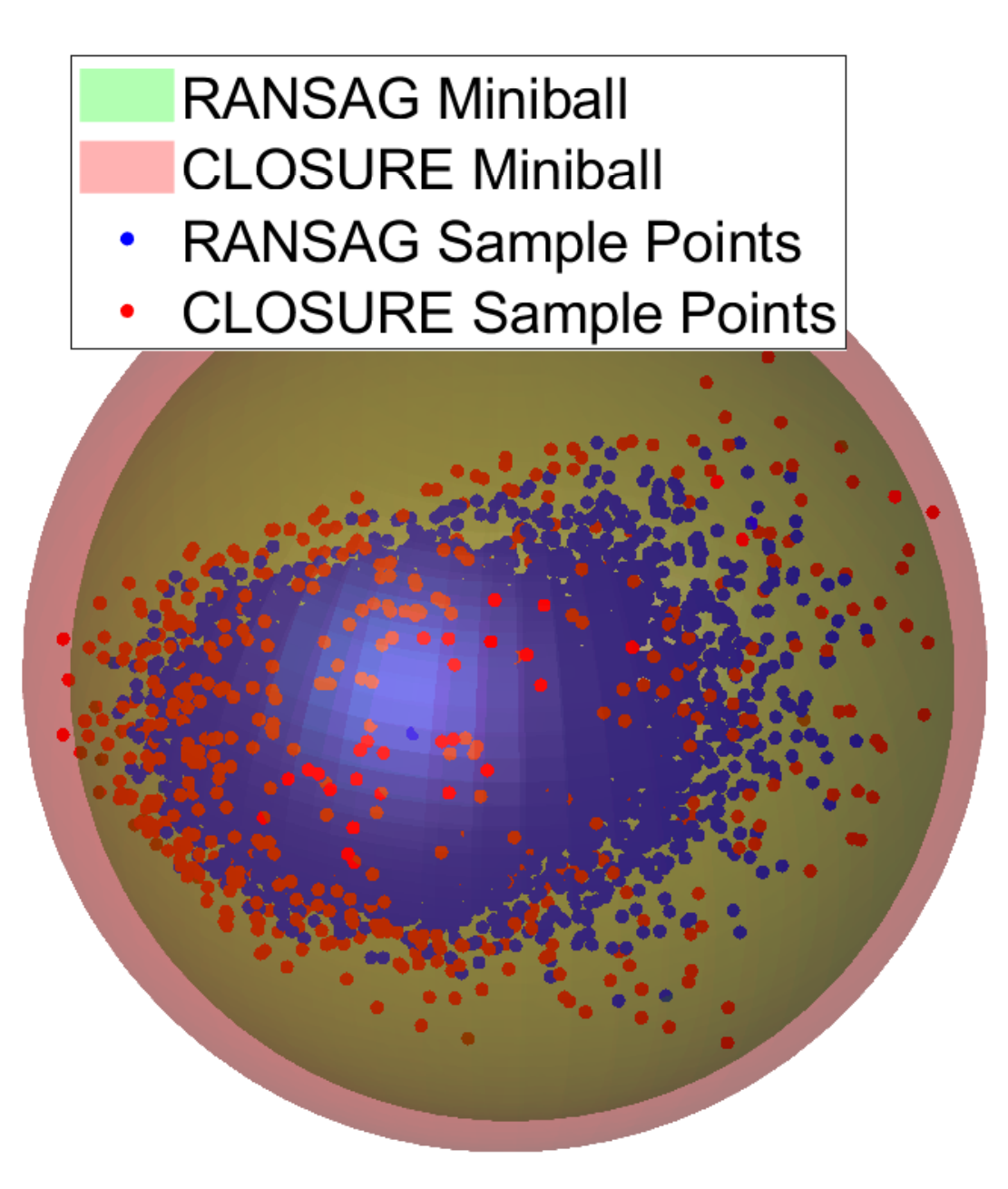}
			\end{minipage}
			&\hspace{-6mm}
			\begin{minipage}{0.2\textwidth}
				\centering
				\includegraphics[width=\textwidth]{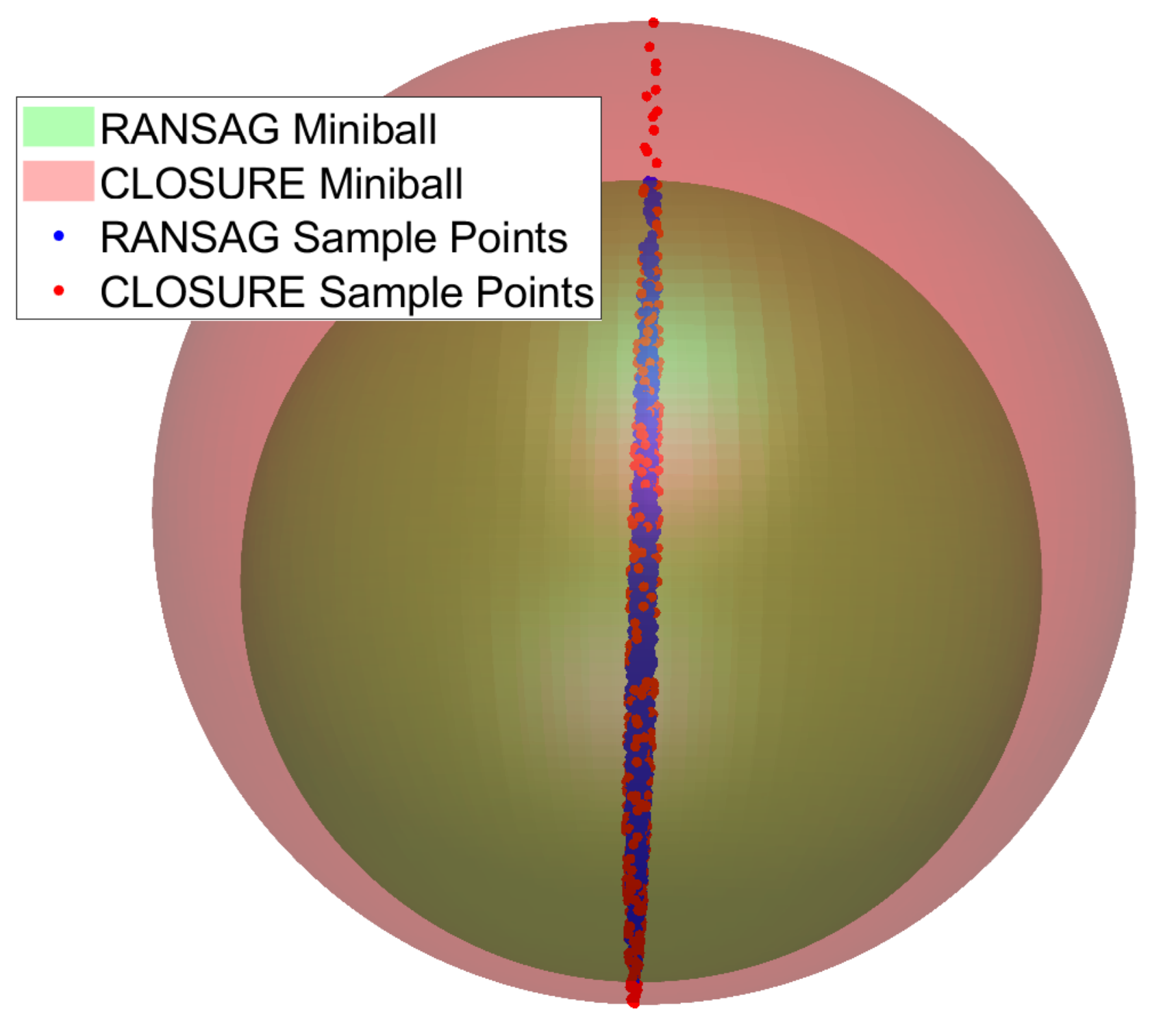}
			\end{minipage}\\
			\multicolumn{2}{c}{{\small (a) An example from \lmo.}}\\
			\hspace{-10mm}	
			\begin{minipage}{0.2\textwidth}
				\centering
				\includegraphics[width=\textwidth]{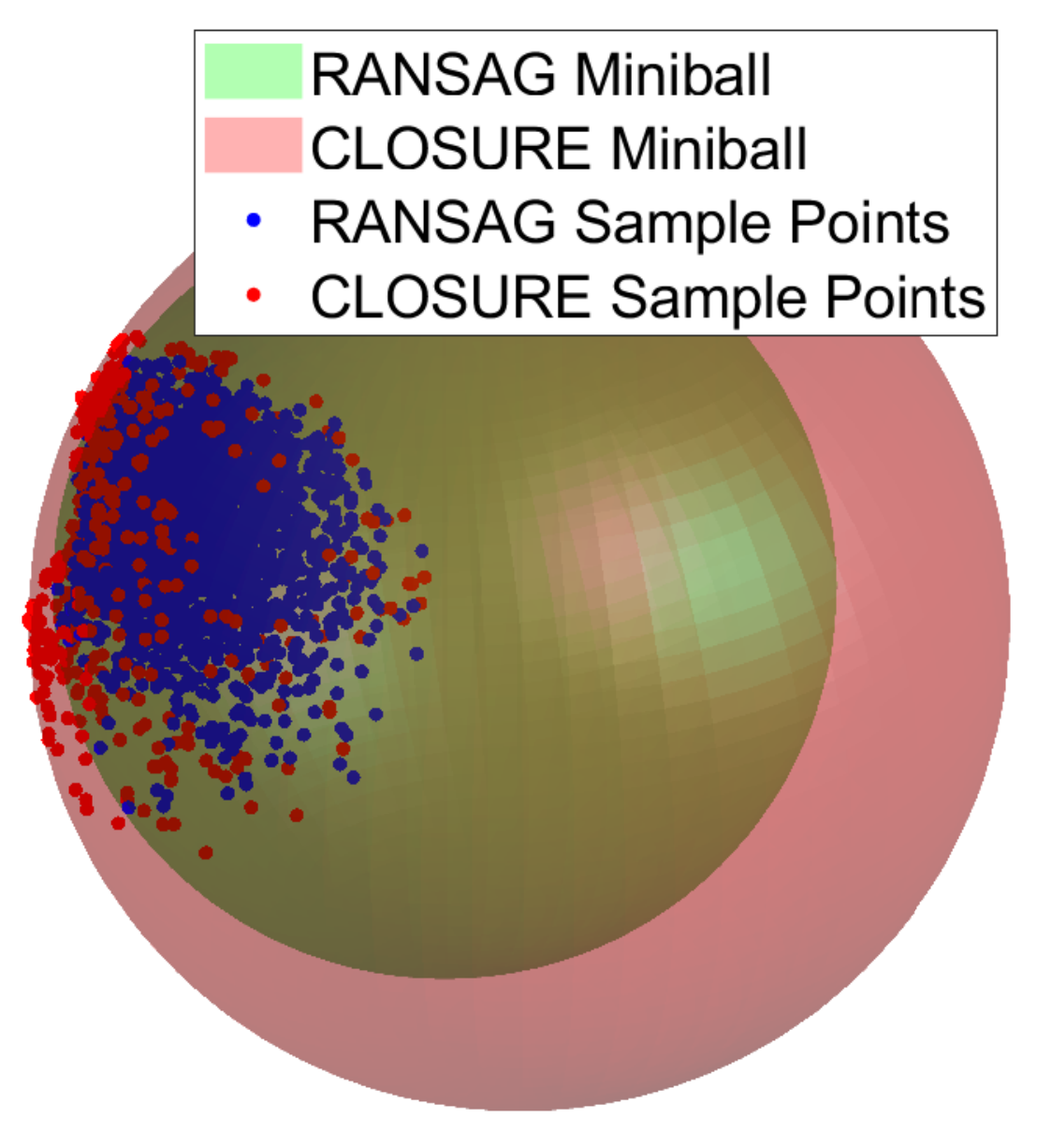}
			\end{minipage}
			&\hspace{-6mm}
			\begin{minipage}{0.2\textwidth}
				\centering
				\includegraphics[width=\textwidth]{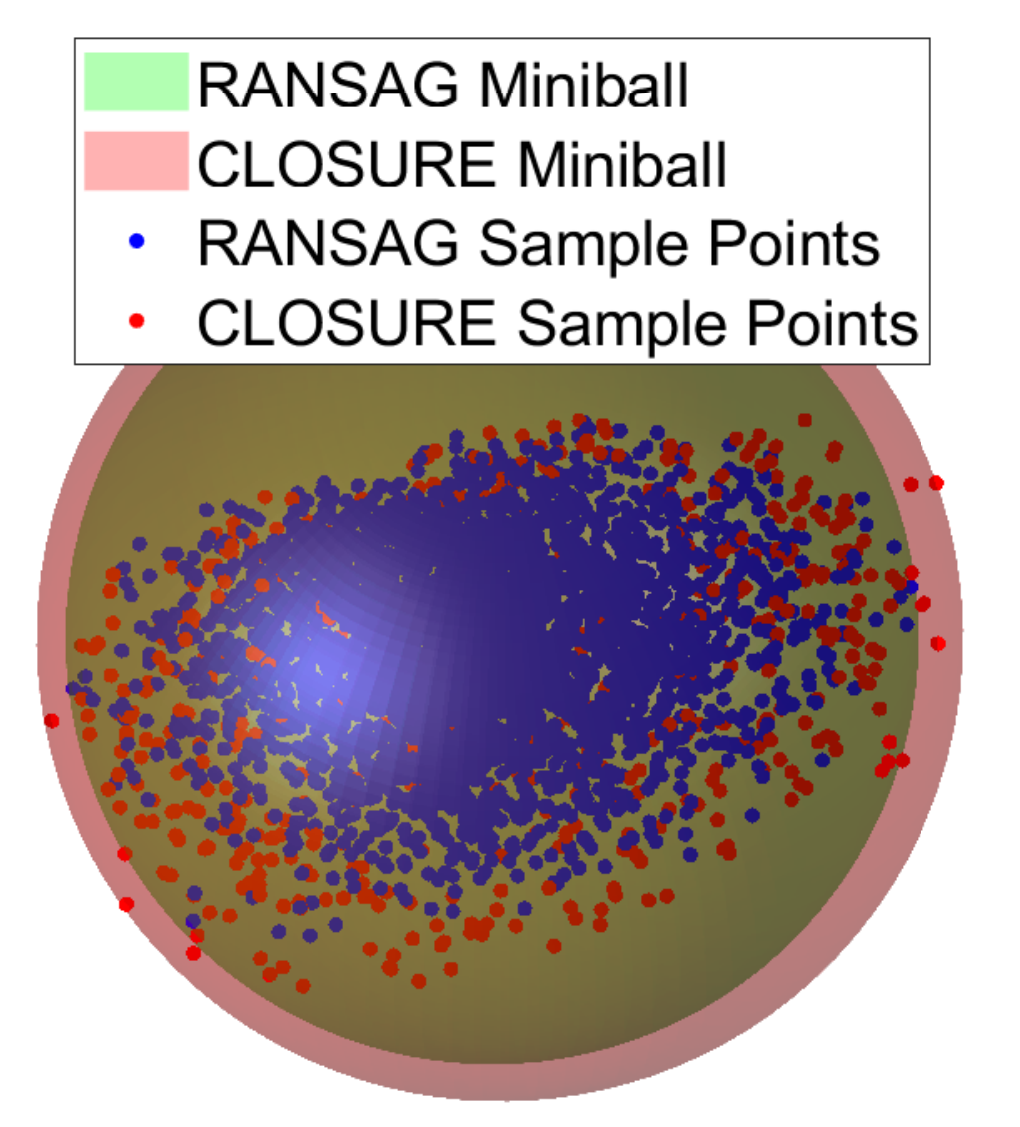}
			\end{minipage}\\
			\multicolumn{2}{c}{{\small (b) An example from \threedmatch.}}
        \end{tabular}
	\end{center}
	\vspace{-2mm}
	\caption{{Effectiveness of \purse boundary sampler on (a) an example from \lmo, and (b) an example from \threedmatch. Left: rotation, Right: translation.} \label{fig:sample_comp}}
\end{figure}

%% file: sections/alg-fmincon.tex

\begin{algorithm}[t]
\SetAlgoLined
    {\bf Input: } Calibration dataset $\calD$, a pose uncertainty set {\purse} $S$; \ransag trial number {$N_{\text{sample}}$}; key point search index set $\calI$; \\
    {\bf Output: } sampled boundary poses $\hat{\partial S} \subset \SEthree$ in $\purse$; center pose $s^*\in \SEthree$; radius of minimum enclosing geodesic ball $\hatD$; radius of minimum enclosing ball $\hatd$\\

    $S_0 \gets$ \ransag($S$, {$N_{\text{sample}}$});\\
    $\hat{\partial S} \gets S_0$;\\
    \For{i $\in \calI$}{
        \For{$(R_0, t_0) \in S_0$}{
            $(R,t)\gets$ fmincon(Problem($i$), init:$(R_0, t_0)$);\\
            \If{$(R,t)\in S$}{
                $\hat{\partial S} \gets \hat{\partial S} \cup \{(R,t)\}$;
            }
        }
    }
    $R^*, \hatD \gets$ minimum_enclosing_geodesic_ball($(\hat{\partial S})|_{\SOthree}$);\\
    $t^*, \hatd \gets$ minimum_enclosing_ball($(\hat{\partial S})|_{\R^3}$);\\
    {\bf return:} $\hat{\partial S}$, $s^* \gets (R^*,t^*)$, $\hatD$, $\hatd$\\

    \caption{\ransagfmincon \label{alg:fmincon-miniball}}
\end{algorithm}

%% file: sections/fig-parameter-sweep.tex

\begin{figure}[t]
	\begin{center}
		\begin{tabular}{c}
			\begin{minipage}{0.8\columnwidth}
				\centering
				\includegraphics[width=\textwidth]{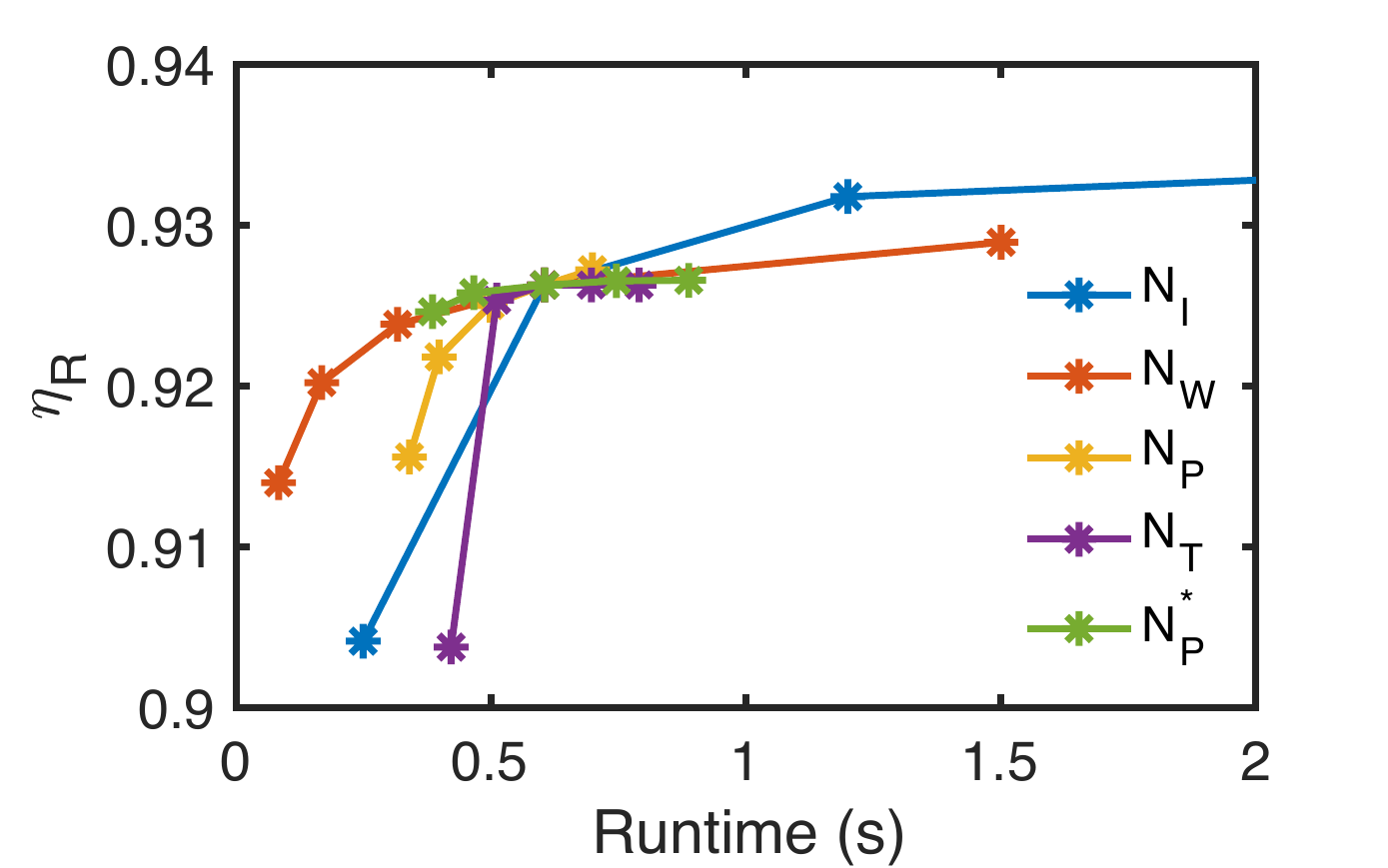}\\
			\end{minipage}
		\end{tabular}
	\end{center}
	\vspace{-2mm}
	\caption{The ratio-runtime curve of different parameters. $N_I\in\{2, 5, 10, 20\}$, $N_W\in\{2, 5, 10, 20, 50\}$, $N_P\in\{20, 50, 100, 150, 200\}$, $N_T\in\{5, 10, 15, 20, 25\}$, $N_P^*\in\{2, 5, 10, 15, 20\}$
	\label{fig:param-sweep}}
\end{figure}

%% file: sections/fig-cdf.tex

\begin{figure}[h]
	\begin{center}
		\begin{tabular}{c}
            \hspace{-10mm}	
			\begin{minipage}{0.40\textwidth}
				\centering
				\includegraphics[width=\textwidth]{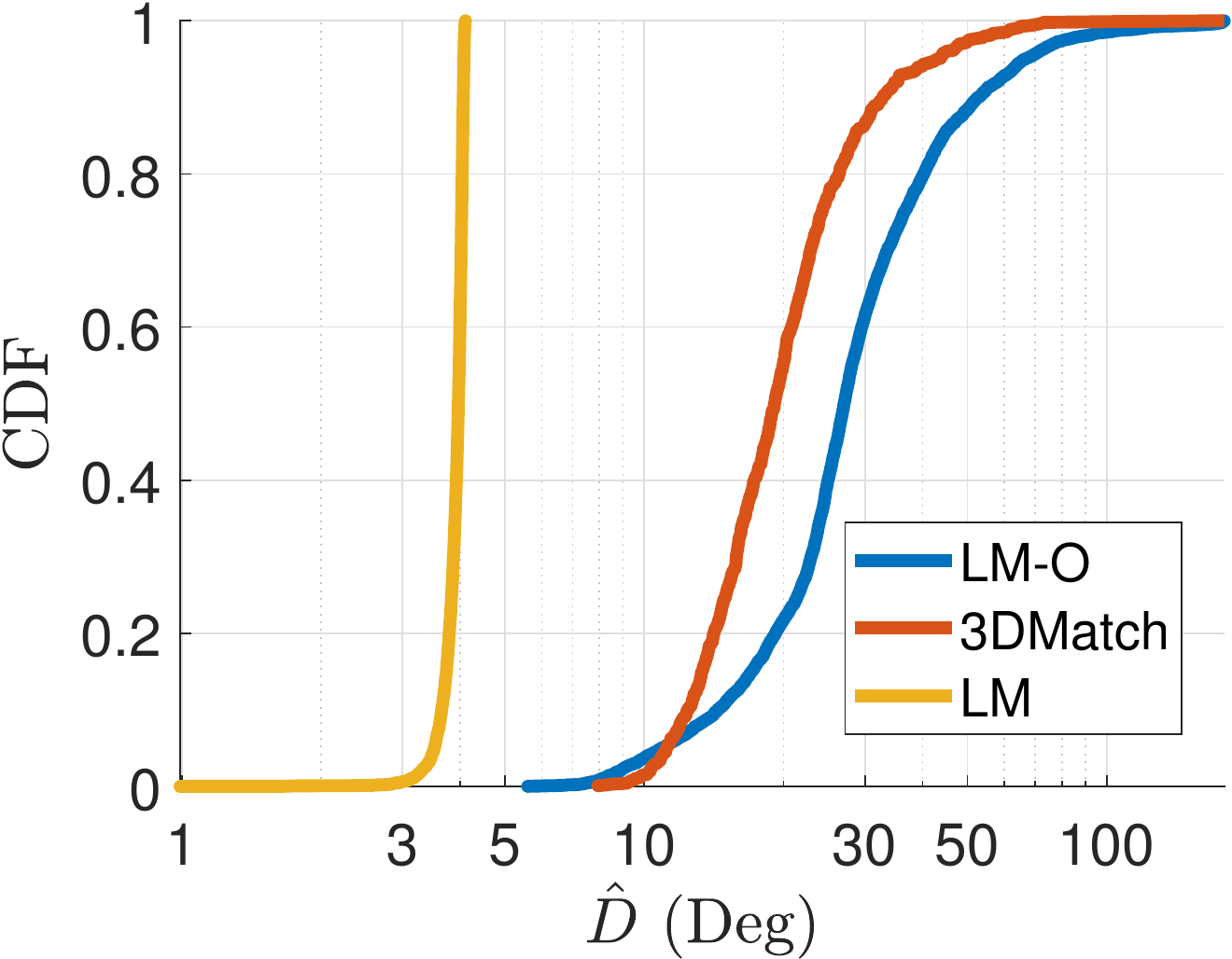}
			\end{minipage}
			\\
            \hspace{-10mm}
			\begin{minipage}{0.40\textwidth}
				\centering
				\includegraphics[width=\textwidth]{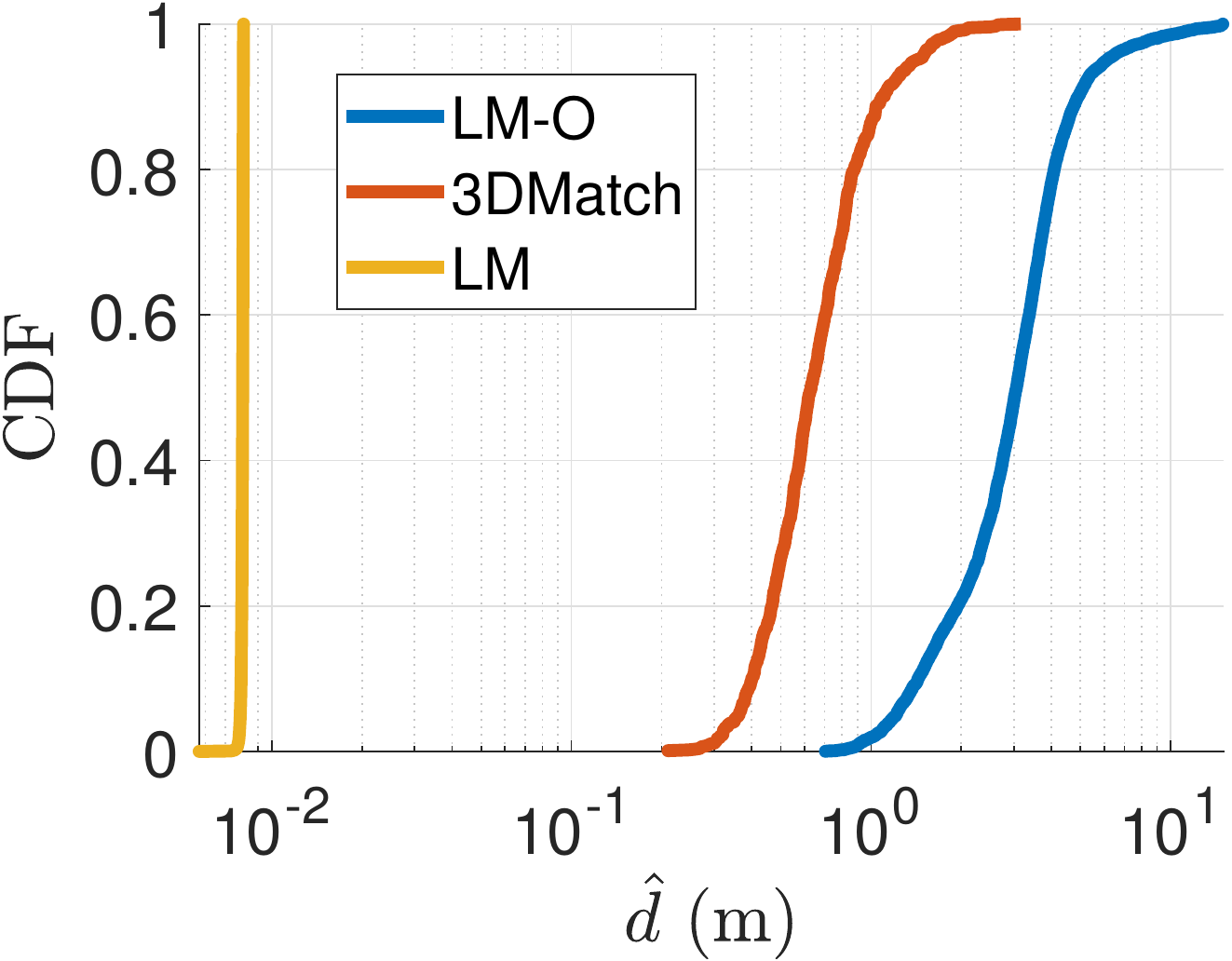}
			\end{minipage}
        \end{tabular}
	\end{center}
	\caption{{Cumulative distributions of the absolute radii of the enclosing geodesic balls produced by \nameshort for Examples~\ref{ex:2D3D}-\ref{ex:poseregression}.
    Top: rotation, bottom: translation. Note the $x$-axis is plotted in $\log$ scale. \label{fig:cdf_plot}}}
\end{figure}

%% file: sections/conclusion.tex
\section{Conclusion} 
\label{sec:conclusion}

We introduced \nameshort, a GPU-accelerated fast algorithm that can quantify pose uncertainty in real time from learned {noisy measurements such as keypoints and pose hypotheses}. The key perspective that led to the design of \nameshort is that the pose uncertainty set (\purse), despite being algebraically unstructured, has nice geometric interpretations. The key algorithmic component of \nameshort is a strategy to generate random walks towards the boundary of the \purse, and the key enabler for the efficiency of \nameshort is parallel computing in GPUs. With {three} real-world datasets, we demonstrate the efficiency and effectiveness of \nameshort in producing tight uncertainty estimation of 6D poses.

%% file: sections/app-non-gaussian-3d3d.tex

\subsection{Non-Gaussian Measurement Noise in Example~\ref{ex:3D3D}}
\label{sec:non-gaussian-3d3d}

In this section, we provide numerical evidence that the noise in measurements generated by either handcrafted features or learned features in point cloud registration (Example \ref{ex:3D3D}) do not follow a Gaussian distribution.

Our strategy is to compute the noise vector $\epsilon_i$ from \eqref{eq:generative-model} as 
$$ \epsilon_i =  b_i - (R a_i + t), i=1,\dots,N, $$
where $y_i = (a_i, b_i) \in \Real{3} \times \Real{3}$ is a pair of matched 3D keypoints in the source point cloud and the target point cloud respectively, typically available through matching (handcrafted or learned) features of the point clouds. We evaluate the noises on the popular \threedmatch dataset~\cite{zeng17cvpr-3dmatch}
using two feature matching algorithms: (i)
a handcrafted feature \fpfh~\cite{rusu09icra-fpfh}, and (ii)
a learned feature \fcgf~\cite{choy19iccv-fcgf}.

This procedure produces a set of $N = 1,431,186$ and $N = 8,410,386$ pairs of matching points $y_i = (a_i, b_i)$, respective for \fpfh and \fcgf. We want to test if the noise $\left\{\epsilon_i\right\}_{i=1}^N $  are drawn from a multivariate Gaussian (normal) distribution. To do so, we use the R package MVN \cite{korkmaz2014mvn} that provides a suite of popular multivariate normality tests well established in statistics. To consider potential outliers in the noise vectors $\left\{\epsilon_i\right\}_{i=1}^N$, we run MVN on $\alpha \% $ of the noise vectors with smallest norms, and we sweep $\alpha$ from 1 (\ie keep only the $1 \%$ smallest noise vectors) up to 100 (\ie keep all noise vectors). Then we provide perspective plots and the Chi-square quantile-quantile (Q-Q) plots of the noise vectors under different inlier thresholds.

\begin{table}[h]
	\centering
	\adjustbox{max width=\columnwidth}{%
	\begin{tabular}{ c|c|c|c|c|c } 
		\midrule
		Percentage  & Mardia &  Henze-Zirkler & Royston &  Doornik-Hansen & Energy\\ 
		\hline
		$1 \%$ & NO &NO & NO&NO &NO \\
		$5 \%$& NO &NO & NO&NO &NO \\
		$10 \%$ & NO &NO & NO&NO &NO \\
		$20 \%$ & NO &NO & NO&NO &NO \\
		$40 \%$ & NO &NO & NO&NO &NO \\
		$100 \%$ & NO &NO & NO&NO &NO \\
		\hline
		
	\end{tabular}}
	\caption{MVN test for \fpfh method on $\alpha \%$ percentage of the noise vectors with smallest norms. (MVN test for \fcgf method has the same result.)}
	 \label{table:noise-table-3d}
\end{table}

Table \ref{table:noise-table-3d} shows all the MVN tests indicate that the noise residuals do not follow a Gaussian distribution even after filtering potential outliers and no matter whether the matching points are calculated with handcrafted or deep learning based methods. Fig.~\ref{figure:noise-graph-3d-fpfh} and Fig.~\ref{figure:noise-graph-3d-fcgf} show the perspective plots (top) and the Chi-square quantile-quantile (Q-Q) plots (bottom) of the empirical density functions under different inlier thresholds, in comparison to that of a Gaussian distribution. We can see that the empirical density functions deviate far away from a Gaussian distribution, and are difficult to characterize. This motivates the set membership estimation framework in Section~\ref{sec:intro}. 

\begin{figure*}[!htbp]

	\centering
	{\includegraphics[width=0.15\textwidth]{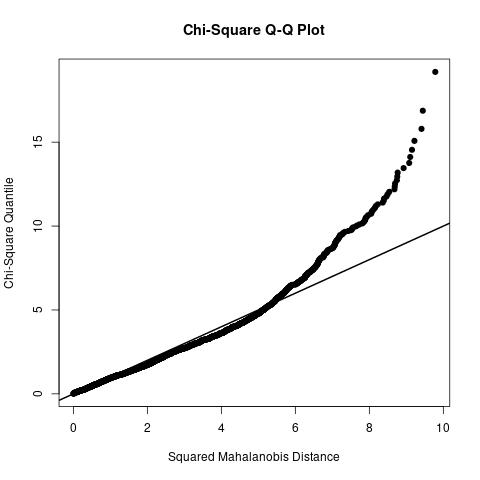}} 
	{\includegraphics[width=0.15\textwidth]{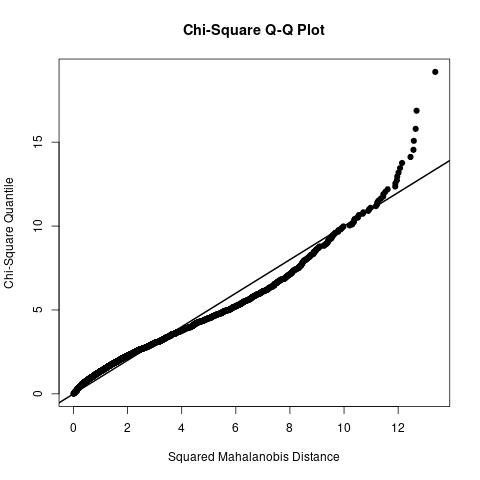}} 
	{\includegraphics[width=0.15\textwidth]{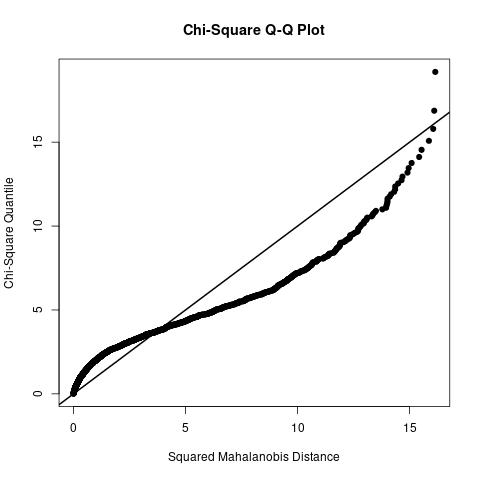}} 
	{\includegraphics[width=0.15\textwidth]{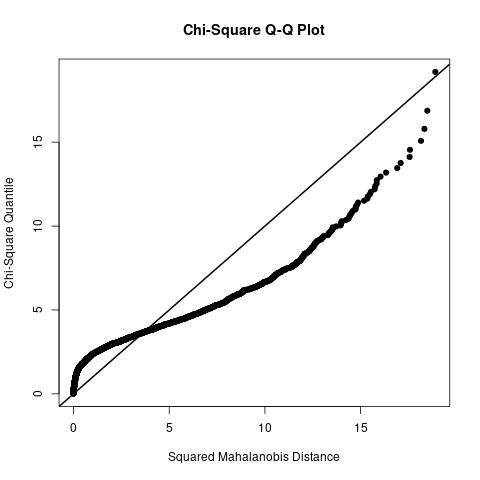}} 
	{\includegraphics[width=0.15\textwidth]{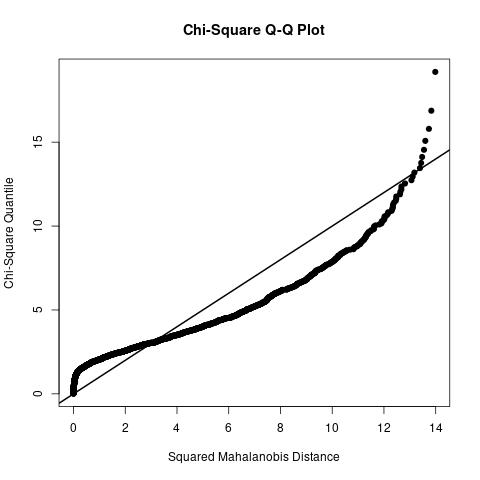}} 
	{\includegraphics[width=0.15\textwidth]{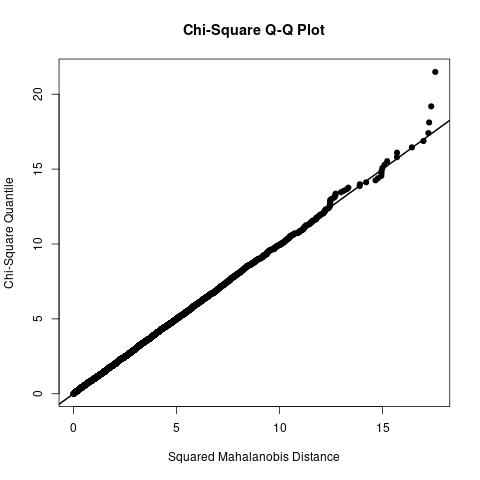}} 
	
	\centering
	{\includegraphics[width=0.15\textwidth]{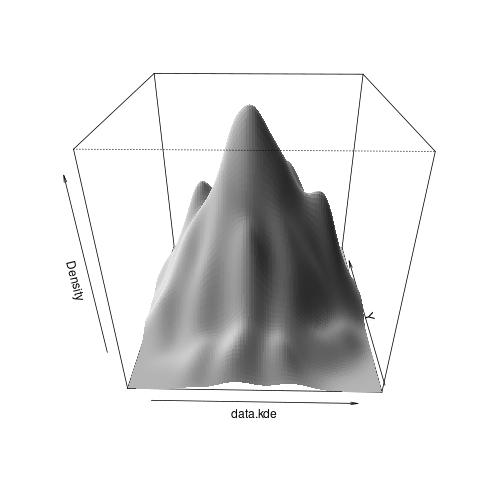}} 
	{\includegraphics[width=0.15\textwidth]{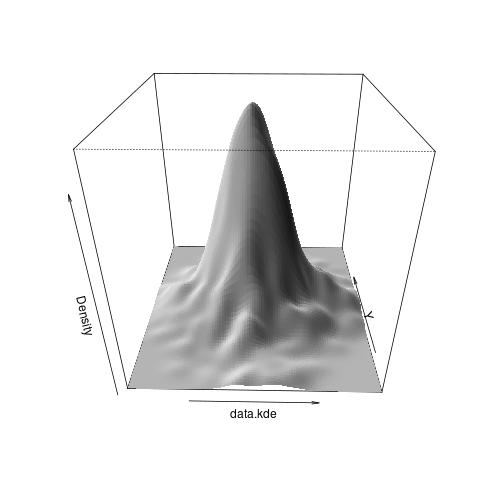}} 
	{\includegraphics[width=0.15\textwidth]{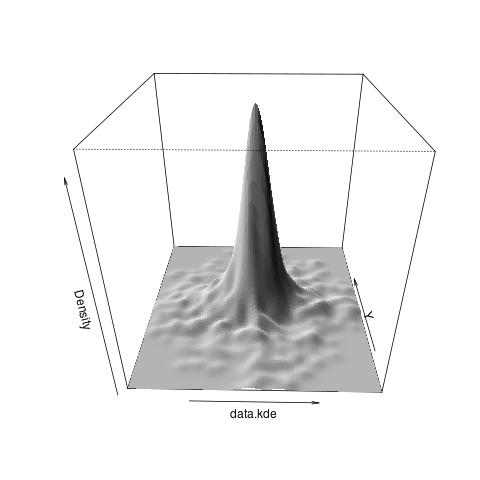}} 
	{\includegraphics[width=0.15\textwidth]{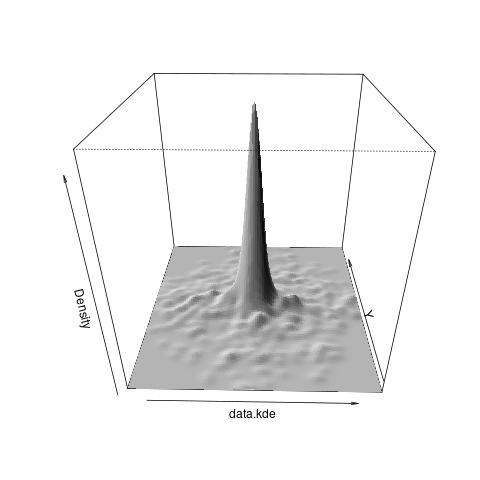}} 
	{\includegraphics[width=0.15\textwidth]{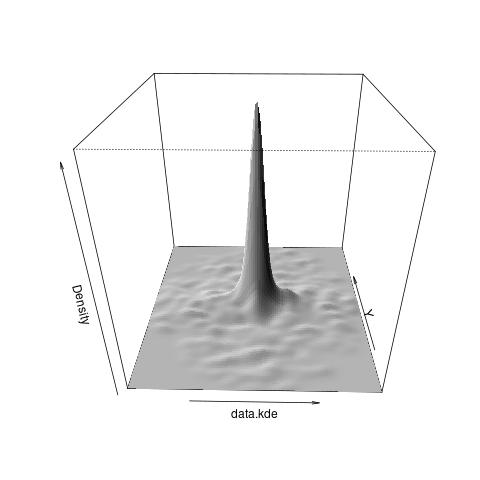}} 
	{\includegraphics[width=0.15\textwidth]{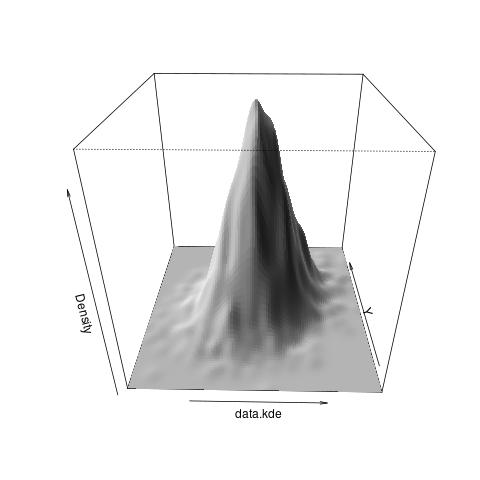}} 
	
	\caption{Perspective plots (top) and Chi-square Q-Q plots (bottom) of the noise vectors generated with {\bf handcreafed features (\fpfh)} from the \threedmatch dataset. From left to right, they are respectively for noise residual range 5cm, 10cm, 20cm, 50cm, 100cm. The right-most graph shows the perspective plot for a Gaussian distribution that can be used as comparison.}
	\label{figure:noise-graph-3d-fpfh}
\end{figure*}

\begin{figure*}[!htbp]

	\centering
	{\includegraphics[width=0.15\textwidth]{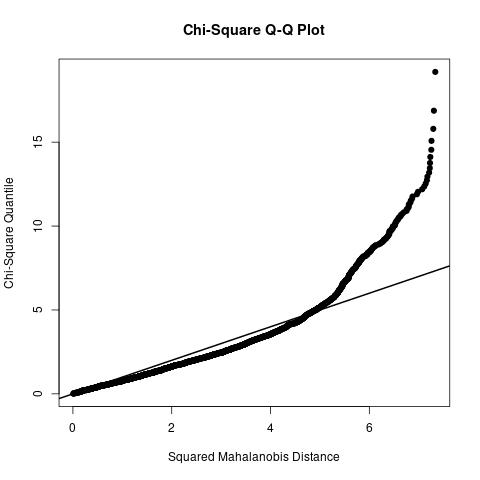}} 
	{\includegraphics[width=0.15\textwidth]{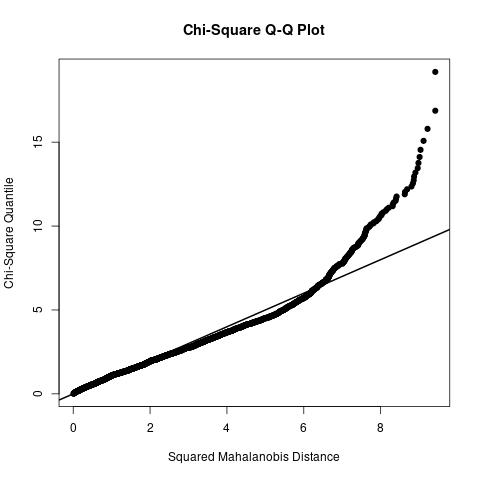}} 
	{\includegraphics[width=0.15\textwidth]{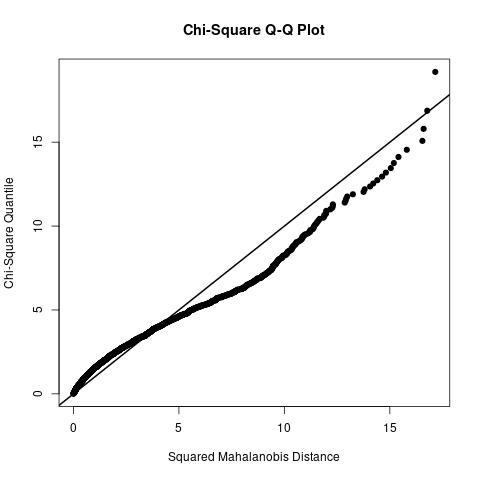}} 
	{\includegraphics[width=0.15\textwidth]{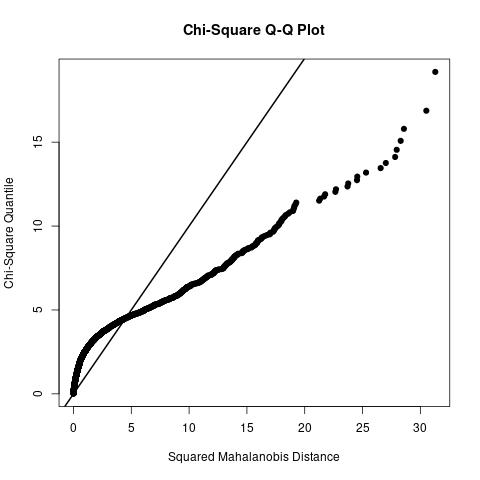}} 
	{\includegraphics[width=0.15\textwidth]{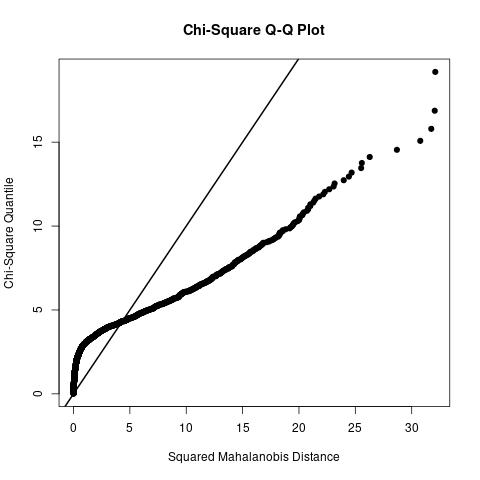}} 
	{\includegraphics[width=0.15\textwidth]{figures/fpfh/Chi-Square_Q-Q_normal.jpg}}

	\centering
	{\includegraphics[width=0.15\textwidth]{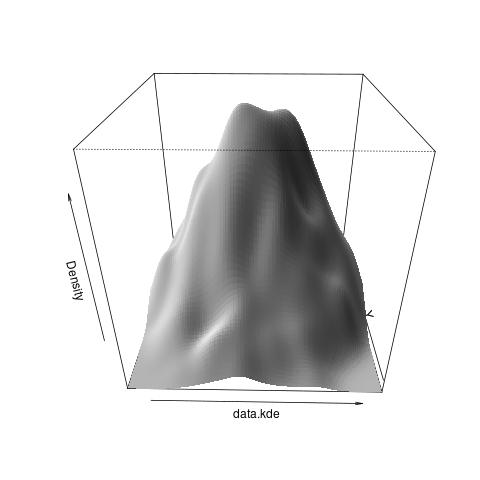}} 
	{\includegraphics[width=0.15\textwidth]{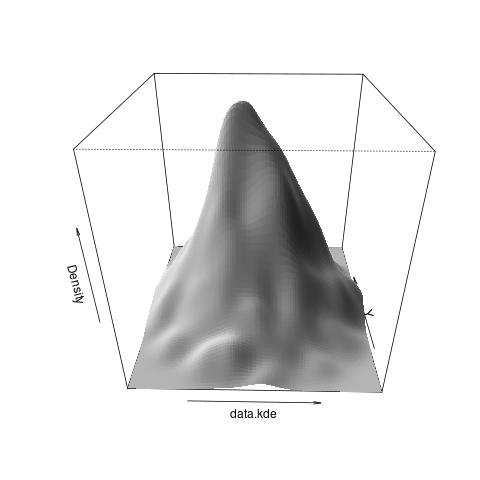}} 
	{\includegraphics[width=0.15\textwidth]{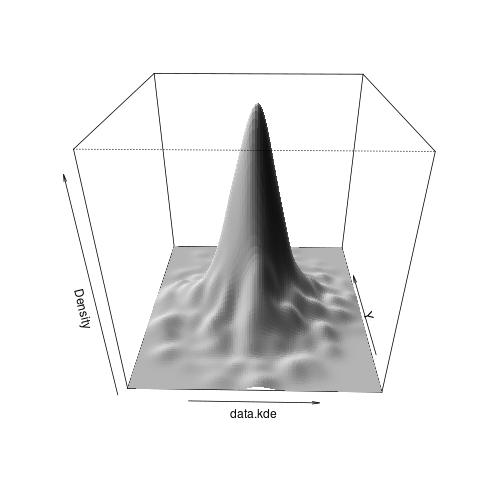}} 
	{\includegraphics[width=0.15\textwidth]{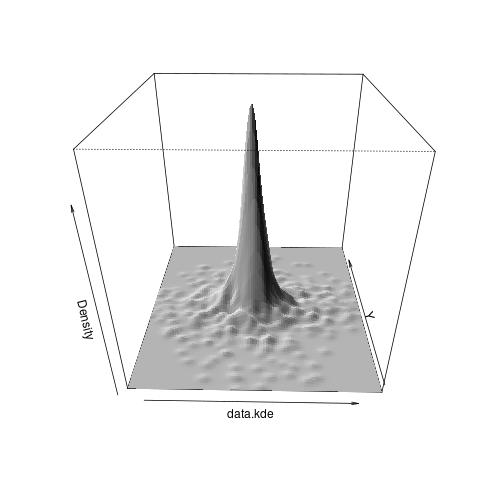}} 
	{\includegraphics[width=0.15\textwidth]{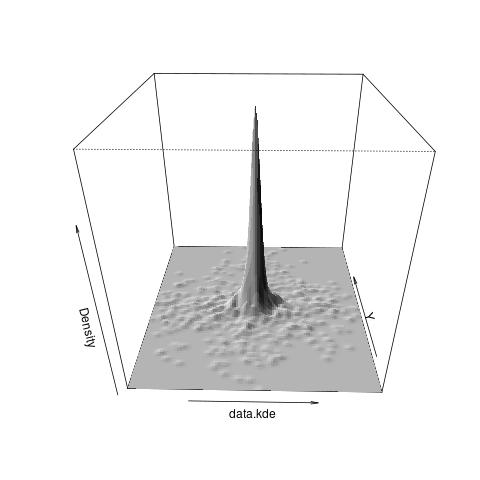}} 
	{\includegraphics[width=0.15\textwidth]{figures/fpfh/Persp1_normal.jpg}}

	\caption{Perspective plots (top) and Chi-square Q-Q plots (bottom) of the noise vectors generated with {\bf learned features (\fcgf)} from the \threedmatch dataset. From left to right, they are respectively for noise residual range 5cm, 10cm, 20cm, 50cm, 100cm. The right-most graph shows the perspective plot for a Gaussian distribution that can be used as comparison.}
	\label{figure:noise-graph-3d-fcgf}
\end{figure*}

%% file: sections/app-uncertainty-calibration-3d3d.tex

\subsection{Uncertainty Calibration for Example~\ref{ex:3D3D}}
\label{sec:uncertainty-calibration}
\input{sections/alg-3d-calibration.tex}

The sketch of the calibration process for the Example~\ref{ex:3D3D} is shown in Algorithm~\ref{alg:3d-calibration}.

{\bf Basic Setup}.
Given a pair of point clouds, the \dgr network~\cite{choy2020deep} outputs a set of matched keypoints $\{ (a_j,b_j) \}$ together with weights $w_j$, with $w_j$ indicating the confidence of the $j$-th match being valid. Thus, we leverage the weights of the features output by the \dgr network as the confidence function in the calibration and test process. Also, we normalize the 2-norm of the weights of each point clouds to keep weights across different point clouds to have the same scale. To (1) reduce the calculation burden of the \purse set, (2) prune the outliers to reduce the uncertainty of the \purse set, we only select the top $50$ correspondences with the largest weights for the calibration and the test process.

{\bf Calibration process}.
In the calibration process, we follow the inductive conformal prediction procedure~\cite{angelopoulos21arxiv-gentle} and randomly choose 400 pairs of point clouds for calibration from the \threedmatch dataset with ground truth poses, and use the \dgr network to output the features and the weights of the features. The calibration process of the \purse set in Example~\ref{ex:3D3D} is carried out as follows: for the $i$-th point cloud, denote the ground truth pose as $(R_i,t_i)$, the correspondences as $\{a_j,b_j\}_{j=1}^{N_i}$ and the weights as $\{w_j\}_{j=1}^{N_i}$. We then calculate $\|R_ia_j + t_i - b_j\|^2\times w_j$. We denote $\alpha_i = \max_{j=1,...,N_i} (\|Ra_j + t - b_j\|^2\times w_j)$. Then we sort all the $\alpha_i$ as $\alpha_{\pi(1)} \geq \alpha_{\pi(2)} \geq \cdots \geq \alpha_{\pi(400)}$ and for a given error rate $\epsilon = 20\%$, we select $\alpha = \alpha_{\pi(400\times80\%)} = \alpha_{\pi(320)}$.

{\bf Test process}.
Then in the test process, for the $i$-th point cloud, we use the \dgr network to output the features and the weights of the features. We denote the correspondences as $\{a_j,b_j\}_{j=1}^{N_i}$ and the weights as $\{w_j\}_{j=1}^{N_i}$. We formulate the \purse set as:
$$
\cbrace{ (R,t)| \|Ra_j + t - b_j\|^2 \leq \frac{\alpha}{w_j},\ j = 1,\cdots, N_i }
$$

%% file: sections/alg-3d-calibration.tex
\begin{algorithm}[t]
    \SetAlgoLined
        {\bf Input: } Calibration Point Cloud pairs with matched keypoints $\calD = \{\{a_j,b_j\}_{j=1}^{N_i}\}_{i = 1}^{N_{\calnum}}$; error rate $\epsilon$ \\
        {\bf Output: } Calibration threshold $\alpha$\\

        \For{$i \in \{1,\cdots,N_{\calnum}\}$}{
            $N_i \gets$  number of correspondences in the $i$-th point cloud;\\
            $(R_i,t_i) \gets$ ground truth pose of the $i$-th point cloud;\\
            $\{a_j,b_j\}_{j=1}^{N_i} \gets$ correspondences of the $i$-th pair of point clouds;\\
            $\{w_j\}_{j=1}^{N_i} \gets$ weights of the correspondences;\\
            $\alpha_i \gets \max_{j=1,...,N_i} (\|R_ia_j + t_i - b_j\|^2\times w_j)$;\\
        }
        Sort $\alpha_{\pi(1)} \geq \alpha_{\pi(2)} \geq \cdots \geq \alpha_{\pi(N_{\calnum})}$;\\
            $\alpha \gets \alpha_{\pi(N_{\calnum}\times(1-\epsilon))}$;\\
        {\bf return:} $\alpha$\\
        \caption{\threedmatch conformal calibration \label{alg:3d-calibration}}
    \end{algorithm}

%% file: sections/app-uncertainty-calibration-foundation.tex
{
\subsection{Uncertainty Calibration for Example~\ref{ex:poseregression}}
\label{sec:uncertainty-calibration-foundation}

\input{sections/alg-foundation-calibration.tex}

The sketch of the calibration process for Example~\ref{ex:poseregression} is shown in Algorithm~\ref{alg:foundation-calibration}.

{\bf Basic Setup}.
Given an image (as well as its model), the FoundationPose network~\cite{wen24cvpr-foundationpose} can output multiple pose hypotheses and scores. The original paper only outputs the pose with the highest score. In this work, we leverage poses hypotheses with top 10 scores to quantify the uncertainty of this pose estimation paradigm. We use the normalized scores as the confidence function in the calibration and test process. For the experiments carried out in the main article, we choose error rate $\epsilon = 10\%$.

{\bf Calibration process}.
In the calibration process, we follow the inductive conformal prediction procedure~\cite{angelopoulos21arxiv-gentle} and randomly choose 200 images for calibration process from the \lm dataset with ground truth poses. The calibration process is carried out as follows: 
\begin{enumerate}
    \item First, we need to ``normalize'' the scale between rotation and translation. Thus, we first carry out the calibration process as in Section~\ref{sec:uncertainty-calibration}. Denote the 10 pose hypotheses and their normalized scores as $\{(R_{i,1},t_{i,1}),\cdots, (R_{i,10},t_{i,10})\}$ and $\{p_{i,1},\cdots,p_{i,10}\}$. We get the quantile for rotation and translation individually as $k_R,k_t$ by collecting $\max_{j = 1,\cdots, 10}(p_j\times\|\vectorize{R_{i,j} - R_i}\|)$ and $\max_{j = 1,\cdots, 10}(p_j\times\|t_{i,j} - t_i\|)$ for every image in the calibration set.
    \item Then we carry out the calibration process again, but this time we use the quantile $k_R$ and $k_t$ to normalize the rotation and translation counterparts. Thus, we get the calibration threshold $\alpha_i$ for each image $i$ as $\max ( \text{R_scores}(i) /k_R,  \text{t_scores}(i) /k_t )$. And finally we get the quantile for the calibration threshold as $\alpha_{\pi(N_{\calnum}\times(1-\epsilon))}$.

\end{enumerate}

{\bf Testing process}.
    Then in the test process, first we use the FoundationPose network to output the 10 pose hypotheses $\{R_1,\cdots,R_{10}\}$ and their normalized scores $\{p_1,\cdots, p_{10}\}$. We formulate the \purse set as:
    $$
    \cbrace{ (R,t)\mid \|R - R_i\| \leq \frac{\alpha}{p_i\times k_R}, \|t - t_i\| \leq \frac{\alpha}{p_i\times k_t},\ j = 1,\cdots, 10 }
    $$

}

%% file: sections/alg-foundation-calibration.tex
\begin{algorithm}[t]
    \SetAlgoLined
        {\bf Input: } Calibration 2D images $P_1, ..., P_{N_\calnum}$ with ground truth poses; error rate $\epsilon$ \\
        {\bf Output: } Calibration threshold $\alpha$, Rotation scale $k_R$ and translation scale $k_t$\\

        \For{$i \in \{1,\cdots,N_{\calnum}\}$}{
            $(R_i,t_i) \gets$ ground truth pose of the $i$-th image;\\
            $\{(R_{i,1},t_{i,1}),\cdots, (R_{i,10},t_{i,10})\} \gets$ top 10 pose hypotheses of the $i$-th image;\\
            $\{p_{i,1},\cdots,p_{i,10}\} \gets$ scores of top 10 pose hypotheses, normalized with their sum;\\
            R_scores($i$) $\gets \max_{j = 1,\cdots, 10}(p_j\times\|\vectorize{R_{i,j} - R_i}\|)$ \\
            t_scores($i$) $\gets \max_{j = 1,\cdots, 10}(p_j\times\|t_{i,j} - t_i\|)$ \\
        }

        Sort R_scores in descending order as R_scores$(\pi_R(1)),\cdots,$ R_scores$(\pi_R(N_{\calnum}))$;\\
        $k_R \gets $R_scores$(\pi_R(N_{\calnum}\times(1-\epsilon)))$;\\
        Sort t_scores in descending order as t_scores$(\pi_t(1)),\cdots,$ t_scores$(\pi_t(N_{\calnum}))$;\\
        $k_t \gets $t_scores$(\pi_t(N_{\calnum}\times(1-\epsilon)))$;\\

        \For{$i \in \{1,\cdots,N_{\calnum}\}$}{
            $(R_i,t_i) \gets$ ground truth pose of the $i$-th image;\\
            $\{(R_{i,1},t_{i,1}),\cdots, (R_{i,10},t_{i,10})\} \gets$ top 10 pose hypotheses of the $i$-th image;\\
            $\{p_{i,1},\cdots,p_{i,10}\} \gets$ scores of top 10 pose hypotheses, normalized with their sum;\\
            R_scores($i$) $\gets \max_{j = 1,\cdots, 10}(p_j\times\|\vectorize{R_{i,j} - R_i}\|)$ \\
            t_scores($i$) $\gets \max_{j = 1,\cdots, 10}(p_j\times\|t_{i,j} - t_i\|)$ \\
            $\alpha_i \gets \max ( \text{R_scores}($i$) /k_R,  \text{t_scores}($i$) /k_t )$;\\
        }
        Sort $\alpha_{\pi(1)} \geq \alpha_{\pi(2)} \geq \cdots \geq \alpha_{\pi(N_{\calnum})}$;\\
            $\alpha \gets \alpha_{\pi(N_{\calnum}\times(1-\epsilon))}$;\\
        {\bf return:} $\alpha,k_R,k_t$\\
        \caption{\lm conformal calibration \label{alg:foundation-calibration}}
    \end{algorithm}

%% file: sections/app-geodesic-gradient-descent.tex

\subsection{Geodesic Gradient Descent}
\label{sec:geodesic-gradient-descent}

In this section, we focus on finding the \megb on $\SOthree$ for rotation through geodesic gradient descent, as the geometry for $\Real{3}$ is simpler and already studied in prior work \cite{badoiu03-smaller}. We specify the algorithm framework introduced in \cite{arnaudon2013approximating} to $\SOthree$ and prove that it is equal to geodesic gradient descent. We then provide convergence analysis and numerical experiments to verify the results.

{\bf Inner maximization problem}. We first focus on the inner maximization problem. Let $\calS_R \subset B(O,\rho)$ be a nonempty compact subset of $\SOthree$ contained in the open ball $B(O,\rho)$ at center $O$ with radius $\rho \leq \halfpi$, consider the function
\bea 
f: B(O,\rho)  \rightarrow & \Real{} \nonumber \\
c  \mapsto & f(c) = \displaystyle \max_{s \in \calS_R} \rotdist^2 (c,s). \label{eq:defellinf}
\eea

If $\rho \leq \halfpi$, we can guarantee the continuity and convexity of $f$. Specifically, we have the following proposition:
\begin{proposition}[Convexity and (Sub)gradient] The following properties hold true for $f$:
    \begin{enumerate}[label=(\roman*)]
        \item $f$ is Lipschitz continuous with $L=4\rho$, \ie $\abs{f(c_1) - f(c_2)} \leq L \rotdist(c_1,c_2),\forall c_1,c_2 \in B(O,\rho)$;
        \item $f$ is strictly convex when $\rho = \halfpi$;
        \item $f$ is $\mu$-strongly convex when $\rho < \halfpi$, with $\mu = 2\rho \cot{\rho}$.
    \end{enumerate}
\end{proposition}
\begin{proof}
{\bf Lipschitz continuity of $f$}. Pick any two points $c_1,c_2 \in B(O,\rho)$, since $\calS_R$ is a nonempty compact set, the ``$\max$'' in~\eqref{eq:defellinf} is attained, say at $s_1$ and $s_2$, respectively ($s_1$ and $s_2$ are not necessarily unique). Consequently, we can write
\bea 
& \abs{f(c_1) - f(c_2)} \nonumber\\
= &\abs{\rotdist^2(c_1,s_1) - \rotdist^2(c_2,s_2)} \nonumber \\
= & \parentheses{\rotdist(c_1,s_1) + \rotdist(c_2,s_2)} \abs{\rotdist(c_1,s_1) - \rotdist(c_2,s_2)}  \nonumber \\
\leq & 4\rho \abs{\rotdist(c_1,s_1) - \rotdist(c_2,s_2)} \nonumber
\eea 
where the last inequality holds because of the triangle inequality
\bea 
\rotdist(c_i,s_i) \leq \rotdist(c_i,O) + \rotdist(O,s_i) < 2\rho, \quad i=1,2. \nonumber
\eea 
It remains to show $\abs{\rotdist(c_1,s_1) - \rotdist(c_2,s_2)} \leq \rotdist(c_1,c_2)$. First, the inequality trivially holds when $\rotdist(c_1,s_1) = \rotdist(c_2,s_2)$. Second, when $\rotdist(c_1,s_1) > \rotdist(c_2,s_2)$, we have 
\bea 
\rotdist(c_1,s_1) > \rotdist(c_2,s_2) \geq \rotdist(c_2,s_1), \nonumber
\eea 
with the last inequality due to $s_2$ attains the maximum (squared) distance to $c_2$. Hence,
\bea 
& \rotdist(c_1,s_1) - \rotdist(c_2,s_2) \nonumber\\
& \leq \rotdist(c_1,s_1) - \rotdist(c_2,s_1) \leq \rotdist(c_1,c_2), \nonumber
\eea 
where the last inequality follows again from the triangle inequality. Third, when $\rotdist(c_1,s_1) < \rotdist(c_2,s_2)$, we similarly have 
\bea 
& \rotdist(c_2,s_2) - \rotdist(c_1,s_1) \nonumber \\
& \leq \rotdist(c_2,s_2) - \rotdist(c_1,s_2) \leq \rotdist(c_1,c_2). \nonumber
\eea 
Therefore, $f$ is Lipschitz continuous with $L=4\rho$.

To show $f$ is convex, we first need the domain $B(O,\rho)$ to be convex. This is evident in~\cite[Lemma 6]{hartley13ijcv-rotation} for $\rho \leq \halfpi$. We then show $f$ has a positive (semi-)definite Hessian in $B(O,\rho)$.

{\bf Convexity of $\rotdist^2(c,s)$}.
We first analyze the gradient and Hessian of $\rotdist^2(c,s)$ at $c \in B(O,\rho)$ with $s \in \calS_R \subset B(O,\rho)$. By definition, the gradient and Hessian of $\rotdist^2(c,s)$ are the gradient and Hessian of the following function \wrt $x \in \Real{3}$, a tangent vector at $c$:
\bea 
\rotdist^2(c\Exp{x},s) = \rotdist^2(\Exp{x},c\tran s), \nonumber
\eea 
where $\Exp{x} = \exp\parentheses{x^\wedge}$ is the composition of the ``$^\wedge$'' map and the exponential map.
Let $c\tran s \in \SOthree$ be a rotation about the unit axis $w$ with angle $\gamma$ (\ie the axis-angle representation for $c\tran s$ is $(w,\gamma)$). We then recall the following cosine rule in $\SOthree$ about geodesic triangles.
\begin{lemma}[Cosine rule in $\SOthree$~{\cite[Proposition 2]{hartley13ijcv-rotation}}] \label{lemma:cosineruleSOthree}
    Let $A_1,A_2,A_3 \in \SOthree$ be three points forming a triangle and let $a_1,a_2,a_3$ be the lengths of the three geodesic line segments. If $a_3$ is the length of the smaller geodesic arc between $A_1$ and $A_2$, then 
\bea \label{eq:cosineruleSOthree}
\cos{\frac{a_3}{2}} = \cos{\frac{a_1}{2}}\cos{\frac{a_2}{2}} + \sin{\frac{a_1}{2}} \sin{\frac{a_2}{2}} \cos{\angle A_1 A_3 A_2 },
\eea 
where $\angle A_1 A_3 A_2$ is the angle formed by the geodesic segments $A_1 A_3$ and $A_2 A_3$ at point $A_3$.
\end{lemma}
Using Lemma~\ref{lemma:cosineruleSOthree} and letting $A_1 = \Exp{x}$, $A_2 = c\tran s$, and $A_3 = \eye_3$, we have
\bea 
& a_1 = \gamma, \quad a_2 = \norm{x}, \quad a_3 = \rotdist(\Exp{x},c\tran s),\nonumber \\ & \cos{ \angle A_1 A_3 A_2 } = \frac{x\tran w}{\norm{x}}. \nonumber
\eea 
Invoking the cosine rule~\eqref{eq:cosineruleSOthree}, we obtain
\bea \label{eq:rawcosinerule}
& \cos{\frac{a_3}{2}} = \cos{\frac{\gamma}{2}} \cos{\frac{\norm{x}}{2}} + \sin{ \frac{\gamma}{2} } \sin{ \frac{\norm{x}}{2} } \frac{x\tran w}{\norm{x}} \Longrightarrow \nonumber \\
& \rotdist(\Exp{x},c\tran s) = a_3 \nonumber \\
& = 2 \acos{ \cos{\frac{\gamma}{2}} \cos{\frac{\norm{x}}{2}} + \sin{ \frac{\gamma}{2} } \sin{ \frac{\norm{x}}{2} } \frac{x\tran w}{\norm{x}} }. 
\eea 
We want derivatives up to second order, so we perform a second-order Taylor expansion for~\eqref{eq:rawcosinerule}:
\bea 
a_3 = 2\acos{ \cos{\frac{\gamma}{2}} \parentheses{1 - \frac{\norm{x}^2}{8}} + \half \sin{\frac{\gamma}{2}} x\tran w + o(\norm{x}^3) } \nonumber
\eea
Finally, we have the gradient
\bea \label{eq:gradientdsquare}
\nabla \rotdist^2 (c,s) = \nabla_x \rotdist^2 (\Exp{x},c\tran s ) \vert_{x=0} = - \gamma w\tran, 
\eea 
where $(w,\gamma)$ is the axis-angle representation of $c\tran s$. Equation~\eqref{eq:gradientdsquare} states that the gradient of $\rotdist^2(c,s)$ at $c$ points in the negative direction of the geodesic from $c$ to $s$ (\ie along the direction from $s$ to $c$), and has magnitude equal to the length of the geodesic. When $c=s$, $\gamma=0$ and the gradient is equal to zero. Similarly, for the Hessian we have 
\bea
\nabla^2 \rotdist^2 (c,s) =& \nabla_{xx} \rotdist^2 (\Exp{x},c\tran s ) \vert_{x=0} \nonumber \\
= & \nabla_x (2a_3 \nabla_x a_3)\vert_{x=0} \nonumber \\
= & 2 (\nabla_x a_3) \tran \nabla_x a_3\vert_{x=0} + 2a_3 \nabla_{xx} a_3\vert_{x=0} \nonumber\\
= & 2 w w\tran + 2\gamma \nabla_{xx} a_3 \vert_{x=0} \nonumber \\
=& \displaystyle 2ww\tran + \gamma \cot{\frac{\gamma}{2}} \parentheses{\eye_3 - ww\tran}.
\eea 
Note that 
\bea 
\lim_{\gamma \rightarrow 0} \gamma \cot{\frac{\gamma}{2}} = \lim_{\gamma \rightarrow 0} \frac{ (\gamma \cos{\frac{\gamma}{2}})' }{ (\sin{\frac{\gamma}{2}})' } = \lim_{\gamma\rightarrow 0} \frac{ \cos{\frac{\gamma}{2}} + \frac{\gamma}{2} \sin{\frac{\gamma}{2}} }{ \half \cos{\frac{\gamma}{2}} } = 2,
\eea
which implies the Hessian at $c$ is equal to $2\eye_3$ when $c=s$ and $\gamma = 0$. Let $(w,u,v)$ be a set of orthonormal basis in $\Real{3}$ (\ie choose $u,v$ as two orthogonal unit vectors in the plane perpendicular to $w$), we have 
\begin{subequations}
    \bea 
\nabla^2 \rotdist^2(c,s) w = & 2w, \\
\nabla^2 \rotdist^2(c,s) u = & \gamma \cot{\frac{\gamma}{2}} u, \\
\nabla^2 \rotdist^2(c,s) v = & \gamma \cot{\frac{\gamma}{2}} v,
\eea 
\end{subequations}
which states the Hessian has one eigenvalue equal to $2$ and two eigenvalues equal to $\gamma \cot{\frac{\gamma}{2}}$. Since $\gamma \cot{\frac{\gamma}{2}} \in (2\rho \cot{\rho},2]$ when $\gamma 
\in [0,2\rho)$,\footnote{Note that since $c,s \in B(O,\rho)$, we have $\gamma = \rotdist (c,s) \leq \rotdist (c,O) + \rotdist(O,s) < 2\rho$.} we have
\bea 
\nabla^2 \rotdist^2(c,s) \succ  0, \quad  \forall c,s \in B\parentheses{O,\rho} \text{ with }\rho = \halfpi,
\eea 
and $\rotdist^2(c,s)$ is strictly convex on $B(O,\rho)$ with $\rho = \halfpi$;
\bea \label{eq:strongconvexitydsquaresmallrho}
\nabla^2 \rotdist^2(c,s) \succeq 2\rho\cot{\rho} \eye_3, \quad \forall c,s \in B(O,\rho) \text{ with } \rho < \halfpi,
\eea 
and $\rotdist^2(c,s)$ is strongly convex on $B(O,\rho)$ with $\rho < \halfpi$.

{\bf Convexity of $f$}. We then proceed to show the convexity of $f$. Let $c_1,c_2 \in B(O,\rho)$, and let $c_\alpha$ be the $\alpha$-midpoint along the geodesic from $c_1$ to $c_2$ with $\alpha \in (0,1)$. Because $\calS_R$ is compact, we can write
\bea \label{eq:fcalpha}
f(c_\alpha) = \max_{s \in \calS_R} \rotdist^2(c_\alpha,s) = \rotdist^2 (c_\alpha,s_\alpha)
\eea 
with $s_\alpha \in \calS_R$ a point that attains the maximum distance to $c_\alpha$ (note the point $s_\alpha$ needs not be unique). By the strict convexity of $\rotdist^2 (c_\alpha,s_\alpha)$ when $\rho = \halfpi$, one obtains 
\bea \label{eq:convexdsquare}
\rotdist^2(c_\alpha,s_\alpha) < (1-\alpha) \rotdist^2(c_1,s_\alpha) + \alpha \rotdist^2(c_2,s_\alpha).
\eea 
One now notices that by definition of $f$,
\bea \label{eq:fci}
f(c_i) = \max_{s \in \calS_R} \rotdist^2 (c_i,s) \geq \rotdist^2 (c_i,s_\alpha), i=1,2.
\eea 
Hence, combining~\eqref{eq:fcalpha}-\eqref{eq:fci}, we conclude with the strict convexity of $f$ when $\rho = \halfpi$:
\bea 
f(c_\alpha) < (1-\alpha) f(c_1) + \alpha f(c_2).
\eea 
When $\rho < \halfpi$, from~\eqref{eq:strongconvexitydsquaresmallrho} we know $\rotdist^2(c,s)$ is $\mu$-strongly convex with $\mu = 2\rho\cot{\rho}$. Consequently, \eqref{eq:convexdsquare} can be modified according to the definition of geodesic strong convexity~\cite[Definition 11.5]{boumal23book-introduction} 
\bea 
& \rotdist^2(c_\alpha,s_\alpha) \nonumber\\
\leq & \displaystyle (1-\alpha)\rotdist^2(c_1,s_\alpha) + \alpha \rotdist^2 (c_2,s_\alpha) - \\
& \frac{\mu}{2}\alpha (1-\alpha) \rotdist^2(c_1,c_2) \nonumber\\
\leq & \displaystyle (1-\alpha) f(c_1) + \alpha f(c_2) - \frac{\mu}{2}\alpha (1-\alpha) \rotdist^2(c_1,c_2), \label{eq:dsquarestrongconvexity}
\eea
where the second inequality is again by the definition of $f$ in~\eqref{eq:fci}. Combining~\eqref{eq:fcalpha} and~\eqref{eq:dsquarestrongconvexity} we obtain the geodesic $\mu$-strong convexity of $f$:
\bea 
f(c_\alpha) \leq (1-\alpha)f(c_1) + \alpha f(c_2) - \frac{\mu}{2}\alpha (1-\alpha) \rotdist^2(c_1,c_2). \nonumber
\eea
\end{proof}

{\bf Subdifferential of supreme functions}.
To derive our final theorem on solving subdifferential of supreme functions on Riemannian manifolds (\eg $\SOthree$), we first need to introduce some definitions and lemmas.

\begin{definition}[Directional derivative~{\cite[Section 3, Definition 4.1]{udriste1994convex}}]\label{def:directderivative}
    Let $f: M \rightarrow \R$ be a function on a Riemannian manifold $M$. The directional derivative of $f$ at $x \in M$ in the direction $v \in T_x M$ is defined as
   \bea
       f'(x;v) = \lim_{t \rightarrow 0^+} \frac{f(\gamma(t)) - f(x)}{t}.
   \eea
   where $\gamma: \in [-\delta, \delta]\rightarrow M$ is a geodesic segment on $M$ with $\gamma(0) = x$ and $\gamma'(0) = v$.
\end{definition}

For a geodesically convex function $f$, for a given $x$ and $\gamma$, $f\circ \gamma: [-\delta, \delta] \rightarrow \R$ is convex. Thus $\frac{f\circ\gamma(t)-f\circ\gamma(0)}{t}$ is non-decreasing on $t>0$, and we have \cite[Section 3, Theorem 4.2]{udriste1994convex}
\bea
 f'(x;v) = \inf_{t\in(0, \delta]} \frac{f\circ\gamma(t)-f\circ\gamma(0)}{t}. \label{eq:nondecreasing}
\eea

\begin{definition}[Subdifferential on Riemannian manifold]
    Let $f: M \rightarrow \R$ be a convex function on a Riemannian manifold $M$. The subdifferential of $f$ at $x$ is defined as 
    \bea
        \partial f(x) = &\{v \in T_x M| f(\gamma(t)) \geq f(x) + t\langle v, \gamma'(0) \rangle, \nonumber\\
        & \quad \forall t>0,\forall \gamma: [-\delta, \delta]\rightarrow \R \quad s.t. \gamma(0) = x\}.
    \eea
\end{definition}

We can also use directional derivatives to define the subdifferential.
\begin{proposition}[Equivalent expression of subdifferential on Riemannian manifold~{\cite[Section 3, Theorem 4.8]{udriste1994convex}}] \label{prop:anothersubdiff}
    For every $x \in M$, $\partial f(x)$ is non-empty, convex and compact. Moreover, we have
    \bea
        \partial f(x) = \{v \in T_x M| f'(x;w) \geq \langle v, w \rangle, \quad \forall w \in T_xM\}.
    \eea
\end{proposition}

Then we can derive the following theorem on solving subdifferential of supreme functions on Riemannian manifolds.
\begin{theorem}[Danskin's Theorem on Riemannian manifolds] \label{thm:danskin}
    Let $\calS \subset M$ be a compact set on a Riemannian manifold $M$, and let $\phi: M \times \calS \rightarrow \R$ be continuous and such that $\phi(\cdot,s)$ is geodesically convex for all $s \in \calS$.
    \begin{enumerate}[label=(\roman*)]
        \item The function $f: M \rightarrow \R$ defined by $f(x) = \max_{s \in \calS} \phi(x,s)$ is geodesically convex and has directional derivative given by 
        \bea \label{eq:directioaldiffeq}
            f'(x;v) = \max_{\bar{s} \in \calS(x)} \phi'(x,\bar{s};v), \quad \forall x \in M, v \in T_x M.
        \eea
        where $\phi'(x,s;v)$ is the directional derivative of $\phi$ at $(x,\bar{s})$ in the direction $v$ and $\calS(x) = \{\bar{s} \in \calS: f(x) = \phi(x,\bar{s})\}$ is the set of maximizing points of $\phi(x, \cdot)$.
        \item If $\phi(\cdot, s)$ is differentiable at all $s \in \calS$ and $\nabla_x\phi(x, \cdot)$ is continuous on $\calS$ for each $x$, then 
        \bea \label{eq:subdiffeq}
            \partial f(x) = \conv{\{\nabla_x \phi(x,\bar{s}): \bar{s} \in \calS(x)\}}.
        \eea
    \end{enumerate}
\end{theorem}

\begin{proof}
    The structure of the proof follows the Euclidean version in \cite[Proposition B.25]{bertsekas97-nonlinear} but with different definitions according to the Riemannian manifold.

    {\bf Convexity of $f$}. Similar to the proof of Proposition 1. The conclusion holds for an arbitrary geodesic convex function $\phi(c, s)$.

    {\bf Directional derivative of $f$}. According to the definition of $f$, for any $\bar{s}\in \calS(x)$, we have $f(x) = \phi(x, \bar{s})$ and $f(\gamma(t)) \geq \phi(\gamma(t), \bar{s})$ for all $t \in [-\delta, \delta]$. Thus we have
    \bea
        \frac{f(\gamma(t)) - f(x)}{t} \geq \frac{\phi(\gamma(t), \bar{s}) - \phi(x, \bar{s})}{t}.
    \eea
    We take the limit $t \rightarrow 0^+$ on both sides and obtain $f'(x;v) \geq \phi'(x,\bar{s};v)$ for all $\bar{s} \in \calS(x)$. Thus we have
    \bea \label{eq:directioaldiffgeq}
        f'(x;v) \geq \max_{\bar{s}\in \calS(x)}\phi'(x,\bar{s};v)
    \eea
     
    To prove the reverse inequality, we consider $\{t_k\}\rightarrow 0^+$ and let $x_k = \gamma(t_k)$ where velocity $\gamma'(t_k)=v_k$. For each $k$, we find $s_k\in \calS(x_k)$. Since $\calS$ is compact, there exists a subsequence $\{s_{k_j}\}$ converging to some $s_0 \in \calS$. Without loss of generality, we assume $\{s_k\}$ converges to $s_0$. Since $\phi$ is continuous, we have $\phi(x_k, s_k) \rightarrow \phi(x, s_0)$. 
    Thus, for any $s\in \calS$, we have $\phi(x_k, s_k) \geq \phi(x_k, s)$. We let $k\rightarrow +\infty$ and we obtain $\phi(x, s_0) \geq \phi(x, s)$. Thus $s_0 \in \calS(x)$ and $f(x) = \phi(x, s_0)$. We now have
    \bea
        f'(x; v)   \leq & \displaystyle \frac{f(x_k)-f(x)}{t_k} \\
                   = & \displaystyle \frac{\phi(x_k, s_k) - \phi(x, s_0)}{t_k}\\
                   \leq & \displaystyle \frac{\phi(x_k, s_k) - \phi(x, s_k)}{t_k}\\
                   \leq & \displaystyle -\phi'(x_k, s_k; -v_k) \\ \label{eq:negativedirectdiff}
                   \leq & \displaystyle \phi'(x_k, s_k; v_k) 
    \eea

    The inequality in \eqref{eq:negativedirectdiff} is the result of \eqref{eq:nondecreasing} applied on $\phi'(x_k, s_k; -v_k)$. The last inequality is proved in \cite[Section 3, Theorem 4.2]{udriste1994convex}.
    We now consider function $g_k(\cdot) = \phi(\gamma(\cdot), s_k):\mathbb{R}\to\mathbb{R}$, and $g(\cdot) = \phi(\gamma(\cdot), s_0)$. From \cite[Section 3, Theorem 4.2]{udriste1994convex}, $g_k,g$ are convex functions. By the continuity of function $\phi,\gamma$, $\lim_{k\to\infty}g_k(t_k) = g(0)$. Then we apply the following lemma.

    \begin{lemma}[{\cite[Proposition B.23]{bertsekas97-nonlinear}}]
        Suppose $g:\mathbb{R}\to \mathbb{R}$ is a convex function and $\{t_k\}\to t$ is a convergent sequence. If $g_k:\mathbb{R}\to \mathbb{R}$ is a sequence of convex functions with property that $\lim_{k\to\infty}g_k(t_k) = g(t)$. Then we have$$\limsup_{k\to \infty}g_k^'(t_k)\leq g^'(t)$$
    \end{lemma}

    \begin{proof}
        In this proof, $g(\cdot)^'$ and $g_k(\cdot)^'$ denotes the right directional derivative.

        Suppose $\mu > g^'(t)$, then since $g$ is convex, we can choose a $\bar{t}$, such that $\forall \Delta t<\bar{t}$, we have:
        \bea
            \frac{g(t + \Delta t) - g(t)}{\Delta t} < \mu
        \eea
        Thus, for relatively large $k$, we have:
        \bea
            \frac{g_k(t_k + \Delta t) - g_k(t_k)}{\Delta t} < \mu
        \eea
        But letting $\Delta t \to 0^+$, this implies:
        \bea
        \limsup_{k\to \infty}g_k^'(t_k) < \mu
        \eea
        Since this inequality holds for any $\mu > g^'(t)$, so we can conclude that: 
        \bea
        \limsup_{k\to \infty}g_k^'(t_k)\leq g^'(t)
        \eea
    \end{proof}

    We can see that the directional derivative of $g_k,g$ is just the directional derivative we have defined for $\phi$ along the geodesic. So that we have
    \bea
        f'(x; v)\leq \limsup_{k\to\infty} \phi'(x_k, s_k; v_k) \leq  \phi'(x, s_0; v)
    \eea

    Since $s_0 \in \calS(x)$ is arbitrary, we have $f'(x;v) \leq \max_{\bar{s} \in \calS(x)} \phi'(x,\bar{s};v)$. This relation together with inequality \eqref{eq:directioaldiffgeq} proves the equality \eqref{eq:directioaldiffeq}.

    {\bf Subdifferential of $f$}. 
    Since we assumed $\phi(\cdot, s)$ is Riemannian differentiable for all $s\in \calS$, we have $\partial_x \phi(x, s) = \{\nabla_x \phi(x, s)\}$ and for any geodesic segment $\gamma$ with $\gamma(0)=x$, we have $\phi(\gamma(t), s)\geq \phi(\gamma(0), s) + t\langle \nabla_x \phi(x, s), \gamma'(0)\rangle$. Thus for all $\bar{s} \in \calS$ we have
    \bea
        f(\gamma(t)) =    & \displaystyle \max_{s\in \calS} \phi(\gamma(t), s) \nonumber \\
                     \geq & \displaystyle \phi(\gamma(t), \bar{s}) \nonumber \\
                     \geq & \displaystyle \phi(x, \bar{s}) + t\langle \nabla_x \phi(x, \bar{s}), \gamma'(0)\rangle \nonumber \\
                     =    & \displaystyle f(x) + t\langle \nabla_x \phi(x, \bar{s}), \gamma'(0)\rangle
    \eea
    Therefore $\nabla_x \phi(x, \bar{s}) \in \partial f(x)$. Consider a convex combination of two gradients
    \bea 
        v = \alpha\nabla_x\phi(x, \bar{s}_1) + (1-\alpha) \nabla_x\phi(x, \bar{s}_2), \quad \alpha \in [0,1],
    \eea
    we still have 
    \bea
        f(\gamma(t)) = & \displaystyle \max_{s\in \calS} \phi(\gamma(t), s)\nonumber \\
                    \geq & \displaystyle \alpha \phi(\gamma(t), \bar{s}_1) + (1-\alpha) \phi(\gamma(t), \bar{s}_2)\nonumber \nonumber\\
                    \geq & \displaystyle \alpha\left(\phi(x, \bar{s}_1) + t\langle \nabla_x \phi(x, \bar{s}_1), \gamma'(0)\rangle\right)\nonumber \\
                    &+ (1-\alpha)\left(\phi(x, \bar{s}_2) + t\langle \nabla_x \phi(x, \bar{s}_2), \gamma'(0)\rangle\right)\nonumber \\
                    = & \displaystyle \alpha \phi(x, \bar{s}_1) + (1-\alpha) \phi(x, \bar{s}_2) + t\langle \alpha \nabla_x\phi(x, \bar{s}_1)\nonumber\\
                    &+(1-\alpha) \nabla_x\phi(x, \bar{s}_2), \gamma'(0)\rangle \nonumber\\
                    =   & \displaystyle f(x) + t\langle v, \gamma'(0)\rangle
    \eea
    Thus $v \in \partial f(x)$. Therefore 
    \bea \label{eq:supset}
        \partial f(x) \supset \conv{\{\nabla_x \phi(x,\bar{s}): \bar{s} \in \calS(x)\}}.
    \eea

    To prove the reverse inclusion, we use the hyperplane separation theorem. By the continuity of $\nabla_x\phi(x, \cdot)$ and the compactness of $\calS$, we have $\calS(x)$ is compact and $\{\nabla_x\phi(x, \bar{s}) | s\in\calS\}$ is also compact. If $d\in \partial f(x) \subset T_xM$ while $d\notin \conv{\{\nabla_x \phi(x,\bar{s}): \bar{s} \in \calS(x)\}}$, according to the strict separating theorem: there exist $v\in T_xM$ and $\mu \in \R$ such that 
    \bea
        \langle d, v \rangle > \mu > \langle \nabla_x \phi(x, \bar{s}), v \rangle, \quad \forall \bar{s} \in \calS(x)
    \eea
    Thus $\langle d, v\rangle > \max_{\bar{s}\in \calS(x)} \langle \nabla_x \phi(x, \bar{s}), v \rangle = f'(x; v)$, which contradicts the equivalent definition of subdifferential in Proposition \ref{prop:anothersubdiff}. Therefore $\partial f\subset \conv{\{\nabla_x \phi(x,\bar{s}): \bar{s} \in \calS(x)\}}$ and together with \eqref{eq:supset} we obtain the equality \eqref{eq:subdiffeq}.
\end{proof}

The key takeaway of the theorem above is that, if $f$ satisfies some convexity and continuity conditions, the subdifferential of a $\min\max f$ is just the convex hull of the gradients of the function $f$ at the maximizer. For our problem, we have $f(c) = \max_{s\in \calS_R} \rotdist^2(c,s)$, the problem is equivalent to the $f(c) = \max_{s\in \calS_R} \|c - s\|_F^2$. Thus, we just need to find the gradients of the function $\|c - s\|_F^2$ at the maximizer $s$ to obtain the subdifferential of $f$, which is just the vector points to the maximum distance point.

{\bf SDP-based geodesic gradient descent}.

The original maximization problem is presented as follows:
\bea
    \max_{R,t} & \|R - R_0\|_F^2 \nonumber\\
    \subject & (R,t) \in \purse
\eea

The \purse set constraint in Example~\ref{ex:3D3D} contains multiple polynomial constraints. For how to formulate the \purse set constraint in Example~\ref{ex:2D3D} into polynomial constraints, we refer to~\cite{yang23cvpr-object}, and for how to solve this problem through relaxations, we refer to the original paper~\cite{lasserre2001global}, and two appendices in~\cite{yang23cvpr-object,tang23arxiv-uncertainty}.

\input{sections/alg-sdp-iter.tex}

Based on the analysis above, the key problem turns into finding the maximum distance point to the current point $R_0$ on $\calS_R$. Maximizing a convex function on a convex set is already difficult to solve, let alone our problem here is a non-convex basic semialgebraic set. We leverage a hierarchy of convex relaxations based on sums-of-squares (SOS) programming~\cite{lasserre2001global}.

With the properties discussed in Theorem~\ref{thm:danskin}, Algorithm~\ref{alg:sdp-iter} can be regarded as a subgradient method on $\SOthree$. We then apply the following theorem to guarantee the linear convergence of our subgradient method.
\begin{theorem}[Subgradient Descent Convergence Rate~{\cite[Theorem 11]{zhang2016first}}] \label{thm:subgradconv}
    If $f$ is geodesically $\mu$-strongly convex and $L_f$-Lipschitz, and the sectional curvature of the manifold is lower bounded by $\kappa\le 0$, then the subgradient method with $\eta_s = \frac{2}{\mu(s+1)}$  satisfies
    \[ f\left(\overline{x}_t \right) - f(x^*) \le \frac{2\zeta(\kappa,D)L_f^2}{\mu(t+1)},\]
     where $\overline{x}_1 = x_1$, and $\overline{x}_{s+1} = \exp_{\overline{x}_s}\left(\frac{2}{s+1}\exp_{\overline{x}_s}^{-1}(x_{s+1})\right)$.
\end{theorem}

\input{sections/fig-sdp-iteration.tex}

We apply numerical experiments to validate the convergence rate of Algorithm~\ref{alg:sdp-iter} using commercial optimizer MOSEK\footnote{https://www.mosek.com/}. We run the SDP iteration for $100$ steps and calculate the difference between the rotation at each step and the final result $\epsilon_i = \rotdist(R_i, R^*)$. We then plot the $\epsilon_i$ against the iteration number $i$ in Figure~\ref{fig:sdp-iter} in logarithm scales. The $\log(\epsilon_i)$ decreases linearly with the iteration number $\log(i)$ with a slope close to $1$, which supports the theoretical conclusion of Theorem~\ref{thm:subgradconv}.

However, solving SDP is extremely time-consuming, and such an iterative method cannot be implement in parallel. Running on a laptop with a 12-core Intel i5-13500H CPU, each step takes more than $15$ seconds to solve, thus the overall algorithm with $100$ steps takes more than $25$ minutes to run. Therefore, this method is impractical in real-time applications.

%% file: sections/alg-sdp-iter.tex

\begin{algorithm}[t]
    \SetAlgoLined
    {\bf Input: } \purse $S$; initial rotation center $R_0$; overall iteration steps $N_I$;\\
    {\bf Output: } rotation center $R^*$; sampled boundary poses $\hat{\partial S_R}$; \\$\hat{\partial S_R}\gets \emptyset$;\\

    \For{$i \gets 1$ to $N_I$}{
        $(\hat{R}, \hat{t}) \gets$ sdp_relaxation($R_{i-1}$, $S$); \% Find the poses in $S$ with the maximum distance with $R_{i-1}$\\
        $R_i$ = SLERP($R^*$, $\hat{R}$, $1/i$); \% Apply spherical linear interpolation (SLERP) with ratio $1/i$ to find the next step\\
        $\hat{\partial S_R} \gets \hat{\partial S_R} \cup \{\hat{R}\}$;\\
    }
    \% calculate $R^*$ as the average rotation of the last $10$ $R_t$ \\
    $R^* \gets \text{proj}_{\SOthree} \sum_{t=T-9}^{T} R_t$;\\
    {\bf return:} $R^*$, $\hat{\partial S_R}$\\
    \caption{SDP-based geodesic gradient descent~\label{alg:sdp-iter}}
\end{algorithm}

%% file: sections/fig-sdp-iteration.tex

\begin{figure}[t]
	\begin{center}
		\begin{tabular}{c}
			\begin{minipage}{0.5\columnwidth}
				\centering
				\includegraphics[width=\textwidth]{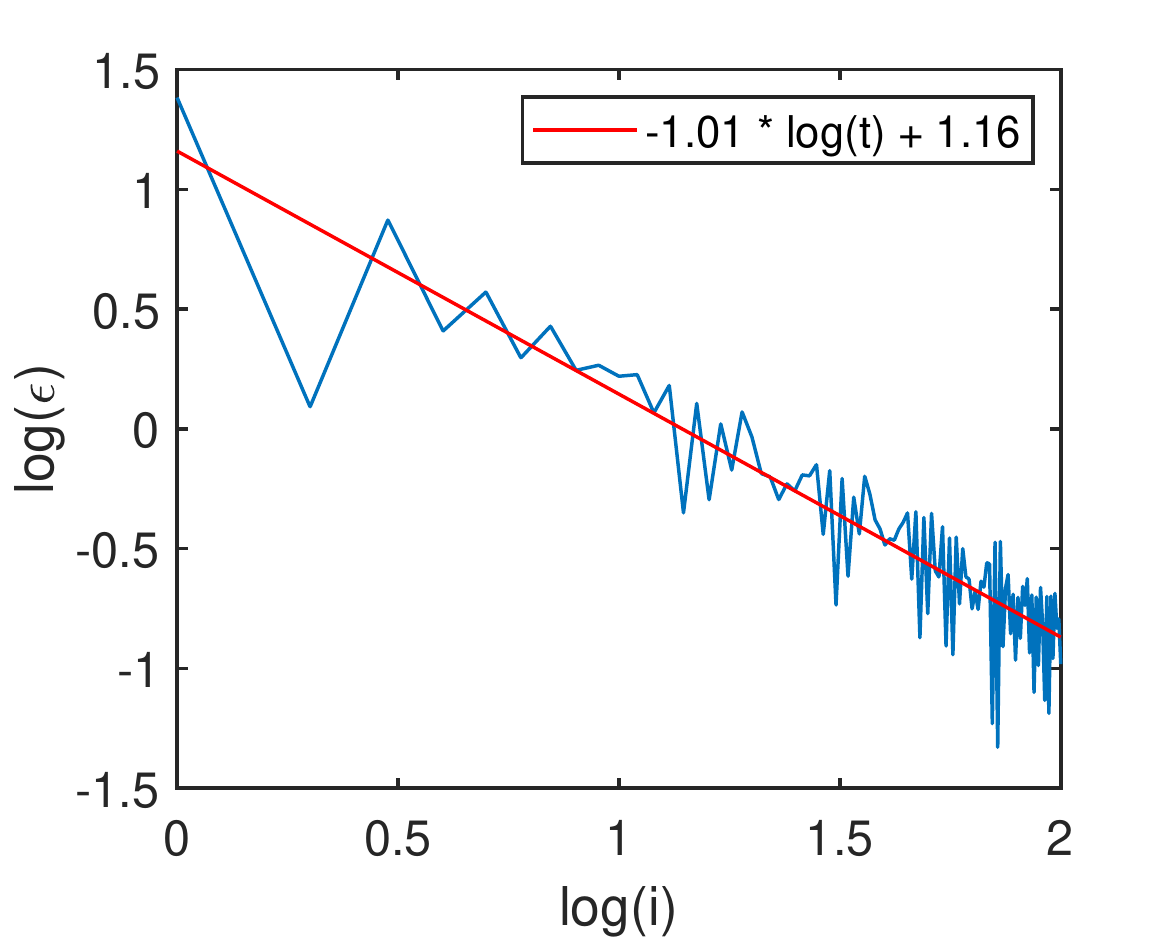}\\
			\end{minipage}
		\end{tabular}
	\end{center}
	\vspace{-2mm}
	\caption{The distance between the rotation center output and the intermediate rotations at different steps $\epsilon = \rotdist(R^*, R_i)$ acquired by SDP solver. The slope in the log scaled figure indicates the convergence rate of Algorithm~\ref{alg:sdp-iter}.
	\label{fig:sdp-iter}}
\end{figure}

%% file: sections/app-algorithm.tex

\subsection{Boundary Sampler Algorithms}
In this section, we give a more detailed explanation of the translation boundary $\partial S_t$ sampler in Algorithm~\ref{alg:sample_t}. We also provide its parallelized version in Algorithm~\ref{alg:sample_t_parallel}. 

\subsubsection{Translation Boundary $\partial S_t$ Sampler}

\input{sections/alg-t-boundary-sampler.tex}

In general, $\partial S_t$ sampler is similar to $\partial S_R$ sampler, but in $\R^3$ instead of $\SOthree$. However, there are two differences significantly impacting the result of relative ratio $\eta_t$ for the two examples that need to be considered.

For object pose estimation (Example \ref{ex:2D3D}), we need to initialize the center velocity with additional scaling along a specific axis. As shown in Fig.~\ref{fig:sample_comp}(a), we notice that $S_t$ (on the right) is in a needle-like shape with long expansion in one direction but very thin in two other orthogonal directions. If we uniformly randomize center velocities of the random walks, most of the walks will stop close to the center of $S_t$. Therefore, the two points on the two ends of $S_t$ are not likely to be reached.

To address this problem, we apply additional scaling to the translation velocity according to the principle component analysis (PCA) of $S_{t0}=\{t_0| (R_0, t_0)\in S_0\}$. Specifically, we first calculate the PCA of $S_{t0}$ so that we obtain the first weight vector ${u_1}$ which is parallel to the longest axis of $S_t$. In this way, we can rescale the center velocity based on the expansion of $S_{t0}$ on $3$ PCA weights. Thus, the center velocity is more likely to point to the direction of the longest axis of $S_t$. In fact, the algorithm already performs well if we directly use the first PCA weight as the center velocity. 

For point cloud registration (Example~\ref{ex:3D3D}), we found that $\partial S_R$ sampler is more efficient in sampling $\partial S_t$ than the $\partial S_t$ sampler in Algorithm~\ref{alg:sample_t}. We explain this surprising result with the following reason. 

Suppose the first point cloud is $A\in \R^{3\times N}$, a transform in $(R, t)\in \SEthree$ represents $RA + t$. During the random walk process, suppose we find an additional transform $(\tilde{R}, \tilde{t})$, and then the overall transform is $(\tilde{R}R, \tilde{R}t + \tilde{t})$. The rotation part is a direct multiplication of the two rotation matrices, therefore $\partial S_R$ is efficiently sampled even if we only vary $\tilde{R}$. However, when we use Algorithm~\ref{alg:sample_t} to sample $\partial S_t$, the translation part $\tilde{R}t + \tilde{t}$ not only depends on $\tilde{t}$, but also on $\tilde{R}$. Even if we extensively explore $\tilde{t}$ in the random walks, the overall translation could expand even wider due to the independent term $\tilde{R}t$ which is not optimized in the $\partial S_t$ sampler. Thus $\partial S_t$ might not be fully explored.

\subsubsection{Translation Boundary $\partial S_t$ Sampler with Parallelization}

Similar to the parallelized $\partial S_R$ sampler, we implement the parallelized version of $\partial S_t$ sampler on NVIDIA GPUs using CuPy. We annotate all the dimensions of each variable, so that readers can easily understand the parallelized computation pipeline. We provide more detailed explanations of the GPU operators as follows.

\begin{itemize}
    \item repeat($A$, $N$): repeat the array $A$ for $N$ times along certain dimension.
    \item matmul($A$, $B$): parallel matrix multiplication of $A$ and $B$. If $A$ or $B$ has more than 2 dimensions, the last two dimensions are treated as the matrix dimensions and all other dimensions are broadcasted.
    \item find_farthest_translation($R$, $t$, $I$): find the pose that is still in $S$ (indicated by $I$) and has the maximum translation movement from all the poses in $R$ and $t$. This is equivalent to line~\ref{line:start-find-farthest}-\ref{line:end-find-farthest} in Algorithm~\ref{alg:sample_t} but with parallel computation on all $|S_0|N_W$ walks.
\end{itemize}

The other operators not mentioned here are similar to the ones in the non-parallelized version but applies to all input values simultaneously.

\input{sections/alg-t-boundary-sampler-parallel.tex}

%% file: sections/alg-t-boundary-sampler.tex
\setlength{\textfloatsep}{0pt}%
\begin{algorithm}[t]
\SetAlgoLined
    {\bf Input: } initial poses $S_0\subset \SEthree$; \purse $S$; base center velocity magnitude $v_0$; time step decay factor $\gamma$; rotation perturbation scale $\theta_p$; random walk trial number $N_W$; iteration number $N_I$; perturbation number $N_P$; optimal perturbation number $N_P^*$; time step scaling number $N_T$;\\
    {\bf Output: } sampled boundary poses $\hat{\partial S_t} \subset \R^3$ as an inner approximation of $\partial S_t$;\\

    $\bar{R} \gets \text{proj}_{\SOthree}( \sum_{(R_j,*) \in S_0} R_j )$; \label{line:Ravg}\\
    $\bar{t} \gets \frac{1}{\vert S_0 \vert } \sum_{(*,t_j) \in S_0} t_j $; \label{line:tavg} \\
    $\hat{\partial S_t} \gets \emptyset$; \\
    \For{ $(R_0, t_0) \in S_0$}{
        \For { $n \gets 1$ to $N_W$} {
            $v \gets$ init_center_velocity$(t_0, \bar{t}, v_0)$; \\
            $R^*\gets R_0, t^* \gets t_0$; \\
            \% Iterate $N_I$ times so that the evolved pose gets close to $\partial S_t$\\
            \For { $i \gets 1$ to $N_I$} { \label{line:iterate}
                \% Randomize $N_P$ translation perturbations
                \For {$j \gets 1$ to $N_P$} {
                    $R_p$ = perturbation($\theta_p$); \\
                    $R_j \gets R_pR^*$ ; \\
                    $d_j \gets \text{dist}(R^*, t_j, \partial S)$; \\
                }
                \% Pick out $N_P^*$ perturbations that drag the pose away from $\partial S$\\
                $\{j_k\}_{k=1}^{N_P^*} \gets$ top_k_indices($\{d_j\}_{j=1}^{N_P}$, $N_P^*$); \\
                \For {$k \gets 1$ to $N_p^*$} {
                    \For {$ m \gets 1$ to $N_T$} {
                        $\Delta T \gets \gamma^{m-1}$ \\
                        $t_{km} \gets $ update_translation($t^*$, $R_{j_k}$, $v$, $\Delta T$); \label{line:update-t}\\
                        $I_{km} \gets $ in_purse($R_{j_k}, t_{km}$); \\
                    }
                }
                \% Find the optimal pose that is still in $S$ and has the maximum rotation movement\\
                $m_0 \gets \min\cbrace{m \mid \exists k \text{\ s.t.\ } I_{km} = 1}$; \label{line:start-find-farthest}\\ 
                $k_0 \gets \cbrace{k \mid I_{km_0} = 1}$; \\
                $R^* \gets R_{k_0m_0}$; \\
                $t^* \gets t_{k_0m_0}$; \label{line:end-find-farthest}\\
            }
            
            $\hat{\partial S_t} \gets \hat{\partial S_t} \cup t^*$; \\
        }
    }
    {\bf return:} $\hat{\partial S_t}$\\

    \caption{$\partial S_t$ sampler~\label{alg:sample_t}}
\end{algorithm}

%% file: sections/alg-t-boundary-sampler-parallel.tex

\begin{algorithm}[t]
    \SetAlgoLined
    {\bf Input: } initial poses $S_0\subset \SEthree$; \purse $S$; base center velocity magnitude $v_0$; time step decay factor $\gamma$; rotation perturbation scale $\theta_p$; random walk trial number $N_W$; iteration number $N_I$; perturbation number $N_P$; optimal perturbation number $N_P^*$; time step scaling number $N_T$;\\
    {\bf Output: } sampled boundary poses $\hat{\partial S_t} \subset \R^3$ as an inner approximation of $\partial S_t$;\\
    $v \gets$ init_center_velocity($S_0, v_0, N_W$);  \% $(|S_0|, N_W, 3)$ \\
    $(R^*, t^*) \gets$ repeat($S_0, N_W$);  \% $R^*$: $(|S_0|, N_W, 3, 3)$, $t^*$: $(|S_0|, N_W, 3)$\\
    $\Delta T \gets (1, \beta, \beta^2, \cdots, \beta^{N_T-1})$; \% $(N_T)$\\

    \For { $i\gets 1$ to $N_I$} {
        $R_p \gets$ perturbation($\theta_p, |S_0|N_WN_P$);  \% $(|S_0|, N_W, N_P, 3, 3)$ \\
        $R \gets$ matmul($R_p$, repeat($R^*, N_P$));   \% $(|S_0|, N_W, N_P, 3, 3)$\\
        $t \gets$ repeat($t^*, N_P$); \% $(|S_0|, N_W, N_P, 3)$\\
        $d \gets$ dist($R, t, \partial S$); \% $(|S_0|, N_W, N_P)$\\ \label{line:dist_gpu}
        $j \gets$ top_k_indices($d, N_P^*$);  \% $(|S_0|, N_W, N_P^*)$\\
        $\tilde{R} \gets$ repeat($R_{[j]}, N_T$); \% $(|S_0|, N_W, N_T, N_P^*, 3, 3)$\\
        $\tilde{t} \gets $ update_translation(repeat($t^*, N_P^*$)$, \tilde{R}, v, \Delta T$); \% $(|S_0|, N_W, N_T, N_P^*, 3)$\\
        $I \gets$ in_purse($\tilde{R}, \tilde{t}$); \% $(|S_0|, N_W, N_T, N_P^*)$\\ \label{line:in_purse_gpu}
        $R^*, t^* \gets$ find_farthest_translation($\tilde{R}, \tilde{t}, I$); \\
    }
    {\bf return:} $\hat{\partial S_t}\gets t^*$\\
    \caption{$\partial S_t$ parallel sampler~\label{alg:sample_t_parallel}}
\end{algorithm}

%% file: sections/app-intuition-of-alternate.tex
\pdfstringdefDisableCommands{\def\neq{≠}}
\subsection{Intuitive example for $\partial S \neq \partial S_R \times \partial S_t$}
\label{app:intuitive-example}

Here we will give a simple example showing that by simply pushing the poses to the boundary of \purse set could never be enough for getting the tight characterization of the boundary of the \purse set with respect to rotation and translation.

We suppose $R$ lies in the $x$-$y$ plane and satisfies $x^2 + y^2 \leq 1$. $t$ lies on the z axis and satisfies $-1\leq z \leq 1$. And the \purse set as $S \coloneqq \{(x,y,z)|x^2+y^2+z^2\leq 1\}$. We remark that here $R$ and $t$ do not refer to rotation and translation anymore (so with a bit abuse of notation). We here take sampling the boundary of $R$ set here as an example.

We argue that, by simply sampling the boundary of the \purse set, it's hard to densely sample from the $R$ set and $t$ set.  See the visualization in Fig~\ref{fig:low-sample-eff-example}. We can see that, if we just simply uniformly sample from the boundary of the \purse set, it won't be efficient in sampling the boundary of the $R$ set, let alone this is just an over-simplified 2D example.

\begin{figure}
    \begin{center}
        \includegraphics[width = 0.5\textwidth]{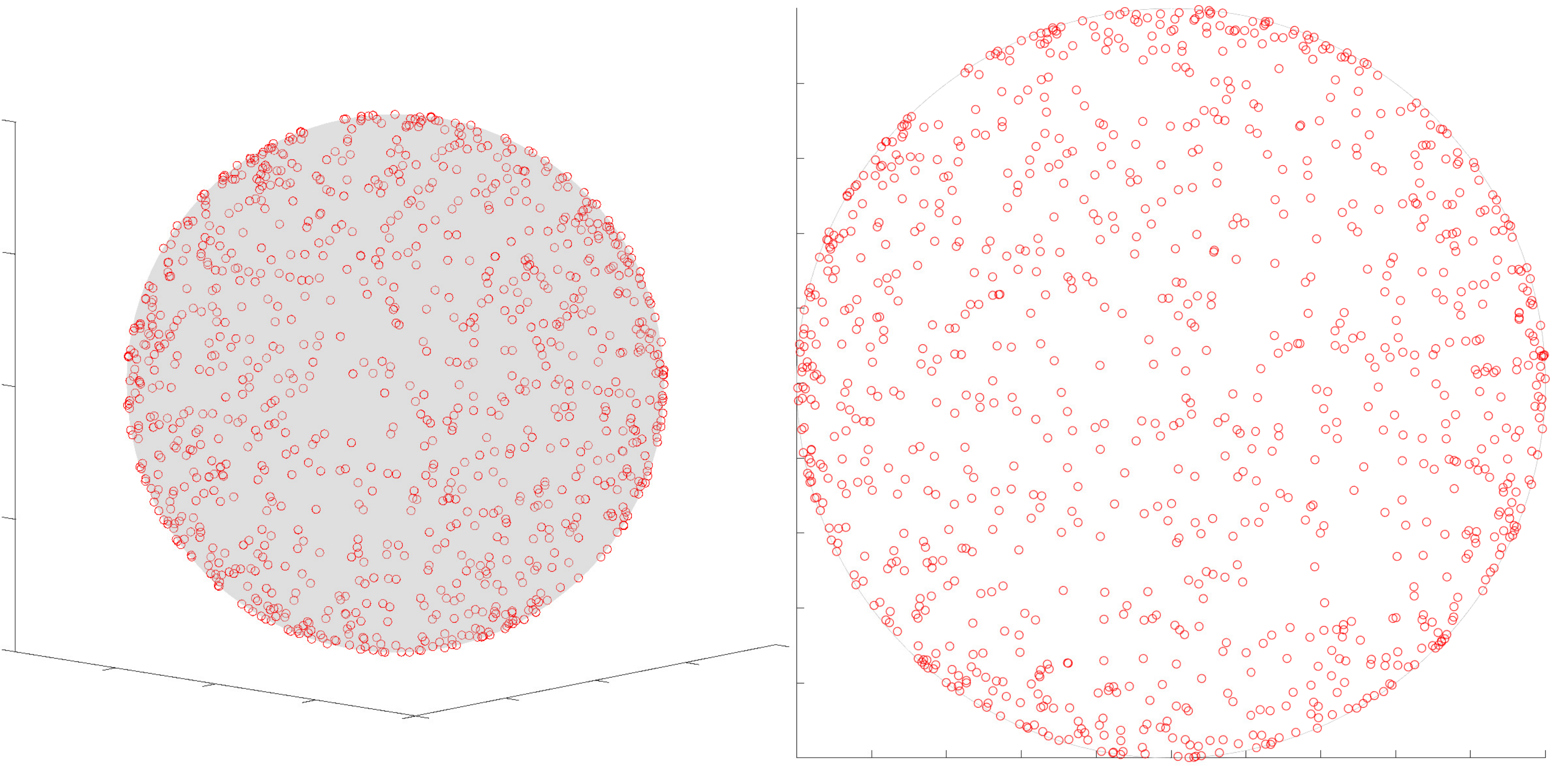}
    \end{center}
    \caption{Simply sample from \purse set boundary leads to low sampling efficiency of the $R$ set boundary. Left: Uniform sampling of the unit sphere, Right: $R$ set (projection to the x-y plane) \label{fig:low-sample-eff-example}}
\end{figure}

Our algorithm proceeds as follows: 

{\bf Step 1.} For the initial point, we first perturbs the translation $t$ $N_P$ times. And we keep the top $N_P^*$ best perturbations.

{\bf Step 2.} Then we move the rotation $R$ in some random direction. We keep $N_T$ exponentially-decay movements and check which is the farthest valid movement.

The visualization of the intuitive example is shown in Fig.~\ref{fig:intuitive example}. In each row, the left plot shows the points in the 3D \purse set, the middle plot shows the $R$ set (projection onto the x-y plane), and the right plot shows the $t$ set (projection onto the z axis). It's easy to see that our algorithm could efficiently sample the boundary of the $R$ set and $t$ set because as the iteration proceeds, the sampled point is becoming closer to the boundary of the $R$ set.

\input{sections/fig-intuition-example.tex}

%% file: sections/fig-intuition-example.tex
\begin{figure}[h]
	\begin{center}
		\begin{tabular}{c}
            \hspace{-10mm}	
			\begin{minipage}{0.5\textwidth}
				\centering
				\includegraphics[width=\textwidth]{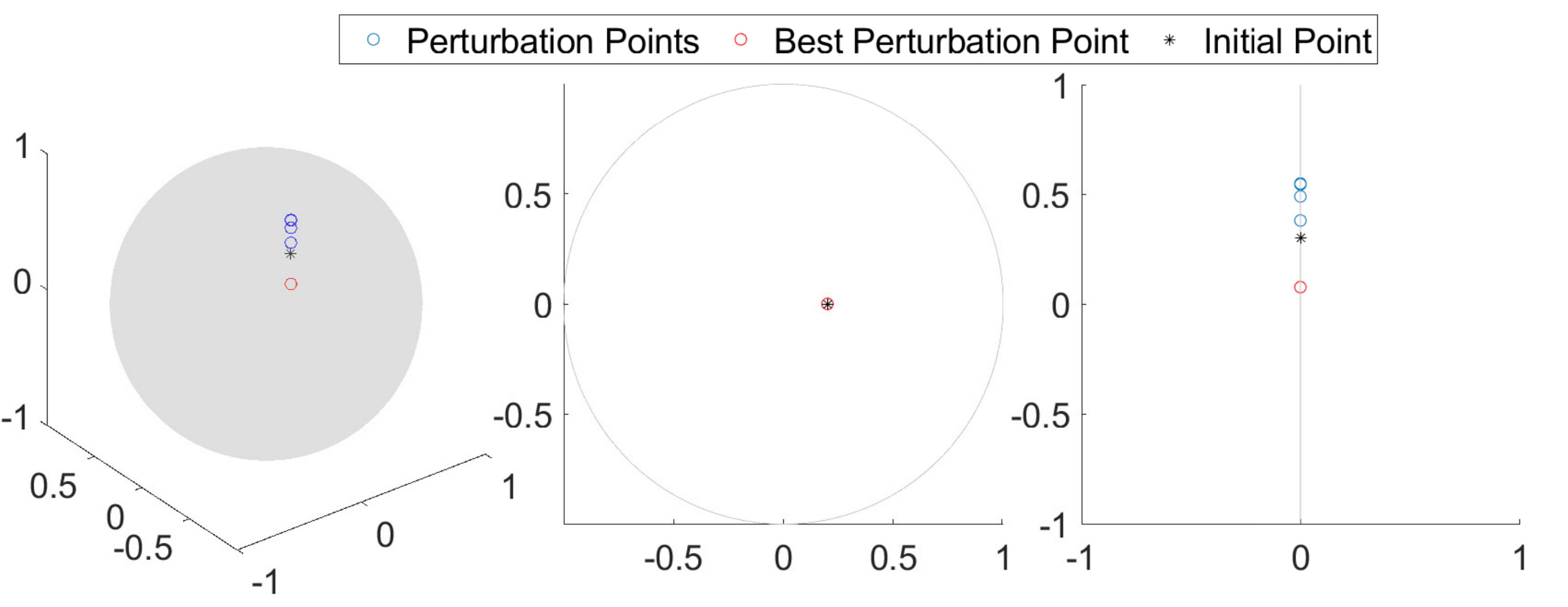}
			\end{minipage}\\
			\hspace{-10mm}	
			\begin{minipage}{0.5\textwidth}
				\centering
				\includegraphics[width=\textwidth]{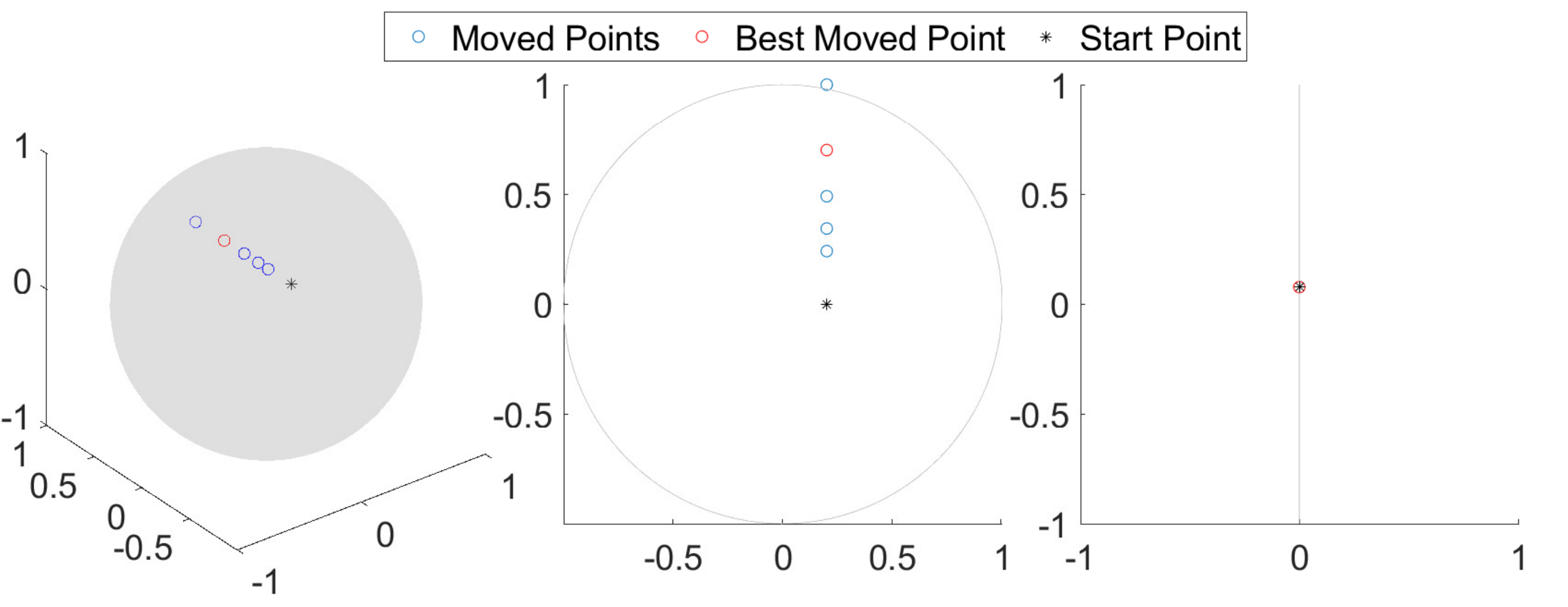}
			\end{minipage}\\
        \end{tabular}
	\end{center}
	\vspace{-2mm}
	\caption{Plot for the intuitive example. Up: Step 1, Down: Step 2.\\
    Left: 3D \purse set, Middle: $R$ set (projection onto the x-y plane), Right: $t$ set (projection onto the z axis) \label{fig:intuitive example}}
\end{figure}